\documentclass[11pt]{article} 
\usepackage{amsmath}
\usepackage{lscape}
\usepackage[section]{placeins}
\usepackage{mathtools}
\usepackage{amsthm,bm}
\usepackage{enumerate}
\usepackage{pifont}
\usepackage{dsfont}
\usepackage{times}
\usepackage{graphicx}
\usepackage{amssymb}
\usepackage{url} %,hyperref
\usepackage{microtype}
\usepackage{authblk}
\usepackage{siunitx}
\usepackage{subcaption}
\usepackage[margin=1in]{geometry}
\usepackage{longtable}
\usepackage{booktabs}
\usepackage{multirow}
\usepackage{xcolor}
\usepackage{comment}
\usepackage{color,soul} 
\usepackage{graphicx,tabularx} 
\usepackage{caption}
\usepackage{algorithm}
\usepackage{colortbl}

\usepackage{algpseudocode}
\usepackage{float}
\usepackage{minitoc}

\usepackage[toc,page,header]{appendix}
\usepackage{etoc}
\definecolor{darkgreen}{RGB}{0,90,0.1}
\usepackage[colorlinks=true, citecolor=darkgreen, urlcolor=darkgreen]{hyperref}

\usepackage{multicol}
\usepackage{dcolumn}
\usepackage{enumitem}
\usepackage[capitalize,noabbrev]{cleveref}
\usepackage[version=4]{mhchem} 
\usepackage[textsize=tiny]{todonotes}
\usepackage{natbib}
\usepackage{tikz}
\usetikzlibrary{calc}
\usetikzlibrary{calc}
\usepackage[most]{tcolorbox}
\usepackage{graphicx}   % for \scalebox
\usepackage{varwidth}
\usepackage{xcolor}
\hypersetup{
    colorlinks=true,
    linkcolor=blue,
    urlcolor=blue,
    citecolor=red
}

\theoremstyle{plain}

\newtheorem{theorem}{Theorem}[section]
\newtheorem{proposition}[theorem]{Proposition}
\newtheorem{lemma}[theorem]{Lemma}
\newtheorem{corollary}[theorem]{Corollary}
\theoremstyle{definition}

\newtheorem{assumption}[theorem]{Assumption}
\theoremstyle{remark}
\newtheorem{remark}[theorem]{Remark}
\newtheorem{exampleinner}{Example}[section]

\definecolor{LightBlue}{rgb}{1,0.9,0.9}
\definecolor{meaningless}{RGB}{255,0,0}
\definecolor{worstcase}{RGB}{0,255,0}
\definecolor{reliability}{RGB}{0,255,0}
\definecolor{robustoptim}{RGB}{255,255,0}
\definecolor{lightgray}{gray}{0.93}

\title{ReliableNet: A Chance-Constrained Approach to Trustworthy Classification in Deep Learning}

\author[1,2]{Ange-Clément Akazan}

\author[3]{Ineza Remy Mugenga}

\author[3]{Abebe Geletu}

\author[1]{Jean Medard Ngnotchouye}

\author[2]{Issa Karambal}

\affil[1]{School of Agriculture and Science, University of KwaZulu-Natal,
Pietermaritzburg, KwaZulu-Natal, South Africa}

\affil[2]{AIMS Research and Innovation Centre, African Institute for Mathematical Sciences,
Kigali, Rwanda}
\affil[3]{African Institute for Mathematical Sciences,
Kigali, Rwanda}

\date{}

\begin{document}

\maketitle

\begin{abstract}
A prediction that is both confident and wrong is a critical reliability failure because it can bypass abstention and human review precisely when the model is mistaken. Empirical risk minimization (ERM) controls average loss but not this failure directly, while calibration, uncertainty estimation, conformal risk control, and selective prediction methods target related reliability properties rather than bounding the joint failure event during training. We propose ReliableNet, which constrains the Joint Confident-Wrong (JCW) probability, the probability that a prediction is simultaneously confident and incorrect, below a user-specified risk budget $\alpha\in(0,1)$. We formulate this as a chance-constrained ERM problem, use a conservative smooth inner approximation whose population feasibility implies the original JCW constraint. Across four tabular and two image datasets, ReliableNet is the only method certified within the JCW budget for every dataset and seed in distribution, when compared against baselines spanning ERM, post-hoc calibration, conformal risk control, and selective prediction. Under demographic, ambiguity, spurious-correlation, novel-class, and covariate shifts, it achieves the lowest empirical JCW among the compared methods while remaining very competitive in accuracy, coverage, calibration, and selective prediction. Risk-coverage results further indicate that ReliableNet achieves better selective ranking than the benchmark methods  on most datasets. Overall, ReliableNet provides a principled approach to trustworthy classification.
\end{abstract}

\section{Introduction}
\label{sec_intro}

Machine learning systems are increasingly used across science, industry, and daily life, powering systems that translate language, recommend content, drive vehicles,  support medical decisions, and more. As these systems move from research benchmarks into deployment, the demands placed on them shift: a model is no longer judged solely by its average accuracy, but by whether it can be trusted in the specific situations where its outputs drive real decisions. This is especially acute in classification, where predictions increasingly influence consequential outcomes in medical diagnosis, autonomous perception, financial risk assessment, and scientific discovery \citep{yeturu2020machine,kotsiantis2007supervised}.
In such settings, accuracy alone is insufficient: a prediction is useful only if the system also signals when it should be trusted. 

\noindent Classifiers make this concrete through confidence scores, which are often used operationally, a high-confidence prediction may be trusted automatically, while a low-confidence one is deferred, rejected, or routed to a human reviewer. The central reliability question is therefore not whether a classifier is accurate on average, but whether it also avoids being wrong precisely when it appears most confident. This is difficult to guarantee in modern learning systems, where real-world problems involve noisy measurements, partial observability, distribution shift, and model misspecification, the very conditions under which confident errors are most likely and most costly.

These sources of uncertainty are commonly separated into \emph{aleatoric} uncertainty, which reflects irreducible ambiguity in the data-generating process, and \emph{epistemic} uncertainty, which reflects incomplete knowledge of the predictive mechanism and can in principle be reduced with better data or modeling \citep{HORA1996217,der2009aleatory,hullermeier2021aleatoric,gruber2023sources}. 

\medskip
\noindent In deployment, for classification tasks, both forms of uncertainty are often compressed into class probability vector and acted on using a scalar confidence score, typically the maximum predicted class probability, for class prediction. This creates a sharp failure mode: the classification model may assign high confidence to a prediction that is wrong, either because the input is intrinsically ambiguous or because the model is extrapolating beyond what it has learned. 
However, the dominant supervised learning paradigm, empirical risk minimization (ERM), is not designed to prevent this failure. ERM seeks parameters that minimize an expected loss,  $\mathbb{E}\big[\ell(f_\theta(X),Y)\big]$, or its empirical counterpart. This objective is well suited to average-case predictive performance, but it does not directly control rare but consequential tail events. A classifier can thus attain low average error while assigning near-certain confidence to a small set of incorrect predictions, and such high-confidence errors are qualitatively different from ordinary ones: they are the errors most likely to bypass caution, abstention, or human review.

\noindent The consequences are not abstract. A confident false negative in medical triage can send a malignant lesion home without further examination. A confident perception error in an autonomous vehicle can suppress the very braking response that was needed. A confident underestimate of credit risk can wave through a loan that should have been declined.  In each case it is not the error that causes the damage but the confidence attached to it: a hesitant mistake still invites a second look, whereas a confident one bypasses review entirely. The failure that matters in deployment is therefore specific, predictions that are simultaneously confident and wrong. These slip through abstention, deferral, and human review untouched, because
every confidence-based safeguard is triggered by low confidence. Worse, they are disproportionately the errors that survive training: ERM drives down average loss with no pressure on the high-confidence tail.

% That asymmetry is what we want a reliable classifier to respect, its confident predictions  should rarely, if ever, be wrong.
% \medskip
\medskip

\noindent Existing reliability methods address related aspects of this problem, but they do not target it directly. Post-hoc calibration (Platt and temperature scaling) adjusts predicted probabilities so that confidence better matches empirical accuracy \citep{platt1999probabilistic,guo2017calibration}, but it optimizes an average: expected calibration error sums the confidence-accuracy gap over all predictions, so it can be small while confident errors concentrate in the high-confidence region, a blind spot that even modern, well-calibrated architectures inherit \citep{minderer2021revisiting}. Average calibration is not even a reliable summary of confidence quality: \citet{chidambaram2025reassessing} show that a trivial constant-confidence predictor attains near-zero calibration error while carrying no discriminative information. Training-time modifications, label smoothing, focal loss, Brier-score objectives, mixup, likewise reduce overconfidence on average rather than bounding confident error \citep{szegedy2016rethinking,muller2019does,mukhoti2020calibrating,zhang2018mixup}. Selective prediction equips a classifier with a rejection option and studies the coverage-accuracy tradeoff on accepted samples \citep{chow1957reject,bartlett2008classification,geifman2019selectivenet,liu2019deep}; conformal methods give finite-sample post-hoc guarantees under exchangeability by calibrating prediction sets or risk-controlling thresholds on held-out data \citep{angelopoulos2024conformal,angelopoulos2023conformal}; and Bayesian and ensemble methods improve uncertainty estimates through posterior or model averaging \citep{gal2016dropout,lakshminarayanan2017simple}.
 These approaches target related reliability objectives, but not the specific training-time constraint on confident misclassification considered here.

\medskip
\noindent We focus on the event where a classifier predicts  confidently and incorrectly. We call the probability of this event the \emph{Joint Confident-Wrong} (JCW) probability, where a prediction is deemed confident when its top-class probability (confidence) exceeds a chosen confidence threshold $\varepsilon \in (0,1)$. Rather than measuring or correcting this quantity after training, we treat it as a constraint during training: we minimize the usual ERM classification loss while requiring the Joint Confident-Wrong
probability to stay below a user-specified risk budget $\alpha \in (0,1)$ (e.g., $\alpha=0.05$). The threshold defines what counts as a confident prediction, and the risk budget sets how much confident error
is tolerated, so the practitioner controls reliability directly through two interpretable quantities. This formulation results into  a single chance-constrained (probabilistic-constrained) optimization problem.

We encoded the high-confidence misclassifications using a scalar violation function $g_\theta$  that is differentiable almost everywhere to facilitate training. This formulation directly encodes the deployment requirement that accepted high-confidence predictions should rarely be wrong. The resulting optimization problem was challenging because the JCW constraint contains a discontinuous indicator event and is generally nonconvex. To make the problem tractable, we use the inner-outer analytic approximation framework for probabilistic-constrained optimization \citep{ nemirovski2007convex, doi:10.1137/15M1049750}. ReliableNet replaces the hard event during training with a smooth conservative inner surrogate, while the corresponding outer surrogate is used to assess approximation tightness. The resulting problem is a surrogate constrained optimization problem whose Lagrangian has a nonconvex-concave minimax structure, which motivates a two-timescale primal-dual scheme in which the model parameters are updated by descent and a bounded multiplier by projected ascent, increasing the influence of the constraint when violations occur \citep{lin2020gradient,lin2025two}.  Unlike post-hoc calibration or threshold selection, ReliableNet incorporates the reliability requirement during training, so that the model parameters are optimized under an explicit constraint on confident error.

The main contributions of this work are threefold.
\begin{itemize}
    \item
    We formalize the Joint Confident-Wrong probability, a direct measure of high-confidence misclassification, and formulate classifier
    training as a probabilistic-constrained problem that explicitly bounds this failure event. We further establish its connections to selective risk  and high-confidence calibration.

    \item
    We adapt a smooth inner approximation whose population feasibility implies satisfaction of the JCW constraint, together with
    finite-sample results linking empirical and population surrogate feasibility. We also provide an independent held-out Clopper-Pearson certificate for the final hard JCW event.

    \item
    We evaluate ReliableNet against six baselines spanning post-hoc calibration, conformal risk control, and selective prediction. Across
    four tabular and two image datasets, ReliableNet is the only method certified on every dataset and seed in distribution. Under demographic, ambiguity, spurious-correlation, novel-class, and covariate shifts, it attains the lowest empirical JCW while remaining stronger or very competitive in overall accuracy, calibration, and selective-prediction performance.
\end{itemize}

The remainder of the paper is organized as follows. \Cref{sec_related} reviews the state of the art and identifies the gap addressed by this study. \Cref{sec_formulation} concerns the problem definition, \Cref{sec_method} presents the JCW formulation and the ReliableNet training procedure. \Cref{sec_experiments} describes the experimental setup and reports the empirical results. \Cref{sec_conclusion} concludes the paper.

\section{State-of-the-Art}
\label{sec_related}
Work on reliable classification is most closely connected to four neighboring literature: calibration, selective prediction, conformal risk control, and predictive uncertainty estimation. Across most of these lines of work, the common objective is to make model confidence more informative for downstream decision making, especially when overconfident mistakes are costly. Our method is close in spirit to this literature, but differs in that it does not treat reliability only as a post-hoc diagnostic, an abstention policy, or an uncertainty score. Instead, it incorporates the failure event of interest directly into the training objective by constraining the probability of confident misclassifications.

A first line of work studies \emph{calibration}, namely the agreement between predicted confidence and empirical correctness. Classical
post-hoc methods include Platt scaling, isotonic regression, binning-based calibration, and temperature scaling \citep{platt1999probabilistic,zadrozny2002transforming, niculescu2005predicting,guo2017calibration}. A complementary line modifies
training itself through label smoothing, mix-up, focal loss, logit-norm penalties, or trainable calibration regularizers \citep{szegedy2016rethinking,muller2019does,zhang2018mixup, thulasidasan2019mixup,mukhoti2020calibrating,joo2020revisiting, kumar2018trainable}. These methods generally improve average calibration by reducing overconfidence. However, lower raw ECE does not necessarily imply more informative confidence scores. \citet{wang2021rethinking} show that confidence-penalizing regularizers can compress the score distribution and remove information about sample difficulty, thereby limiting subsequent post-hoc calibration. ReliableNet differs in both target and mechanism because rather than penalizing confidence globally, it directly constrains the probability that a prediction is simultaneously confident and wrong. Thus, high confidence remains permissible on correct predictions, while the specific failure event of confident misclassification is explicitly controlled.
% Accordingly, our aim is not to minimize average miscalibration but to bound the probability of high-confidence failure, a marginal-tail quantity that a low ECE does not control. We report ECE and AURC as secondary diagnostics rather than primary objectives, and find that our method remains competitive on both while directly optimizing the confident-wrong rate.

A second line of work addresses reliability through \emph{abstention} or \emph{selective prediction}. The classical reject-option formulation goes back to Chow, and later learning-theoretic work developed selective classification as a principled risk-coverage trade-off \citep{chow1957reject,chow1970optimum,bartlett2008classification,JMLR:v11:el-yaniv10a}. In deep learning, this perspective has led to confidence-thresholding baselines, SelectiveNet, and Deep Gamblers, which equip the predictor with an explicit reject option and optimize performance on the accepted subset of inputs \citep{geifman2017selective,geifman2019selectivenet,liu2019deep}. Closely related are trust estimation methods such as Trust Score, which attempt to predict whether an individual decision should be trusted by comparing classifier outputs to local neighborhood structure in feature space \citep{jiang2018trust}. This literature is operationally very close to our setting, but its guarantees are typically stated in terms of selective risk at a chosen coverage level, rather than as a direct upper bound on the probability that the model is simultaneously confident and wrong.

A third line of work provides \emph{post-hoc statistical guarantees}. Conformal prediction converts point predictions into prediction sets or abstaining decision rules with finite-sample validity guarantees under exchangeability \citep{angelopoulos2023conformal}. More recently, Conformal Risk Control generalized this paradigm beyond marginal coverage to broader monotone risk functionals, including risks relevant to selective prediction \citep{angelopoulos2024conformal}. Recent extensions have begun to relax the idealized exchangeability assumption in more specialized non-exchangeable settings, but the framework remains fundamentally calibration-based and post-hoc \citep{farinhas2024non, fontana2023conformal}. This makes conformal methods one of the strongest available tools for uncertainty control after training, yet conceptually they differ from our objective: they calibrate a decision rule built on top of a fixed predictor, whereas our method changes the predictor itself so as to reduce confident errors during learning.

A fourth line of work seeks more faithful \emph{predictive uncertainty estimates}. Bayesian approximations such as MC dropout and ensemble-based methods such as deep ensembles aim to represent epistemic uncertainty more faithfully than a single deterministic network \citep{gal2016dropout,lakshminarayanan2017simple,kendall2017uncertainties}. Large-scale evaluations have shown, however, that uncertainty quality can deteriorate substantially under dataset shift, even for strong approximate Bayesian and ensemble methods \citep{ovadia2019can}. Closely related is the literature on out-of-distribution and failure detection, where maximum softmax probability, ODIN, Mahalanobis-based scores, energy-based methods, and Outlier Exposure are all used to detect inputs on which the model is less trustworthy \citep{hendrycks17baseline,liang2018enhancing,lee2018simple,liu2020energy,hendrycks2018deep}. These approaches are important because they help identify when the model may fail, but they still treat uncertainty estimation or anomaly detection as the main object. They do not directly train the classifier to satisfy a user-specified bound on the event of confident misclassification.

On the other hand, probabilistic-constrained optimization has appeared in machine learning, but mostly in specialized settings rather than in modern confidence-aware classifier training. Early work focused on uncertain or robust support vector machine formulations, where probabilistic constraints were used to control misclassification under input uncertainty \citep{wang2018robust,khanjani2023distributionally,LIN2024106755}. More recent examples include sequence optimization, safe reinforcement learning, and deep metric learning, where probabilistic constraints are used to enforce validity, safety, or feasibility with high probability \citep{pmlr-v119-liu20i,chow2018risk,pmlr-v202-wang23as,10484253}. These works demonstrate that probabilistic-constrained optimization is a viable tool for risk control in deep learning, but they target different failure modes and typically rely on specialized formulations or problem-specific approximations. By contrast, our setting is standard probabilistic classification, and our goal is to control the probability of a model being simultaneously confident and wrong during end-to-end training.

\medskip

\noindent Relative to these literatures, ReliableNet is best viewed as a training-time reliability method. It is closer to selective prediction and calibration-aware learning than to post-hoc conformal adjustment or general-purpose uncertainty estimation, but differs from all of them in one essential respect: it treats \emph{high-confidence error itself} as the quantity to be controlled during optimization. In that sense, our method is not primarily a recalibration procedure, an abstention wrapper, or an uncertainty estimator; it is a classifier training objective designed to suppress a concrete tail-risk event that is operationally meaningful in deployment.

\section{Problem Definition}
\label{sec_formulation}
As discussed above, the high-confidence misclassification event is the main failure event to consider.
In this part, we formulate a probabilistic-constrained optimization learning problem that minimizes the empirical risk while directly controlling the probability of confident misclassifications, and derive its consequences for selective prediction and calibration.

\subsection{Setup}

Let $(\mathcal{X}, \mathcal{B}_{\mathcal{X}})$ be a
measurable input space,
$\mathcal{Y} = \{1, \dots, K\}$ a finite label space, and $Z = (X, Y) \sim P$ a random input-output pair drawn from an unknown reference distribution $P$. All probabilities and expectations are with respect to $P$. A parametric classifier $f_\theta : \mathcal{X} \to \mathbb{R}^K$,
$\theta \in \Theta \subseteq \mathbb{R}^d$, produces logits from which we derive class probabilities, predictions, and confidence, respectively for $k\in \{1,\cdots,K\}$:
\begin{equation}
    p_\theta(k \mid x)
    :=  \frac{\exp(f_\theta(x)_k)}{\sum_{j=1}^{K} \exp(f_\theta(x)_j)},\qquad \hat{y}_\theta(x) := \arg\max_k  p_\theta(k \mid x), \qquad c_\theta(x) := \max_k  p_\theta(k \mid x).
\end{equation}

We maintain the following assumptions throughout.

\begin{assumption}[Measurability]
\label{ass:measurability}
$f_\theta$, $\hat{y}_\theta$, and $c_\theta$ are Borel measurable for every $\theta \in \Theta$.
\end{assumption}

\begin{assumption}[Parameter space]
\label{ass:compact}
$\Theta \subseteq \mathbb{R}^d$ is nonempty. Where uniform convergence is invoked, we take $\Theta$ to be compact.
\end{assumption}

% \begin{assumption}[Integrability]
% \label{ass:integrability}
% The smoothed surrogates defined in Section~\ref{sec_method}
% are $P$-integrable for every $\tau > 0$ and
% $\theta \in \Theta$.
% \end{assumption}

\subsection{The Confident Misclassification Event}

For a classifier $f_\theta$ and a fixed confidence threshold $\varepsilon \in (1/K, 1)$ ($1/K$ because we want $\varepsilon$ to be greater than uniform confidence spread) , we are interested in the event that a random pair $Z = (X, Y) \sim P$ is simultaneously misclassified and assigned high confidence, that is, $\hat{y}_\theta(X) \neq Y$ and $c_\theta(X) \geq \varepsilon$. To refer to this event compactly, we introduce the notation
\begin{equation}
    \mathcal{C}_\theta(\varepsilon)
    := \{(x,y) \in \mathcal{X} \times \mathcal{Y} : c_\theta(x) \geq \varepsilon\}, \qquad \mathcal{B}_\theta := \{(x,y) \in \mathcal{X} \times \mathcal{Y} : \hat{y}_\theta(x) \neq y\},
\end{equation}
where $(x,y)$ is a generic point in $\mathcal{X} \times \mathcal{Y}$. These are deterministic level sets of the classifier $f_\theta$,
fully characterized by $\theta$ and $\varepsilon$. Under Assumption~\ref{ass:measurability}, both sets are Borel measurable, ensuring that the following probability is well-defined.
% : $\mathcal{C}_\theta(\varepsilon) = c_\theta^{-1}([\varepsilon,1])$
% is the preimage of a closed set under a measurable function, and
% $\mathcal{B}_\theta$ is the preimage of a measurable set under
% $\hat{y}_\theta$,
We define the \emph{joint confident-wrong} (JCW) probability as
\begin{equation}
\label{eq_jcw_def}
    \mathrm{JCW}(\theta)
    :=
    \Pr_{Z \sim P} \left(
      Z \in \mathcal{C}_\theta(\varepsilon) \cap \mathcal{B}_\theta
    \right),
\end{equation}
the probability, under the full data-generating distribution $P$, that a randomly drawn pair $Z \sim P$ is both misclassified and
highly confident. For any fixed $\theta \in \Theta$, the sets $\mathcal{C}_\theta(\varepsilon)$ and $\mathcal{B}_\theta$ are
deterministic, so the sole source of randomness in~\eqref{eq_jcw_def} is $Z \sim P$; consequently, $\mathrm{JCW}(\theta)$ is a deterministic scalar-valued function of $\theta$ alone, and the constraint $\mathrm{JCW}(\theta) \leq \alpha$ is a well-posed condition on the model parameters. The threshold $\varepsilon$ enters only as a pre-fixed hyperparameter.

JCW measures the probability of the specific failure mode that is most dangerous in decision-critical applications: the model is wrong \emph{and} reports high confidence. Unlike expected calibration error, which averages discrepancies across all confidence levels, JCW targets a single, operationally meaningful event. Unlike selective risk, which conditions on a selection decision, JCW is an unconditional joint probability over $P$. To encode the event $\{Z \in \mathcal{C}_\theta(\varepsilon) \cap \mathcal{B}_\theta\}$ we first used the multiclass margin defined as:
\begin{equation}
    \ell_\theta(x, y) := \max_{k \neq y}  f_\theta(x)_k - f_\theta(x)_y,
\end{equation}
which satisfies $\ell_\theta(x,y) > 0 \Leftrightarrow \hat{y}_\theta(x) \neq y$ under Assumption~\ref{ass:boundary}. Define the scalar violation function
\begin{equation}
\label{eq_g_def}
 g_\theta(Z) := \min \left\{  c_\theta(X) - \varepsilon,  \ell_\theta(X, Y) \right\}.
\end{equation}

\begin{assumption}[Continuity]
\label{ass:smooth_g}
For $P$-almost every $z$, the map $\theta \mapsto g_\theta(z)$ defined in~\eqref{eq_g_def} is continuous on $\Theta$.
\end{assumption}

\begin{assumption}[Boundary regularity]
\label{ass:boundary}
$Pr\{g_\theta(Z) = 0\} = 0$ for every $\theta \in \Theta$.
\end{assumption}

\noindent
Assumption~\ref{ass:boundary} excludes the measure-zero event of exact logit ties or exact threshold hits, which is satisfied generically for neural networks with continuous-valued logits under any non-degenerate input distribution.
\begin{corollary}[Exact encoding]
\label{corr_encoding}
Under Assumption~\ref{ass:boundary}, for $P$-almost every $z=(x,y)$, $ g_\theta(z) > 0 \quad\Longleftrightarrow\quad c_\theta(x) \geq \varepsilon \ \wedge\ \hat{y}_\theta(x) \neq y$,
 and consequently $\Pr\{g_\theta(Z) > 0\} = \mathrm{JCW}(\theta)$.
\end{corollary}

\begin{proof}
Recall $g_\theta(z)=\min\{\,c_\theta(x)-\varepsilon,\ \ell_\theta(x,y)\,\}$. By Assumption~\ref{ass:boundary} the tie set $\{g_\theta(Z)=0\}$ is $P$-null, so we may work on its complement $\{g_\theta\neq 0\}$. There, $g_\theta(z)\geq 0$
implies $g_\theta(z)>0$, a fact we use in the reverse direction below.

If $g_\theta(z)>0$, the minimum is positive, so \emph{both} terms are: hence $c_\theta(x)>\varepsilon$ (so in particular $c_\theta(x)\geq\varepsilon$) and $\ell_\theta(x,y)>0$. Thus the prediction is confident and, by the argmax rule, wrong ($\hat y_\theta(x)\neq y$).

Conversely, if the prediction is confident and wrong, then $c_\theta(x)-\varepsilon\geq 0$ and $\ell_\theta(x,y)>0$ (a misclassification
without an exact tie has strictly positive margin), so the minimum satisfies $g_\theta(z)\geq 0$; being off the tie set, $g_\theta(z)>0$. This proves the
equivalence $P$-almost everywhere. Since the events $\{g_\theta(Z)>0\}$ and $\{c_\theta(X)\geq\varepsilon,\ \hat y_\theta(X)\neq Y\}$ coincide up to a $P$-null set, they have equal probability, giving $\Pr\{g_\theta(Z)>0\} = \mathrm{JCW}(\theta)$.
\end{proof}

\begin{remark}[Choice of the minimum]
Writing $g_1 := c_\theta(X) - \varepsilon$ and $g_2 := \ell_\theta(X,Y)$, the reliability requirement $P(g_1 > 0 \wedge g_2 > 0) \leq \alpha$ is equivalent to $P(\min\{g_1,g_2\} > 0) \leq \alpha$, since the minimum is positive whenever both arguments are positive.
This reduces the intersection of two events to a
standard single probabilistic constraint.
% Using the maximum would instead require the model to be
% simultaneously unconfident and correct, penalizing
% the behaviour a reliable classifier should exhibit.
\end{remark}

\subsection{The Probabilistic-Constrained Learning Problem}

The learning problem we propose is

\begin{equation}
\label{eq_population_problem}
    \min_{\theta \in \Theta}
    \mathbb{E}~\bigg[
      \mathcal{L}~\big(Y,  p_\theta(\cdot \mid X)\big) \bigg]+\mathcal{R}(\theta)
    \quad
    \text{subject to}
    \quad
    \mathrm{JCW}(\theta) \leq \alpha,
\end{equation}

where $\mathcal{L}$ is a classification loss (cross entropy, Brier score, etc.), $\mathcal{R(\theta)}$ is a regularizer (Weight decay in our case), $\alpha \in (0,1)$ is a prescribed risk budget and, by Corollary ~\ref{corr_encoding}, the constraint is equivalently $Pr\{g_\theta(Z) \leq 0\} \geq 1 - \alpha$. The threshold $\varepsilon$ is a data-driven threshold selected from the validation set.

Problem~\eqref{eq_population_problem} differs from ERM in a fundamental way: ERM controls the expected loss, which is silent about the distribution of errors across confidence levels, whereas~\eqref{eq_population_problem} supplements expected loss with an explicit constraint on the joint occurrence of high confidence and incorrect prediction.

\subsection{Properties of the JCW Constraint}

All statements below are population-level identities or bounds and do not depend on how~\eqref{eq_population_problem} is solved. Fix $\varepsilon$ and abbreviate the coverage $\mathrm{Cov}(\theta)=\Pr\{c_\theta(X)\geq\varepsilon\}$ and the high-confidence accuracy $\mathrm{Acc}_{\mathrm{hc}}(\theta)=\Pr\{\hat{y}_\theta(X)=Y\mid c_\theta(X)\geq\varepsilon\}$. The rest follows from a single identity, obtained by conditioning on whether the model acts.

\begin{corollary}[Decomposition]
\label{corr_decomposition}
For every $\theta$,
\begin{equation}
\label{eq_jcw_decomp}
    \mathrm{JCW}(\theta)=\mathrm{Cov}(\theta)\bigl(1-\mathrm{Acc}_{\mathrm{hc}}(\theta)\bigr).
\end{equation}
and $\mathrm{JCW}(\theta)\leq\alpha$ holds iff
$\mathrm{Acc}_{\mathrm{hc}}(\theta)\geq 1-\alpha/\mathrm{Cov}(\theta)$.
\end{corollary}

\begin{proof}
\begin{align}
   \mathrm{JCW}(\theta)&=\Pr\{c_\theta\geq\varepsilon, \hat y\neq Y\} = \Pr\{\hat y\neq Y,c_\theta\geq\varepsilon\}\\
   &= \Pr\{\hat y\neq Y\mid c_\theta\geq\varepsilon\}\Pr\{c_\theta\geq\varepsilon\} \\
   &=\big(1-\mathrm{Acc}_{\mathrm{hc}}(\theta)\big)\,\mathrm{Cov}(\theta)
\end{align}
Therefore $\mathrm{JCW}(\theta)\leq\alpha$ $\Leftrightarrow$ $\big(1-\mathrm{Acc}_{\mathrm{hc}}(\theta)\big)\,\mathrm{Cov}(\theta)\leq\alpha$ $\Leftrightarrow$  $\mathrm{Acc}_{\mathrm{hc}}(\theta)\geq 1-\alpha/\mathrm{Cov}(\theta)$ by rearranging 
\end{proof}

\begin{remark}
\label{rem_throughput}
Writing $T(\theta):=\Pr\{c_\theta(X)\geq\varepsilon,\ \hat y_\theta(X)=Y\}$ for the confident-correct \emph{throughput}, the acceptance region splits as $\mathrm{Cov}(\theta)=T(\theta)+\mathrm{JCW}(\theta)$.  Reliability therefore requires small $\mathrm{JCW}$ \emph{and} large $T$, and the two are supplied by the two terms of problem ~\eqref{eq_population_problem}: the loss minimization rewards confident-correct predictions and pushes $T$ up, while the constraint caps confident errors. Because the JCW constraint alone does not impose a minimum coverage, a feasible model may reduce confident errors partly by lowering its acceptance rate. We therefore report coverage and confident-correct throughput alongside JCW.
% Note also that
% $\mathrm{JCW}(\theta)\leq\min\{\Pr(\hat y\neq Y),\mathrm{Cov}(\theta)\}$: the
% constraint governs confident error alone, leaving low-confidence mistakes untouched,
% but those are flagged and recoverable.
\end{remark}

\begin{corollary}[Selective risk]
\label{cor:selective}
If $\mathrm{JCW}(\theta)\leq\alpha$ and $\mathrm{Cov}(\theta)>0$, the error rate on
accepted inputs satisfies
\begin{equation}
\label{eq_selective_risk}
    \Pr\{\hat y_\theta(X)\neq Y\mid c_\theta(X)\geq\varepsilon\}
\leq\frac{\alpha}{\mathrm{Cov}(\theta)}.
\end{equation}
\end{corollary}

\begin{proof}

\begin{align}
    \Pr\{\hat y_\theta(X)\neq Y\mid c_\theta(X)\geq\varepsilon\}&=\dfrac{\Pr\{\hat y_\theta(X)\neq Y, c_\theta(X)\geq\varepsilon\}}{\Pr\{c_\theta\geq\varepsilon\}}\\
    &=\dfrac{JCW(\theta)}{\mathrm{Cov}(\theta)} \le \frac{\alpha}{\mathrm{Cov}(\theta)}
\end{align}
\end{proof}

Thus $\varepsilon$ already acts as a selection rule since  it defines an acceptance region $ \mathcal{C}_\theta(\varepsilon)$ and bounds the risk there, with no gating network and no separate abstention loss. 

\begin{remark}[Connection to Bayes-optimal rejection]
\label{rem_bayes}
Let $\eta(x)=\max_k \Pr(Y=k\mid X=x)$. By Chow's rule, thresholding $\eta$ yields a Bayes-optimal acceptance set at fixed coverage. Hence, when $c_\theta=\eta$ and $\hat y_\theta$ is Bayes-optimal, $\{c_\theta\geq\varepsilon\}$ is optimal and JCW equals the accepted error probability. Under misspecification, JCW still measures confident error on the model's own acceptance region, but controlling it does not guarantee recovery of the Bayes confidence ranking.
\end{remark}

Modern networks are typically overconfident, with the confidence-accuracy gap concentrated at high confidence \citep{guo2017calibration}, though its magnitude is architecture-dependent \citep{minderer2021revisiting}. This is exactly the region a JCW budget constrains.

\begin{proposition}
\label{prop_calibration}
Let $\mathrm{WCE}_{\mathrm{hc}}(\theta)=\mathrm{Cov}(\theta)\, \bigl|\mathbb E[c_\theta\mid c_\theta\geq\varepsilon]-\mathrm{Acc}_{\mathrm{hc}}(\theta)\bigr|$ be the coverage-weighted calibration gap on accepted inputs. If the model is overconfident there, $\mathbb E[c_\theta\mid c_\theta\geq\varepsilon]\geq\mathrm{Acc}_{\mathrm{hc}}(\theta)$, and $\mathrm{JCW}(\theta)\leq\alpha$, then $\mathrm{WCE}_{\mathrm{hc}}(\theta)\leq\alpha$
\end{proposition}

\begin{proof}
Using the fact that overconfidence removes the absolute value, and $c_\theta\leq 1$ we find:
\[
    \mathrm{WCE}_{\mathrm{hc}}(\theta)
    =\mathrm{Cov}(\theta)\bigl(\mathbb E[c_\theta\mid c_\theta\geq\varepsilon]-\mathrm{Acc}_{\mathrm{hc}}\bigr)
    \leq\mathrm{Cov}(\theta)\bigl(1-\mathrm{Acc}_{\mathrm{hc}}\bigr)
    =\mathrm{JCW}(\theta)\leq\alpha,
\]
using corollary ~\ref{corr_decomposition} and the constraint. 
\end{proof}

\noindent This local guarantee  partly explains the calibration ReliableNet gains as a by-product, since overconfidence is suppressed exactly where the constraint is active.

\section{Methodology}
\label{sec_method}
In this part, we describe the smooth approximation used to solve the  problem~\eqref{eq_population_problem} that was initially not amenable to gradient-based optimization due to the probabilistic constraint $\mathrm{JCW}(\theta)\le\alpha$. We also explore the finite sample feasibility of the smooth approximation and the constrained-optimization technique used to solve our problem.

\subsection{Inner-Outer Analytic Approximation}
\label{sec_inner_outer}
 Replace $\mathrm{JCW}(\theta)\le\alpha$ by a smooth surrogate constraint, adopting the inner-outer smoothing framework of \citet{doi:10.1137/15M1049750}.  Consider the Geletu-Hoffman parametric function
\begin{equation}
    \zeta(\tau,s)=\frac{1+m_1\tau}{1+m_2\tau\exp(-s/\tau)},
\qquad \tau\in(0,1),\quad  s\in \mathbb{R}, \quad 0<m_2<\dfrac{m_1}{1+m_1},
\label{eq_zeta_fs}
\end{equation}

together with its reflection $\Pi(\tau,s):=\zeta(\tau,-s)$.  The family is bounded, dominates the Heaviside function $h(s):=\mathbf 1\{s\ge0\}$ pointwise ($\zeta(\tau,s)\ge h(s)$ for all $s$), and converges to  pointwise as $\tau\downarrow0$. %The ordering $m_1>m_2$ is exactly what makes the domination hold: since $\zeta(\tau,\cdot)$ is increasing, the binding case is $s=0$, where $\zeta(\tau,0)=\frac{1+m_1\tau}{1+m_2\tau}>1$ iff $m_1>m_2$.

\paragraph{Surrogates and the estimation bracket.}
Define the smoothed surrogate functions
\[
\psi_\tau(\theta):=\mathbb{E} \left[\zeta(\tau,g_\theta(Z))\right],
\qquad
\phi_\tau(\theta):=\mathbb{E} \left[\Pi(\tau,g_\theta(Z))\right].
\]
By Lemma~\ref{lem:family_props}, $\zeta(\tau,s)\ge\mathbf 1\{s\ge0\}$ and, applying it at $-s$, $\Pi(\tau,s)=\zeta(\tau,-s)\ge\mathds 1\{s\le0\}$. Taking  expectation  under $Z\sim P$ gives the estimation bracket
\begin{equation}
\label{eq_io_bracket}
1-\psi_\tau(\theta)
\;\le\;
\mathbb{P}(g_\theta(Z)\le 0)
\;\le\;
\phi_\tau(\theta),
\qquad \forall \tau>0,
\end{equation}
which holds pointwise in $\theta$ for any $P$. In particular $\psi_\tau(\theta)\ge\mathbb{P}(g_\theta(Z)\ge0)\ge\mathrm{JCW}(\theta)$,
($\mathbb{P}\big(g_\theta(Z)\ge0\big)=\mathrm{JCW}(\theta)$ under Assumption~\ref{ass:boundary})\citep{doi:10.1137/15M1049750}, hence the conservative implication:
\begin{equation}
\label{eq_conservative_implication}
\psi_\tau(\theta)\le\alpha
\;\Longrightarrow\;
\mathrm{JCW}(\theta)\le\alpha .
\end{equation}

\paragraph{Inner and outer feasible sets.}
Let $\mathcal{F}:=\{\theta\in\Theta:\ \mathbb{P}(g_\theta(Z)\le0)\ge1-\alpha\}$ denote the feasible set of the probabilistic constraint, and define
\[
\mathcal{F}^{\mathrm{in}}_\tau:=\{\theta:\ \psi_\tau(\theta)\le\alpha\},
\qquad
\mathcal{F}^{\mathrm{out}}_\tau:=\{\theta:\ \phi_\tau(\theta)\ge1-\alpha\}.
\]
The bracket~\eqref{eq_io_bracket} yields the inclusion directly: if $\psi_\tau(\theta)\le\alpha$ then $\mathbb{P}(g_\theta(Z)\le0)\ge1-\psi_\tau(\theta)\ge1-\alpha$, so $\theta\in\mathcal{F}$; if $\theta\in\mathcal{F}$ then $\phi_\tau(\theta)\ge\mathbb{P}(g_\theta(Z)\le0)\ge1-\alpha$, so $\theta\in\mathcal{F}^{\mathrm{out}}_\tau$. Therefore,
\begin{equation}
\label{inclusion_set}
\mathcal{F}^{\mathrm{in}}_\tau
\subseteq
\mathcal{F}
\subseteq
\mathcal{F}^{\mathrm{out}}_\tau,
\qquad \forall \tau>0.
\end{equation}

This gives rise to two approximation problems
\begin{align}
\text{(PIA}_\tau\text{)}\qquad
&\min_{\theta\in\Theta}
\mathbb{E}~\bigg[
      \mathcal{L}~\big(Y,  p_\theta(\cdot \mid X)\big) \bigg]+\mathcal{R}(\theta)
\quad\text{s.t.}\quad
\psi_\tau(\theta)\le\alpha,
\label{eq_pia_theta}\\[2mm]
\text{(POA}_\tau\text{)}\qquad
&\min_{\theta\in\Theta}
\mathbb{E}~\bigg[
      \mathcal{L}~\big(Y,  p_\theta(\cdot \mid X)\big) \bigg]+\mathcal{R}(\theta)
\quad\text{s.t.}\quad
\phi_\tau(\theta)\ge1-\alpha.
\label{eq_poa_theta}
\end{align}

By~\eqref{inclusion_set}, any feasible point of $(\mathrm{PIA}_\tau)$ is feasible for the original probabilistic constraint, making the inner problem the vehicle for reliable solutions; $(\mathrm{POA}_\tau)$ relaxes the feasible set and yields a lower bound on the optimal objective value, used purely as a diagnostic.

% Note that $(\mathrm{PIA}_\tau)$ is a population-level problem:
% Algorithm~\ref{alg:inner_cc} enforces its empirical counterpart on finite
% data, and the transfer of feasibility from the empirical to the
% population constraint at level $\alpha$ is established in
% Theorem~\ref{thm_finite_sample}.

\paragraph{A pointwise stopping diagnostic.}
In practice \cite{doi:10.1137/15M1049750} solve  both $(\mathrm{POA}_\tau)$ and $(\mathrm{PIA}_\tau)$ for each decaying $\tau$ and compare their objective values, but solving them at every smoothing level to monitor the objective-value gap is expensive. We instead use a cheap diagnostic requiring only two forward surrogate evaluations at the incumbent which measures the width of the
probability enclosure at the evaluated parameter.

% which rests only on the bracket
% and the boundedness/limit properties of Lemma~\ref{lem:family_props}, 
% in particular it uses neither monotonicity of $\zeta$ in $\tau$ nor any
% density assumption.

\begin{proposition}[Gap diagnostic]
\label{prop_gap_stopping}
Let Assumption~\ref{ass:boundary} ($\mathbb{P}\{g_\theta(Z)=0\}=0$) holds. Write $p(\theta):=\mathbb{P}(g_\theta(Z)\le0)$ and
\begin{equation}
\label{eq_gap_def}
\Delta_\tau(\theta):=\phi_\tau(\theta)-\bigl(1-\psi_\tau(\theta)\bigr)\ \ (\ge 0).
\end{equation}
Then, for every $\theta$ and every $\tau\in(0,1)$,
\begin{equation}
\label{eq_enclosure}
p(\theta)\in\bigl[ 1-\psi_\tau(\theta),\ \phi_\tau(\theta) \bigr],
\qquad\text{so}\qquad
\max\Bigl\{\bigl|p(\theta)-(1-\psi_\tau(\theta))\bigr|,\ \bigl|p(\theta)-\phi_\tau(\theta)\bigr|\Bigr\}\le\Delta_\tau(\theta).
\end{equation}
Moreover, for each fixed $\theta$,
\begin{equation}
\label{eq_gap_limit}
\lim_{\tau\downarrow0}\Delta_\tau(\theta)=0 .
\end{equation}
\end{proposition}

\begin{proof}
\emph{Enclosure~\eqref{eq_enclosure}.} The bracket~\eqref{eq_io_bracket} states $1-\psi_\tau(\theta)\le p(\theta)\le\phi_\tau(\theta)$, so
$p(\theta)$ lies in the stated interval, whose width is $\Delta_\tau(\theta)\ge0$; the two-sided bound is immediate. \emph{Limit~\eqref{eq_gap_limit}.} Fix $\theta$. By Lemma~\ref{lem:family_props}, $\zeta(\tau,g_\theta(z))\to h(g_\theta(z))$ as $\tau\downarrow0$ for every $z$, and $\zeta(\tau,\cdot)\le 1+m_1$ for $\tau\in(0,1)$, a constant $P$-integrable majorant. Dominated convergence gives
\[
\psi_\tau(\theta)\to\mathbb{E}[h(g_\theta(Z))]=\mathbb{P}(g_\theta(Z)\ge0),
\qquad
\phi_\tau(\theta)=\mathbb{E}[\zeta(\tau,-g_\theta(Z))]\to\mathbb{P}(g_\theta(Z)\le0)=p(\theta).
\]
Hence $1-\psi_\tau(\theta)\to\mathbb{P}(g_\theta(Z)<0)$, which equals $p(\theta)$ by Assumption~\ref{ass:boundary} ($\mathbb{P}\{g_\theta(Z)=0\}=0$). Therefore $\Delta_\tau(\theta)=\phi_\tau(\theta)-(1-\psi_\tau(\theta))\to p(\theta)-p(\theta)=0$.
\end{proof}

The enclosure~\eqref{eq_enclosure} provides a pointwise tightness diagnostic: for fixed $\hat\theta$, $\Delta_\tau(\hat\theta)$ is the width
of an interval containing $p(\hat\theta)$. In practice, we solve only $(\mathrm{PIA}_\tau)$, evaluate $\Delta_\tau(\hat\theta_\tau)$ at the current incumbent, and tighten $\tau$ until this width falls below $\vartheta$ ($10^{-4}$ in our experiments). The returned model is always obtained from the inner approximation; $\Delta_\tau$ measures surrogate tightness, not optimization suboptimality.

\subsection{Finite-Sample Feasibility of the Inner Approximation}
\label{sec_finite_sample}

The inner inclusion in~\eqref{inclusion_set} is a population-level statement: if the true surrogate $\psi_\tau(\theta):=\mathbb{E}[\zeta(\tau,g_\theta(Z))]$ satisfies $\psi_\tau(\theta)\leq \alpha$, then the JCW probabilistic constraint is satisfied. In learning, however, the distribution $P$ is unknown and Algorithm~\ref{alg:inner_cc} only observes the empirical surrogate
\begin{equation}
\label{eq_empirical_surrogate}
  \widehat{\psi}_{\tau,n}(\theta)
  :=
  \frac1n\sum_{i=1}^{n}
  \zeta \left(\tau,g_\theta(Z_i)\right),
  \qquad
  Z_1,\ldots,Z_n \overset{\mathrm{i.i.d.}}{\sim} P .
\end{equation}
This section bridges empirical surrogate feasibility to population JCW feasibility at finite sample size. The difficulty is that the returned parameter $\widehat{\theta}$ is data-dependent, selected on the same sample used to evaluate $\widehat{\psi}_{\tau,n}$, so a pointwise bound at fixed $\theta$ is insufficient. We first establish uniform concentration over $\Theta$, apply it at the random output $\widehat{\theta}$, and finally give independent holdout certificates, including the Clopper-Pearson certificate used in the experiments. Throughout, the smoothing level $\tau \in (0,1)$ and the threshold $\varepsilon$ are fixed; since $\varepsilon$ is selected on a validation fold, the statements below are conditional on that value. We use the Geletu–Hoffman parametric function $\zeta(\tau, s)$ ~\ref{eq_zeta_fs}.

\begin{assumption}[Bounded parameter class]
\label{ass:fs_compact}
The parameter set satisfies
$\Theta\subseteq\{\theta\in\mathbb{R}^d:\|\theta\|_2\le R\}$ for some
$R<\infty$.
\end{assumption}

\begin{assumption}[Parameter-Lipschitz violation]
\label{ass:fs_lip}
There exists $L_g<\infty$ such that
 $|g_\theta(z)-g_{\theta'}(z)|
    \le L_g\|\theta-\theta'\|_2$
for all $\theta,\theta'\in\Theta$ and $P$-almost every $z$.
\end{assumption}
Assumption~\ref{ass:fs_lip} is a standard theoretical regularity condition: on compact parameter and input sets, neural-network logits are Lipschitz in the parameters; the confidence $c_\theta$, the margin $\ell_\theta$, and their minimum inherit this Lipschitz property.

\begin{lemma}[Boundedness and conservatism]
\label{lem:smoother_conservative}
For every $\tau\in(0,1)$ and $s\in\mathbb{R}$, $0<\zeta(\tau,s)\le b_\tau:=1+m_1\tau, \quad |\partial_s\zeta(\tau,s)| \le L_\zeta(\tau):=\frac{1+m_1\tau}{4\tau}.$ Moreover, under   $0<m_2<\dfrac{m_1}{1+m_1}$ we have $\zeta(\tau,s)\ge \mathbf{1}\{s\ge 0\}$ for all $s$, and therefore $\psi_\tau(\theta)\ge \mathrm{JCW}(\theta)$ for all $\theta\in\Theta$.
\end{lemma}

\begin{proof}
The denominator in~\eqref{eq_zeta_fs} belongs to $(1,\infty)$, so $0<\zeta(\tau,s)\le 1+m_1\tau=b_\tau$. Setting $w=m_2\tau e^{-s/\tau}$,
differentiation gives $\partial_s\zeta(\tau,s)=\frac{1+m_1\tau}{\tau}\frac{w}{(1+w)^2}$.
Since $\max_{w\ge0}w/(1+w)^2=1/4$, the derivative bound follows. The conservatism $\zeta(\tau,s)\ge\mathbf 1\{s\ge0\}$ is
Lemma~\ref{lem:family_props}(ii). Applying it pointwise to $s=g_\theta(Z)$ and integrating gives $\psi_\tau(\theta)\ \ge\ \Pr\{g_\theta(Z)\ge 0\}\ =\ \Pr\{g_\theta(Z)> 0\}\ =\ \mathrm{JCW}(\theta)$, where the middle equality uses $\Pr\{g_\theta(Z)=0\}=0$ (Assumption~\ref{ass:boundary}) and the last uses Corollary~\ref{corr_encoding}.
\end{proof}
% \begin{lemma}[Boundedness and conservatism]
% \label{lem:smoother_conservative}
% For every $\tau\in(0,1)$,  $0<m_2<\dfrac{m_1}{1+m_1}$, and $s\in\mathbb{R}$,
% \[  0<\zeta(\tau,s)\le b_\tau:=1+m_1\tau, \qquad |\partial_s\zeta(\tau,s)| \le L_\zeta(\tau):=\frac{1+m_1\tau}{4\tau}.
% \]
% Moreover, $\zeta(\tau,s)\ge \mathbf{1}\{s>0\}$ for all $s$, and therefore $\psi_\tau(\theta)\ge \mathrm{JCW}(\theta)$ for all $\theta\in\Theta$.
% \end{lemma}

% \begin{proof}
% The denominator in~\eqref{eq_zeta_fs} belongs to $(1,\infty)$, so $0<\zeta(\tau,s)\le 1+m_1\tau=b_\tau$. Setting $w=m_2\tau e^{-s/\tau}$, differentiation gives
% \[ \partial_s\zeta(\tau,s) =\frac{1+m_1\tau}{\tau}\frac{w}{(1+w)^2}.\]
% Since $\max_{w\ge0}w/(1+w)^2=1/4$, the derivative bound follows. Also $\partial_s\zeta(\tau,s)>0$, so $\zeta$ is strictly increasing
% in $s$. If $s\le0$, then $\zeta(\tau,s)>0=\mathbf{1}\{s>0\}$. If $s>0$, monotonicity gives
% \[ \zeta(\tau,s)>\zeta(\tau,0) =  \frac{1+m_1\tau}{1+m_2\tau}>1.\]
% Thus $\zeta(\tau,s)\ge \mathbf{1}\{s>0\}$ for all $s$. Applying this pointwise to $s=g_\theta(Z)$ and integrating gives
% \[ \psi_\tau(\theta) \ge \Pr\{g_\theta(Z)>0\} = \mathrm{JCW}(\theta),\]
% where the last equality follows from Corollary~\ref{corr_encoding}.
% \end{proof}

\begin{theorem}[Uniform concentration of the empirical surrogate]
\label{thm_concentration}
Under Assumptions~\ref{ass:fs_compact}-\ref{ass:fs_lip}, for any $\rho\in(0,1)$, with probability at least $1-\rho$,
\begin{equation}
\label{eq_uniform_concentration}
  \sup_{\theta\in\Theta} \left| \widehat{\psi}_{\tau,n}(\theta)-\psi_\tau(\theta) \right| \le \delta_n(\tau,\rho),
\end{equation}
where
\begin{equation}
\label{eq_delta_n}
  \delta_n(\tau,\rho) := b_\tau \sqrt{  \frac{d\ln(3n)+\ln(2/\rho)}{2n} }+ \frac{(1+m_1\tau)L_gR}{2\tau n}.
\end{equation}
\end{theorem}

\begin{proof}
Let $\Phi(\theta;z):=\zeta(\tau,g_\theta(z))$. By Lemma~\ref{lem:smoother_conservative} and Assumption~\ref{ass:fs_lip},
\[ |\Phi(\theta;z)-\Phi(\theta';z)| \le L_\zeta(\tau)L_g\|\theta-\theta'\|_2 =:L_\tau\|\theta-\theta'\|_2,
\]
where $L_\tau=(1+m_1\tau)L_g/(4\tau)$. Hence $D(\theta):=\widehat{\psi}_{\tau,n}(\theta)-\psi_\tau(\theta)$ is $2L_\tau$-Lipschitz.

Let $\{\theta_1,\ldots,\theta_N\}$ be a $\kappa$-net of $\Theta$. Since $\Theta$ is contained in a radius-$R$ ball (assumption\ref{ass:fs_compact}), $N\le(3R/\kappa)^d$. For every $\theta\in\Theta$, choose $\theta_j$ with $\|\theta-\theta_j\|_2\le\kappa$. Then
\[ \sup_{\theta\in\Theta}|D(\theta)| \le \max_{1\le j\le N}|D(\theta_j)|+2L_\tau\kappa . \]
For fixed $\theta_j$, the variables $\Phi(\theta_j;Z_i)$ are i.i.d., bounded in $[0,b_\tau]$, and have mean $\psi_\tau(\theta_j)$. Hoeffding's inequality gives
\[ \Pr\{|D(\theta_j)|\ge t\}  \le 2\exp \left(-\frac{2nt^2}{b_\tau^2}\right).
\]
A union bound over the net gives
\[ \Pr\left\{\max_j |D(\theta_j)|\ge t\right\} \le 2N\exp \left(-\frac{2nt^2}{b_\tau^2}\right).\]
Setting the right-hand side equal to $\rho$ yields
$t=b_\tau\sqrt{\ln(2N/\rho)/(2n)}$. Therefore, with probability at
least $1-\rho$,
\[ \sup_{\theta\in\Theta}|D(\theta)|  \le  b_\tau\sqrt{\frac{\ln(2N/\rho)}{2n}}+2L_\tau\kappa . \]
Taking $\kappa=R/n$ gives $N\le(3n)^d$ and $2L_\tau\kappa=(1+m_1\tau)L_gR/(2\tau n)$, which proves
\eqref{eq_uniform_concentration}.
\end{proof}

\begin{theorem}[Finite-sample feasibility for data-dependent outputs]
\label{thm_finite_sample}
Suppose Assumptions~\ref{ass:fs_compact}-\ref{ass:fs_lip} hold. Fix $\tau\in(0,1)$, $\alpha\in(0,1)$, and $\rho\in(0,1)$. With probability at least $1-\rho$, every $\theta\in\Theta$ satisfying
\begin{equation}
\label{eq_empirical_feasibility_slack}
    \widehat{\psi}_{\tau,n}(\theta)
    \le
    \alpha-\delta_n(\tau,\rho)
\end{equation}
also satisfies $\mathrm{JCW}(\theta)\le\alpha$. Consequently, the same implication holds for any data-dependent output $\widehat{\theta}=\mathcal{A}(Z_1,\ldots,Z_n)$ returned by Algorithm~\ref{alg:inner_cc}.
\end{theorem}

\begin{proof}
Let $\mathcal{E}_n$ be the event on which \eqref{eq_uniform_concentration} holds. By Theorem~\ref{thm_concentration}, $\Pr(\mathcal{E}_n)\ge1-\rho$. On $\mathcal{E}_n$, for every $\theta\in\Theta$, \[ \psi_\tau(\theta) \le \widehat{\psi}_{\tau,n}(\theta)+\delta_n(\tau,\rho). \]
Thus any $\theta$ satisfying \eqref{eq_empirical_feasibility_slack} also satisfies $\psi_\tau(\theta)\le\alpha$. By Lemma~\ref{lem:smoother_conservative}, \[ \mathrm{JCW}(\theta) \le \psi_\tau(\theta)  \le \alpha .\] Because the event $\mathcal{E}_n$ holds uniformly for all $\theta\in\Theta$, the same argument applies to a random $\widehat{\theta}$ selected using the sample. This proves the claim.
\end{proof}

\begin{theorem}[Independent holdout certificates]
\label{thm_holdout_certificates}
Let $\widehat{\theta}$, $\varepsilon$, and $\tau$ be fixed before evaluating an independent certification sample $\widetilde Z_1,\ldots,\widetilde Z_m\overset{\mathrm{i.i.d.}}{\sim}P$. Then the following certificates hold.
\begin{itemize}
    \item[(i)]  Hoeffding surrogate certificate. Let $\widehat{\psi}^{ \mathrm{cert}}_{\tau,m}(\widehat{\theta}) :=  m^{-1}\sum_{i=1}^{m} \zeta \left(\tau,g_{\widehat{\theta}}(\widetilde Z_i)\right).$ With probability at least $1-\rho$,
$\psi_\tau(\widehat{\theta})  \le  \widehat{\psi}^{ \mathrm{cert}}_{\tau,m}(\widehat{\theta})  +  b_\tau\sqrt{\frac{\ln(1/\rho)}{2m}}.$
Consequently, if the right-hand side is at most $\alpha$, then $\mathrm{JCW}(\widehat{\theta})\le\alpha$ with probability at least
$1-\rho$.

\item[(ii)]  Clopper-Pearson hard-JCW certificate.
Define $B_i := \mathbf{1} \left\{ c_{\widehat{\theta}}(\widetilde X_i)\ge \varepsilon,   \widehat y_{\widehat{\theta}}(\widetilde X_i)  \neq \widetilde Y_i \right\}, \quad K_m:=\sum_{i=1}^{m}B_i$. Conditional on the fixed model and threshold, the $B_i$ are i.i.d. Bernoulli with success probability $q=\mathrm{JCW}(\widehat{\theta})$. Let
\[ U_{\mathrm{CP}}(K_m,m;\rho) :=
    \begin{cases}
    \beta^{-1}(1-\rho; K_m+1, m-K_m), & K_m<m,\\ 1,  & K_m=m,
    \end{cases} \]
where $\beta^{-1}(\cdot;a,b)$ is the quantile function of the $\operatorname{Beta}(a,b)$ distribution. Then \[ \Pr \left\{ \mathrm{JCW}(\widehat{\theta}) \le  U_{\mathrm{CP}}(K_m,m;\rho)   \right\}  \ge  1-\rho .\] Thus, if $U_{\mathrm{CP}}(K_m,m;\rho)\le\alpha$, the hard JCW constraint is certified at confidence level $1-\rho$.
\end{itemize}
\end{theorem}

\begin{proof}
For (i), condition on the fixed $\widehat{\theta}$, $\varepsilon$, and
$\tau$. The variables
$W_i:=\zeta(\tau,g_{\widehat{\theta}}(\widetilde Z_i))$ are i.i.d.,
bounded in $[0,b_\tau]$, and have mean $\psi_\tau(\widehat{\theta})$.
The one-sided Hoeffding inequality gives
\[
\Pr \left\{ \psi_\tau(\widehat{\theta}) -  \widehat{\psi}^{ \mathrm{cert}}_{\tau,m}(\widehat{\theta}) \ge t \right\} \le  \exp \left(-\frac{2mt^2}{b_\tau^2}\right).\]
Taking $t=b_\tau\sqrt{\ln(1/\rho)/(2m)}$ gives the stated bound. If this upper bound is at most $\alpha$, then $\psi_\tau(\widehat{\theta})\le\alpha$, and Lemma~\ref{lem:smoother_conservative} implies $\mathrm{JCW}(\widehat{\theta})\le\alpha$.

For (ii), condition on the fixed model and threshold. The certification sample is independent, so $B_1,\ldots,B_m$ are i.i.d. Bernoulli with
parameter $q=\mathrm{JCW}(\widehat{\theta})$. The one-sided Clopper-Pearson interval is obtained by inverting the binomial test and
satisfies, for every $q\in[0,1]$,
\[ \Pr_q\{q\le U_{\mathrm{CP}}(K_m,m;\rho)\}\ge 1-\rho. \]
Substituting $q=\mathrm{JCW}(\widehat{\theta})$ proves the certificate.
\end{proof}

\begin{remark}
The uniform theorem explains why empirical surrogate feasibility can imply population JCW feasibility even for a data-dependent $\widehat{\theta}$. In experiments, however, we use the hard-JCW Clopper-Pearson certificate because it directly confirms the deployment event of interest and does not depend on the smooth surrogate or on the parameter dimension or the optimization technique.

\end{remark}

\subsection{Solving the Constrained Problem}
Since $\psi_\tau(\theta)$ in~\eqref{eq_pia_theta} is nonconvex, $(\mathrm{PIA}_\tau)$ is a nonconvex constrained problem, and direct projection onto $\mathcal{F}^{\mathrm{in}}_\tau$ is impractical. For a fixed smoothing level $\tau$, we form the Lagrangian \[ \mathfrak{L}_\tau(\theta,\lambda) = \mathbb{E} \left[ \mathcal{L}\bigl(Y,p_\theta(\cdot\mid X)\bigr) \right] +R(\theta) +\lambda\bigl(\psi_\tau(\theta)-\alpha\bigr), \qquad \lambda\in[0,\lambda_{\max}]. \]  The Lagrangian has a nonconvex-concave minimax structure: it is generally nonconvex in $\theta$ and linear, hence concave, in the dual variable $\lambda$. This motivates a primal-dual update in which the classifier is optimized by descent while the multiplier responds by projected ascent to constraint violations. Two-timescale gradient descent-ascent methods admit stationarity guarantees for related nonconvex-concave minimax problems under suitable regularity conditions \citep{lin2020gradient,lin2025two,ramirez2025feasible}. These modifications are motivated by the minimax structure but fall outside the exact assumptions of the cited convergence results; accordingly, Algorithm~\ref{alg:inner_cc} returns a candidate solution whose JCW feasibility is verified independently on the held-out certification sample.

\begin{algorithm}[H]
\caption{Pseudo Code: ReliableNet Training}
\label{alg:inner_cc}
\begin{algorithmic}[1]

\Require Warm-started parameters $\theta_0$; smoothing levels
$\tau_0,\tau_{\min}$; risk budget $\alpha$; primal and dual step sizes
$\eta_\theta,\eta_\lambda$; EMA factor $\beta\in[0,1]$;
multiplier cap $\lambda_{\max}$; bracket-width tolerance
$\delta_{\mathrm{tol}}$; decay $\gamma\in(0,1)$;
epochs per stage $E$; maximum stages $J_{\max}$

\Ensure Candidate parameters $\theta$; feasibility assessed post hoc by
Theorem~\ref{thm_holdout_certificates}

\State $\theta\gets\theta_0$, $\lambda\gets0$,
$\bar v\gets0$, $\tau\gets\tau_0$

\For{$j=1,\ldots,J_{\max}$}
    \Comment{continuation over smoothing levels}

    \For{$e=1,\ldots,E$}
        \For{each mini-batch $\{(x_i,y_i)\}_{i=1}^{B}$}
            \State $\widehat{\psi}_\tau \gets
            \frac{1}{B}\sum_{i=1}^{B} \zeta\left(\tau,g_\theta(x_i,y_i)\right),\;$  $\;\widehat{\mathcal L} \gets  \frac{1}{B}\sum_{i=1}^{B} \ell(x_i,y_i,\theta)$
            \State $v \gets  \dfrac{\widehat{\psi}_\tau-\alpha} {\max(\alpha,10^{-2})}$
            \Comment{normalized violation}
            \State $\theta \gets  \theta-\eta_\theta \nabla_\theta \left(  \widehat{\mathcal L}  +R(\theta) +\lambda v \right)$
            \Comment{primal descent}
            \State $\bar v \gets  \beta\bar v+(1-\beta)v$
            \Comment{EMA stabilization}
            \State $\lambda \gets \operatorname{clip}\left( \lambda+  \eta_\lambda\frac{B}{n}\bar v,\, 0,\lambda_{\max} \right)$
            \Comment{projected dual ascent}
        \EndFor
    \EndFor

    \State $\Delta_\tau
    \gets \widehat{\phi}_\tau(\theta) -  \left(1-\widehat{\psi}_\tau(\theta)\right)$
    \Comment{bracket width, Prop.~\ref{prop_gap_stopping}}
    \If{$\Delta_\tau\leq\delta_{\mathrm{tol}}$}
        \State \textbf{break}
    \EndIf
    \State $\tau \gets \max(\tau_{\min},\gamma\tau)$
    \Comment{tighten inner approximation}
\EndFor
\State \Return $\theta$
\Comment{certify independently via Thm.~\ref{thm_holdout_certificates}}

\end{algorithmic}
\end{algorithm}
Unless otherwise stated, we use $\tau_0=0.5$, $\tau_{\min}=0.01$, $\gamma=0.65$, $\delta_{\mathrm{tol}}=10^{-4}$,
$\eta_\lambda=1.1$, $\beta=0.75$, and $\lambda_{\max}=50$, with surrogate parameters $m_1=0.5$ and $m_2=0.3$.  ReliableNet training proceeds in two phases. Phase~1 is a 10-epoch ERM warm-up, after which the threshold $\varepsilon^\ast$ is fixed to the 80th percentile of the confidences of misclassified examples on the selection fold.  Phase~2 is run for
$E=6$ epochs for at most $J_{\max}=15$ stages; for CIFAR-10 we use $E=14$ and $J_{\max}=5$ to match the baseline training budget.
The primal optimizer is Adam with learning rate $10^{-3}$, weight decay $10^{-4}$, batch size $256$, and gradient-norm clipping
at $5$.

\subsection{Benchmark Methods}
\label{sec_baselines}

We compare \textbf{ReliableNet} with six representative baselines spanning unconstrained training, post-hoc calibration and risk control, and
end-to-end selective prediction. All methods use the same base architecture and data splits within each dataset. Model parameters and operating points are fixed using the training and selection folds before evaluation on the independent certification fold. ReliableNet uses the fixed threshold $\varepsilon^\ast$ selected before constrained training and accepts an input when $c_\theta(x)\geq\varepsilon^\ast$. The corresponding acceptance rule for each baseline is specified below.

\begin{itemize}

\item \emph{ERM (Empirical Risk Minimization).} ERM minimizes the standard cross-entropy objective 

$\frac{1}{n}\sum_{i=1}^{n}  \ell_{\mathrm{CE}} \left(f_\theta(x_i),y_i\right)$, without an explicit reliability constraint \citep{vapnik1998statistical}. It serves as the primary unconstrained baseline. At evaluation, ERM accepts predictions satisfying $c_\theta(x)\geq\varepsilon^\ast$, using the same confidence threshold as ReliableNet.

\item \emph{Temperature Scaling.}
Temperature Scaling \citep{guo2017calibration} leaves the ERM-trained model unchanged and fits a scalar temperature $T>0$ on the selection fold by minimizing the negative log-likelihood of the scaled logits $f_\theta(x)/T$. The predicted class is unchanged, but the confidence becomes $c_{\theta,T}(x) = \max_k \operatorname{softmax} \left(f_\theta(x)/T\right)_k $.
The method accepts when $c_{\theta,T}(x)\geq\varepsilon^\ast$. It assesses whether post-hoc recalibration alone is sufficient to control confident errors.

\item \emph{Confidence Thresholding.}
This post-hoc baseline keeps the ERM model fixed and selects the most permissive threshold $\varepsilon^{\mathrm{ct}}$ on the selection fold such
that$ \widehat{\mathrm{JCW}}_{\mathrm{sel}} (\varepsilon^{\mathrm{ct}})  =  \frac{1}{n_{\mathrm{sel}}}  \sum_{i=1}^{n_{\mathrm{sel}}}  \mathbf{1} \left\{ \widehat y_i\neq y_i,\, c_i\geq\varepsilon^{\mathrm{ct}} \right\} \leq \alpha$ .
At evaluation, it accepts when $c_\theta(x)\geq\varepsilon^{\mathrm{ct}}$. This baseline isolates the effect of selecting an empirical JCW-controlling threshold from that of modifying the model during training.

\item \emph{Conformal Risk Control (CRC).}
CRC~\citep{angelopoulos2024conformal} keeps the ERM predictor fixed and selects an operating threshold using the conformal finite-sample risk
correction. Define $\widehat R_n(\varepsilon) = \frac{1}{n}\sum_{i=1}^{n} \mathbf{1} \left\{ \widehat y_i\neq y_i,\; c_i>\varepsilon \right\}$.
Since this loss is bounded by $B=1$ and non-increasing in $\varepsilon$, CRC selects $\widehat\varepsilon_{\mathrm{CRC}} = \inf\left\{ \varepsilon: \frac{n}{n+1}\widehat R_n(\varepsilon) +\frac{1}{n+1} \leq\alpha \right\}$. Under the exchangeability and regularity conditions of \citet{angelopoulos2024conformal}, this choice satisfies the finite-sample expected-risk guarantee $\mathbb{E} \left[ R_{n+1}(\widehat\varepsilon_{\mathrm{CRC}}) \right]\leq\alpha.$ At evaluation, CRC accepts predictions whose confidence exceeds the selected threshold. Unlike ReliableNet, CRC controls the operating rule of a fixed predictor rather than modifying the classifier parameters during training.

\item \emph{Selective prediction.} These methods are trained end-to-end to learn an abstention mechanism jointly with the classifier. They optimize the risk-coverage tradeoff but provide no certified operating point.
\begin{itemize}
\item \emph{SelectiveNet.}
SelectiveNet~\citep{geifman2019selectivenet} augments the classifier with a selection head $s_\theta(x)\in[0,1]$ and an auxiliary prediction head
$h_\theta(x)$. The main classifier and selection head minimize the selective risk under a target-coverage penalty,
\begin{equation}
\mathcal{L}_{\mathrm{sel}}(\theta)
    =  \frac{  \sum_i s_\theta(x_i)  \ell_{\mathrm{CE}} \left(f_\theta(x_i),y_i\right) }{ \sum_i s_\theta(x_i)+\epsilon  }
    +  \lambda_{\mathrm{cov}}  \left(  \kappa-\frac{1}{n}\sum_i s_\theta(x_i)   \right)_+^2 ,
    \label{eq_selectivenet}
\end{equation}
where $\kappa$ is the target coverage and $\lambda_{\mathrm{cov}}$ controls the coverage penalty. The auxiliary head is trained on all
inputs using $\mathcal{L}_{\mathrm{aux}}(\theta) = \frac{1}{n}\sum_i  \ell_{\mathrm{CE}} \left(h_\theta(x_i),y_i\right)$,
and the total objective is
\[ \mathcal{L}(\theta) = \omega\mathcal{L}_{\mathrm{sel}}(\theta)  +(1-\omega)\mathcal{L}_{\mathrm{aux}}(\theta),\]
where $\omega\in[0,1]$ is the mixing coefficient. After training, the native rejection score is $r_{\mathrm{SN}}(x)=1-s_\theta(x)$,
so lower values indicate more acceptable predictions. We select the most permissive threshold $t_{\mathrm{SN}}$ on the selection fold such that the empirical JCW of the accepted predictions does not exceed $\alpha$. At evaluation, SelectiveNet accepts when $r_{\mathrm{SN}}(x)\leq t_{\mathrm{SN}}$, equivalently, when $s_\theta(x)\geq 1-t_{\mathrm{SN}}$.

\item \emph{Deep Gamblers}~\citep{liu2019deep} augments the classifier with an additional \emph{abstain} output whose probability
$\hat p_{m+1}(x)$ represents rejection. After a short cross-entropy warm-up, the model is trained with the gambling loss
\begin{equation}
    \mathcal{L}_{\mathrm{DG}}(\theta)
    =
    -\frac{1}{n}\sum_{i=1}^{n}
    \log \left(
        \hat p_{y_i}(x_i)
        +
        \frac{\hat p_{m+1}(x_i)}{o}
    \right),
    \label{eq_deepgamblers}
\end{equation}
where $m$ is the number of classes and $o>1$ is the reward parameter controlling the trade-off between prediction and abstention. We use
$o=1.9$. For class prediction, the non-abstention probabilities are renormalized over the $m$ prediction classes. The native rejection score is
$r_{\mathrm{DG}}(x)=\hat p_{m+1}(x)$, so lower values indicate more acceptable predictions. We select the most permissive threshold
$t_{\mathrm{DG}}$ on the selection fold such that the empirical accepted-and-wrong rate does not exceed $\alpha$. At evaluation, Deep Gamblers accepts when $\hat p_{m+1}(x)\leq t_{\mathrm{DG}}$.
\end{itemize}
\end{itemize}
The benchmarks were chosen to span the main alternatives to ReliableNet: unconstrained ERM, post-hoc calibration and risk control, and learned
selective prediction. ERM, Temperature Scaling, and ReliableNet are evaluated at the common threshold $\varepsilon^\ast$, isolating the effect
of constrained training at a fixed deployment rule. Methods with native thresholding or abstention scores use the most permissive selection-fold
operating point satisfying the empirical JCW budget. All thresholds are fixed before evaluation on the independent certification fold, where every method is assessed using the same hard JCW event and Clopper-Pearson certificate. Because resulting coverages may differ, we also report AURC and risk-coverage curves to compare selective performance at matched coverage.
% \paragraph{Acceptance rules.} Because these methods accept and abstain through different mechanisms, a fair comparison requires fixing each one's operating point and scoring all of them on the same accepted-and-wrong metric. ERM, Temperature scaling and ReliableNet both accept predictions satisfying
% $c_\theta(x)\geq\varepsilon^\ast$.  ReliableNet uses the same threshold, but its constrained training changes the confidence
% distribution and therefore yields a different acceptance rate. Confidence Thresholding and CRC each take the most permissive selection-fold threshold, the former keeping validation JCW at most $\alpha$, the latter via $\frac{n}{n+1}\widehat{\mathrm{JCW}}_{\mathrm{val}} (\varepsilon) + \frac{1}{n+1} \leq \alpha$, and accept above it. SelectiveNet and Deep Gamblers instead accept by their learned gate $s(x)$ and abstention score
% $r(x)$, thresholded on the selection fold; Deep Gamblers uses $o = 1.5$.

\subsection{Evaluation Metrics}
\label{subsec_metrics}

We evaluate predictive performance, reliability, and approximation tightness using metrics that directly correspond to the formulation of probabilistic constraints.

\begin{table}[H]
\centering
\caption{Evaluation metrics (predictive performance and primary reliability).  $\mathcal{A}_\theta(\varepsilon)$
  is the acceptance region;
  $\pi$ sorts predictions by ascending score (most to least acceptable).
  $\downarrow$ = lower is better.}
\label{tab_metrics1}
\setlength{\extrarowheight}{3pt}
\resizebox{\textwidth}{!}{
\begin{tabular}{>{\bfseries}p{1.6cm} *{4}{p{4.2cm}}}
\toprule
 & \textbf{Accuracy $\uparrow$} & \textbf{Coverage $\uparrow$} & \textbf{HC Accuracy $\uparrow$} & \textbf{JCW \textnormal{\small(primary)} $\downarrow$} \\
\midrule
Formula
  & $\mathrm{Acc} = \mathbb{P}(\hat{y}_\theta(X) = Y)$
  & $\mathrm{Cov} = \mathbb{P}(X \in \mathcal{A}_\theta(\varepsilon))$
  & $\mathrm{Acc}_{\mathrm{conf}} = \mathbb{P}(\hat{y}=Y \mid X \in \mathcal{A}_\theta(\varepsilon))$
  & $\mathrm{JCW} = \mathbb{P}\left( c_\theta(X) \geq \varepsilon \wedge \hat{y}_\theta(X) \neq Y \right)$ \\[6pt]
Role
  & Overall correctness. Reported for completeness; does not capture confidence or abstention.
  & Fraction of inputs accepted.
  & Accuracy restricted to accepted predictions. Measures decision quality when the model acts.
  & \textbf{Primary reliability metric.} Exact quantity controlled by the probabilistic constraint $\mathrm{JCW} \leq \alpha$. \\
\bottomrule
\end{tabular}
}
\end{table}

\begin{table}[H]
\centering
\caption{Evaluation metrics (certification, calibration, and selective prediction).}
\label{tab_metrics2}
\setlength{\extrarowheight}{3pt}
\resizebox{\textwidth}{!}{
\begin{tabular}{>{\bfseries}p{1.6cm} *{3}{p{5.2cm}}}
\toprule
 & \textbf{Cert. $\uparrow$} & \textbf{ECE $\downarrow$} & \textbf{AURC $\downarrow$} \\
\midrule
Formula
  & $\mathrm{Cert} = \sum_{s=1}^{S} \mathbf{1}\left\{ U_{\mathrm{CP}}^{(s)} (K_m^{(s)},m;\rho)\leq\alpha\right\}.$
  & $\mathrm{ECE} = \sum_{b=1}^{B} \tfrac{|B_b|}{n} \bigl|\mathrm{acc}(B_b) - \mathrm{conf}(B_b)\bigr|$
  & $\mathrm{AURC} = \frac{1}{n} \sum_{k=1}^{n} R_k$, $R_k = \frac{1}{k} \sum_{i=1}^{k} \mathds{1} \left[ \hat y_{\pi(i)} \neq y_{\pi(i)} \right],$ \\[6pt]
Role
  & \textbf{Certification metric.} Counts how often the one-sided Clopper-Pearson upper bound is below $\alpha$ across $S$ seeds.
  & Alignment between confidence and accuracy across $B$ bins. Low ECE indicates well-calibrated predictions.
  & Area under the Risk-Coverage curve. Summarises selective prediction quality across all operating points. \\
\bottomrule
\end{tabular}
}
\end{table}

\section{Experiments}
\label{sec_experiments}

We evaluate ReliableNet on six benchmarks spanning four tabular  and two image  modalities:
a controlled synthetic ambiguity benchmark, UCI Adult, UCI German Credit, the cybersecurity dataset UNSW-NB15, spurious Colored-MNIST-10, and CIFAR-10. The experiments answer three questions. First, does ReliableNet satisfy the joint confident-wrong (JCW) constraint in distribution, and does it do so where standard methods fail? Second, how does it compare with confidence, calibration, and selective-classification baselines across accuracy, coverage, calibration, and ranking quality? Third, how does empirical reliability performance degrade under
demographic shift, covariate noise, spurious-correlation inversion, and novel classes?

\medskip
All experiments use PyTorch on an NVIDIA GeForce RTX 3090, with the same base architecture for all methods within each dataset. Tabular models use a two-hidden-layer MLP with width $128$; Colored-MNIST uses a five-layer Convolutional Neural Network; and CIFAR-10 uses ResNet-18~\citep{he2016deep}. Results are averaged over five seeds $\{0,37,42,123,2026\}$. We set $\alpha=0.05$ for the tabular datasets, $\alpha=0.03$ for CIFAR-10, and $\alpha=0.01$ for
Colored-MNIST, with certification level $\rho=0.05$. ReliableNet begins with a 10-epoch ERM warm-up, after which $\varepsilon^\ast$ is fixed to the 80th percentile of the confidence scores of misclassified selection-fold examples. The subsequent constrained-training procedure follows Algorithm~\ref{alg:inner_cc}. The other methods are trained over 100 epochs for tabular dataset and Colored MNIST and 80 for CIFAR10.

\subsection{Data Description}
\label{subsec_data_description}

We evaluate ReliableNet on six benchmarks covering intrinsic ambiguity, demographic shift, small-sample credit-risk prediction, spurious correlation, image corruption, and novel cyberattacks. Each benchmark is divided into independent training, selection, and certification folds. The selection fold is used for threshold choice, checkpointing, and operating-point selection, while the certification fold remains untouched until the final model is fixed. All preprocessing is fitted on the training fold only. For tabular datasets, numerical variables are median-imputed and standardized, categorical variables are mode-imputed and one-hot encoded, and zero-variance features are removed.

\paragraph{Synthetic Band.}
Synthetic Band is a controlled ambiguity benchmark in ($\mathbb{R}^{10}$). Labels follow a fixed linear boundary but are flipped with
probability (1/2) within a band of half-width ($\gamma$). Outside the band, the task is nearly separable; inside it, uncertainty is irreducible. We use (d=10), $(n_{\mathrm{train}}=5000)$, and $(n_{\mathrm{test}}=3000)$. The in-distribution setting uses $(\gamma=0.5)$, while the ambiguity stress test uses $(\gamma_{\mathrm{test}}=0.7)$. We also monitor the behavior of the JCW over sweeping $\gamma$ values.

\paragraph{UCI Adult.}
The Adult Income dataset~\citep{adult_2} contains $(48,842)$ census records with 14 demographic and socioeconomic features, with the
binary target indicating whether annual income exceeds \$50{,}000. The in-distribution experiment uses a stratified training-selection-certification split. For the stress test, the model is trained on male records and evaluated on female records. The sex variable is used only to define the domains and is excluded from the predictors. Reliability is also assessed across female education subgroups.
\paragraph{UCI German Credit.}
The German Credit dataset~\citep{hofmann1994germancredit} contains $(1000)$  records describing loan applicants with 20
features covering financial history, loan purpose, and personal attributes, with the binary target indicating creditworthiness. The in-distribution experiment uses a stratified training-selection-certification split. For the
stress test, the model is trained on applicants aged (30) and above and evaluated on applicants younger than (30), producing an age-based subgroup and support shift.

\paragraph{Colored-MNIST-10.}
Colored-MNIST \cite{deng2012mnist} is a 10-class benchmark in which digit identity is the target and color is an artificial spurious feature. Each digit receives its associated color with probability $(\rho)$; otherwise, an incorrect color is assigned. Images are weakened through downsampling, noise, and occasional occlusion to encourage reliance on color. The in-distribution folds use the same color generation process. For the stress test, the model is trained with $(\rho_{\mathrm{train}}=0.9)$ and evaluated with $(\rho_{\mathrm{test}}=0.1)$, creating a controlled spurious-correlation shift. We considered monitoring reliability for sweeping $\rho_{\mathrm{test}}$  values.

\paragraph{CIFAR-10.}
CIFAR-10 \cite{krizhevsky2009learning} contains (60{,}000) color images from ten object classes. Training, selection, and certification data are drawn from the clean official training split, while the clean official test set provides an independent in-distribution evaluation. Distributional stress tests apply label-preserving transformations to the same test images, including grayscale conversion, Gaussian noise, and reduced brightness. Images are normalized using the training-set channel statistics. These experiments evaluate reliability under image corruption and covariate shift rather than unknown-class detection.

\paragraph{UNSW-NB15.}
UNSW-NB15 \citep{7348942} is a network-intrusion dataset of 257{,}673 flow records, each described by 42 features (39 numerical and 3 categorical: proto, service, state), containing normal traffic and nine attack families. The binary task distinguishes normal from malicious traffic. The source distribution comprises normal observations together with four seen attack families (Generic, Exploits, Fuzzers, DoS), split into training, selection, and in-distribution certification folds. The stress test comprises disjoint normal observations together with the five held-out families (Reconnaissance, Analysis, Backdoor, Shellcode, Worms), evaluating generalization to novel attacks. Normal observations are partitioned before constructing the source and shifted sets to prevent overlap. Categorical features are one-hot encoded and numerical features standardized, with the preprocessing pipeline fitted only on the training fold.

\subsection{In-distribution certification and comparison} \label{subsec_in_distribution} We first evaluate all methods in distribution using independent training, selection, and certification folds. The training fold is used to fit the model, the selection fold is used to select thresholds, and the certification fold is held out until the final model and operating point are fixed. The main reliability statistic is the joint confident-wrong probability $\mathrm{JCW} = \mathbb{P} \bigl( \widehat{Y}(X) \neq Y,  \mathrm{conf}(X) \geq \varepsilon \bigr)$.  A seed is certified ($Cert$) when the Clopper-Pearson certificate is below the risk budget: $U_{\mathrm{CP}}(K_m,m;\rho)\leq \alpha$ ( $\rho=0.05$). Table~\ref{tab_id_main} compares ReliableNet with ERM, Temperature Scaling, Confidence Thresholding, CRC, SelectiveNet, and DeepGamblers, it first of all reports that ReliableNet is the only method that is  consistently certified across all datasets.

\begin{table*}[h!]
\centering
\caption{In-distribution results across six datasets (mean $\pm$ std over 5 seeds). Coverage is not bolded because it reflects each method's abstention policy. \textit{Cert.} counts seeds certified by the held-out Clopper-Pearson bound; AURC is threshold-free and measured over the full test split, so it reflects confidence-ranking quality independent of the operating point (identical for ERM, ConfThresholding, and CRC, which share the ERM score). Arrows indicate the preferred direction; $\alpha$ is the per-dataset JCW budget.}
\label{tab_id_main}
\resizebox{1.0\textwidth}{!}{%
\begin{tabular}{ll rrrrrrr c}
\toprule
Dataset & Method & Acc $\uparrow$ & Cov & Acc$_{\mathrm{HC}}$ $\uparrow$ & JCW $\downarrow$ & AURC $\downarrow$ & JCW/$\alpha$ $\downarrow$ & ECE $\downarrow$ & Cert. \\
\midrule
GermanCredit & ERM & $0.762${\scriptsize$ \pm 0.029$} & $1.00${\scriptsize$ \pm 0.00$} & $0.762${\scriptsize$ \pm 0.029$} & $0.2379${\scriptsize$ \pm 0.0294$} & $0.2314${\scriptsize$ \pm 0.0064$} & $4.76${\scriptsize$ \pm 0.59$} & $0.215${\scriptsize$ \pm 0.020$} & 0/5 \\
 & TempScaling & $0.762${\scriptsize$ \pm 0.029$} & $0.50${\scriptsize$ \pm 0.07$} & $0.886${\scriptsize$ \pm 0.011$} & $0.0568${\scriptsize$ \pm 0.0094$} & $0.2334${\scriptsize$ \pm 0.0058$} & $1.14${\scriptsize$ \pm 0.19$} & \textbf{0.090}{\scriptsize$ \pm 0.026$}& 1/5 \\
 & ConfThresholding & $0.762${\scriptsize$ \pm 0.029$} & $0.50${\scriptsize$ \pm 0.07$} & $0.886${\scriptsize$ \pm 0.011$} & $0.0568${\scriptsize$ \pm 0.0094$} & $0.2314${\scriptsize$ \pm 0.0064$} & $1.14${\scriptsize$ \pm 0.19$} & $0.215${\scriptsize$ \pm 0.020$} & 1/5 \\
 & CRC & $0.762${\scriptsize$ \pm 0.029$} & $0.44${\scriptsize$ \pm 0.08$} & $0.894${\scriptsize$ \pm 0.019$} & $0.0463${\scriptsize$ \pm 0.0120$} & $0.2314${\scriptsize$ \pm 0.0064$} & $0.93${\scriptsize$ \pm 0.24$} & $0.215${\scriptsize$ \pm 0.020$} & 2/5 \\
 & SelectiveNet & $0.762${\scriptsize$ \pm 0.035$} & $0.47${\scriptsize$ \pm 0.11$} & $0.912${\scriptsize$ \pm 0.047$} & $0.0442${\scriptsize$ \pm 0.0319$} & $0.2298${\scriptsize$ \pm 0.0119$} & $0.88${\scriptsize$ \pm 0.64$} & $0.202${\scriptsize$ \pm 0.025$} & 2/5 \\
 & DeepGamblers & \textbf{0.773}{\scriptsize$ \pm 0.034$} & $0.45${\scriptsize$ \pm 0.07$} & $0.908${\scriptsize$ \pm 0.036$} & $0.0421${\scriptsize$ \pm 0.0211$} & $0.2289${\scriptsize$ \pm 0.0159$} & $0.84${\scriptsize$ \pm 0.42$} & $0.207${\scriptsize$ \pm 0.027$} & 2/5 \\
  & \textbf{ReliableNet (Ours)} & 0.749{\scriptsize$ \pm 0.039$} & 0.39{\scriptsize$ \pm 0.04$} & \textbf{0.928}{\scriptsize$ \pm 0.031$} & \textbf{0.0274}{\scriptsize$ \pm 0.0120$} & \textbf{0.2145}{\scriptsize$ \pm 0.0129$} & \textbf{0.55}{\scriptsize$ \pm 0.24$} & 0.101{\scriptsize$ \pm 0.030$}& \textbf{5/5} \\

\midrule
Adult & ERM & $0.809${\scriptsize$ \pm 0.003$} & $1.00${\scriptsize$ \pm 0.00$} & $0.809${\scriptsize$ \pm 0.003$} & $0.1911${\scriptsize$ \pm 0.0025$} & $0.0211${\scriptsize$ \pm 0.0013$} & $3.82${\scriptsize$ \pm 0.05$} & $0.056${\scriptsize$ \pm 0.004$} & 0/5 \\
 & TempScaling & $0.809${\scriptsize$ \pm 0.003$} & $0.63${\scriptsize$ \pm 0.01$} & $0.917${\scriptsize$ \pm 0.007$} & $0.0520${\scriptsize$ \pm 0.0050$} & $0.0211${\scriptsize$ \pm 0.0013$} & $1.04${\scriptsize$ \pm 0.10$} & $0.018${\scriptsize$ \pm 0.006$} & 2/5 \\
 & ConfThresholding & $0.809${\scriptsize$ \pm 0.003$} & $0.63${\scriptsize$ \pm 0.01$} & $0.917${\scriptsize$ \pm 0.007$} & $0.0520${\scriptsize$ \pm 0.0050$} & $0.0211${\scriptsize$ \pm 0.0013$} & $1.04${\scriptsize$ \pm 0.10$} & $0.056${\scriptsize$ \pm 0.004$} & 2/5 \\
 & CRC & $0.809${\scriptsize$ \pm 0.003$} & $0.62${\scriptsize$ \pm 0.01$} & $0.917${\scriptsize$ \pm 0.007$} & $0.0517${\scriptsize$ \pm 0.0049$} & $0.0211${\scriptsize$ \pm 0.0013$} & $1.03${\scriptsize$ \pm 0.10$} & $0.056${\scriptsize$ \pm 0.004$} & 2/5 \\
 & SelectiveNet & $0.810${\scriptsize$ \pm 0.004$} & $0.62${\scriptsize$ \pm 0.01$} & $0.916${\scriptsize$ \pm 0.005$} & $0.0519${\scriptsize$ \pm 0.0040$} & $0.0299${\scriptsize$ \pm 0.0032$} & $1.04${\scriptsize$ \pm 0.08$} & $0.053${\scriptsize$ \pm 0.008$} & 1/5 \\
 & DeepGamblers & $0.816${\scriptsize$ \pm 0.008$} & $0.60${\scriptsize$ \pm 0.01$} & $0.915${\scriptsize$ \pm 0.007$} & $0.0515${\scriptsize$ \pm 0.0043$} & $0.0284${\scriptsize$ \pm 0.0047$} & $1.03${\scriptsize$ \pm 0.09$} & $0.138${\scriptsize$ \pm 0.010$} & 2/5 \\
  & \textbf{ReliableNet (Ours)} & \textbf{0.821}{\scriptsize$ \pm 0.004$} & 0.52{\scriptsize$ \pm 0.12$} & \textbf{0.952}{\scriptsize$ \pm 0.023$} & \textbf{0.0271}{\scriptsize$ \pm 0.0144$} & \textbf{0.0159}{\scriptsize$ \pm 0.0018$} & \textbf{0.54}{\scriptsize$ \pm 0.29$} & \textbf{0.015}{\scriptsize$ \pm 0.022$}& \textbf{5/5} \\
\midrule

SyntheticBand & ERM & $0.835${\scriptsize$ \pm 0.029$} & $1.00${\scriptsize$ \pm 0.00$} & $0.835${\scriptsize$ \pm 0.029$} & $0.1653${\scriptsize$ \pm 0.0289$} & $0.0721${\scriptsize$ \pm 0.0249$} & $3.31${\scriptsize$ \pm 0.58$} & $0.065${\scriptsize$ \pm 0.011$} & 0/5 \\
 & TempScaling & $0.835${\scriptsize$ \pm 0.029$} & $0.75${\scriptsize$ \pm 0.08$} & $0.932${\scriptsize$ \pm 0.023$} & $0.0515${\scriptsize$ \pm 0.0206$} & $0.0721${\scriptsize$ \pm 0.0249$} & $1.03${\scriptsize$ \pm 0.41$} & $0.036${\scriptsize$ \pm 0.008$} & 3/5 \\
 & ConfThresholding & $0.835${\scriptsize$ \pm 0.029$} & $0.75${\scriptsize$ \pm 0.08$} & $0.932${\scriptsize$ \pm 0.023$} & $0.0515${\scriptsize$ \pm 0.0206$} & $0.0721${\scriptsize$ \pm 0.0249$} & $1.03${\scriptsize$ \pm 0.41$} & $0.065${\scriptsize$ \pm 0.011$} & 3/5 \\
 & CRC & $0.835${\scriptsize$ \pm 0.029$} & $0.74${\scriptsize$ \pm 0.07$} & $0.935${\scriptsize$ \pm 0.017$} & $0.0479${\scriptsize$ \pm 0.0142$} & $0.0721${\scriptsize$ \pm 0.0249$} & $0.96${\scriptsize$ \pm 0.28$} & $0.065${\scriptsize$ \pm 0.011$} & 3/5 \\
 & SelectiveNet & $0.837${\scriptsize$ \pm 0.032$} & $0.76${\scriptsize$ \pm 0.05$} & $0.936${\scriptsize$ \pm 0.022$} & $0.0483${\scriptsize$ \pm 0.0170$} & \textbf{0.0683}{\scriptsize$ \pm 0.0243$}& $0.97${\scriptsize$ \pm 0.34$} & $0.066${\scriptsize$ \pm 0.004$} & 2/5 \\
 & DeepGamblers & $0.835${\scriptsize$ \pm 0.030$} & $0.67${\scriptsize$ \pm 0.14$} & $0.937${\scriptsize$ \pm 0.014$} & $0.0437${\scriptsize$ \pm 0.0160$} & $0.0709${\scriptsize$ \pm 0.0269$}& $0.87${\scriptsize$ \pm 0.32$} & $0.074${\scriptsize$ \pm 0.016$} & 2/5 \\
  & \textbf{ReliableNet (Ours)} & \textbf{0.839}{\scriptsize$ \pm 0.029$} & 0.42{\scriptsize$ \pm 0.12$} & \textbf{0.999}{\scriptsize$ \pm 0.001$} & \textbf{0.0005}{\scriptsize$ \pm 0.0003$} & 0.0712{\scriptsize$ \pm 0.0257$}& \textbf{0.01}{\scriptsize$ \pm 0.01$} & \textbf{0.029}{\scriptsize$ \pm 0.019$}& \textbf{5/5} \\
  
% \midrule
\bottomrule
\end{tabular}
}
\end{table*}

\begin{table}[h!]
\centering
\caption{In-distribution results across six datasets
(Table~\ref{tab_id_main} continued).}
\label{tab_id_main_prime}
\resizebox{1.0\textwidth}{!}{%
\begin{tabular}{ll rrrrrrr c}
\toprule
Dataset & Method & Acc $\uparrow$ & Cov & Acc$_{\mathrm{HC}}$ $\uparrow$ & JCW $\downarrow$ & AURC $\downarrow$ & JCW/$\alpha$ $\downarrow$ & ECE $\downarrow$ & Cert. \\
\midrule
ColoredMNIST & ERM & $0.980${\scriptsize$ \pm 0.016$} & $1.00${\scriptsize$ \pm 0.00$} & $0.980${\scriptsize$ \pm 0.016$} & $0.0198${\scriptsize$ \pm 0.0163$} & $0.0238${\scriptsize$ \pm 0.0200$} & $1.98${\scriptsize$ \pm 1.63$} & $0.004${\scriptsize$ \pm 0.003$} & 1/5 \\
 & TempScaling & $0.980${\scriptsize$ \pm 0.016$} & $0.97${\scriptsize$ \pm 0.05$} & $0.991${\scriptsize$ \pm 0.002$} & $0.0091${\scriptsize$ \pm 0.0022$} & $0.0239${\scriptsize$ \pm 0.0201$} & $0.91${\scriptsize$ \pm 0.22$} & $0.003${\scriptsize$ \pm 0.001$} & 2/5 \\
 & ConfThresholding & $0.980${\scriptsize$ \pm 0.016$} & $0.97${\scriptsize$ \pm 0.05$} & $0.991${\scriptsize$ \pm 0.002$} & $0.0091${\scriptsize$ \pm 0.0022$} & $0.0238${\scriptsize$ \pm 0.0200$} & $0.91${\scriptsize$ \pm 0.22$} & $0.004${\scriptsize$ \pm 0.003$} & 2/5 \\
 & CRC & $0.980${\scriptsize$ \pm 0.016$} & $0.97${\scriptsize$ \pm 0.06$} & $0.991${\scriptsize$ \pm 0.002$} & $0.0089${\scriptsize$ \pm 0.0022$} & $0.0238${\scriptsize$ \pm 0.0200$} & $0.89${\scriptsize$ \pm 0.22$} & $0.004${\scriptsize$ \pm 0.003$} & 3/5 \\
 & SelectiveNet & $0.985${\scriptsize$ \pm 0.014$} & $0.86${\scriptsize$ \pm 0.25$} & $0.989${\scriptsize$ \pm 0.006$} & $0.0085${\scriptsize$ \pm 0.0020$} & $0.009${\scriptsize$ \pm 0.0043$}& $0.85${\scriptsize$ \pm 0.20$} & $0.003${\scriptsize$ \pm 0.002$} & 3/5 \\
 & DeepGamblers & $0.986${\scriptsize$ \pm 0.008$} & $0.99${\scriptsize$ \pm 0.02$} & $0.990${\scriptsize$ \pm 0.003$} & $0.0093${\scriptsize$ \pm 0.0028$} & $0.0105${\scriptsize$ \pm 0.0201$}& $0.93${\scriptsize$ \pm 0.28$} & $0.004${\scriptsize$ \pm 0.003$} & 2/5 \\
  & \textbf{ReliableNet (Ours)} & \textbf{0.992}{\scriptsize$ \pm 0.001$} & 0.98{\scriptsize$ \pm 0.00$} & \textbf{0.998}{\scriptsize$ \pm 0.001$} & \textbf{0.0023}{\scriptsize$ \pm 0.0006$} & \textbf{0.0076}{\scriptsize$ \pm 0.0008$} & \textbf{0.23}{\scriptsize$ \pm 0.06$} & \textbf{0.002}{\scriptsize$ \pm 0.000$} & \textbf{5/5} \\
  
\midrule
CIFAR10 & ERM & $0.831${\scriptsize$ \pm 0.009$} & $1.00${\scriptsize$ \pm 0.00$} & $0.831${\scriptsize$ \pm 0.009$} & $0.1692${\scriptsize$ \pm 0.0094$} & $0.0681${\scriptsize$ \pm 0.0044$} & $5.64${\scriptsize$ \pm 0.31$} & $0.112${\scriptsize$ \pm 0.009$} & 0/5 \\

 & TempScaling & $0.831${\scriptsize$ \pm 0.009$} & $0.68${\scriptsize$ \pm 0.03$} & $0.953${\scriptsize$ \pm 0.003$} & $0.0319${\scriptsize$ \pm 0.0029$} & $0.0677${\scriptsize$ \pm 0.0044$} & $1.06${\scriptsize$ \pm 0.10$} & $0.023${\scriptsize$ \pm 0.003$} & 1/5 \\
 & ConfThresholding & $0.831${\scriptsize$ \pm 0.009$} & $0.68${\scriptsize$ \pm 0.02$} & $0.952${\scriptsize$ \pm 0.003$} & $0.0331${\scriptsize$ \pm 0.0024$} & $0.0681${\scriptsize$ \pm 0.0044$} & $1.10${\scriptsize$ \pm 0.08$} & $0.112${\scriptsize$ \pm 0.009$} & 1/5 \\
 & CRC & $0.831${\scriptsize$ \pm 0.009$} & $0.68${\scriptsize$ \pm 0.02$} & $0.952${\scriptsize$ \pm 0.003$} & $0.0326${\scriptsize$ \pm 0.0021$} & $0.0681${\scriptsize$ \pm 0.0044$} & $1.09${\scriptsize$ \pm 0.07$} & $0.112${\scriptsize$ \pm 0.009$} & 1/5 \\
 & SelectiveNet & $0.822${\scriptsize$ \pm 0.020$} & $0.63${\scriptsize$ \pm 0.06$} & $0.949${\scriptsize$ \pm 0.008$} & $0.0321${\scriptsize$ \pm 0.0038$} & $0.0541${\scriptsize$ \pm 0.0169$}& $1.07${\scriptsize$ \pm 0.13$} & $0.107${\scriptsize$ \pm 0.014$} & 2/5 \\
 & DeepGamblers & $0.820${\scriptsize$ \pm 0.009$} & $0.64${\scriptsize$ \pm 0.03$} & $0.949${\scriptsize$ \pm 0.003$} & $0.0330${\scriptsize$ \pm 0.0033$} & \textbf{0.0319}{\scriptsize$ \pm 0.0029$}& $1.10${\scriptsize$ \pm 0.11$} & $0.130${\scriptsize$ \pm 0.009$} & 1/5 \\
  & \textbf{ReliableNet (Ours)} & \textbf{0.853}{\scriptsize$ \pm 0.011$} & 0.52{\scriptsize$ \pm 0.06$} & \textbf{0.989}{\scriptsize$ \pm 0.003$} & \textbf{0.0061}{\scriptsize$ \pm 0.0021$} & 0.0530{\scriptsize$ \pm 0.0047$}& \textbf{0.20}{\scriptsize$ \pm 0.07$} & \textbf{0.024}{\scriptsize$ \pm 0.012$} & \textbf{5/5} \\
  
\midrule
UNSW-NB15 & ERM & $0.931${\scriptsize$ \pm 0.002$} & $1.00${\scriptsize$ \pm 0.00$} & $0.931${\scriptsize$ \pm 0.002$} & $0.0693${\scriptsize$ \pm 0.0020$} & $0.0141${\scriptsize$ \pm 0.0020$} & $1.39${\scriptsize$ \pm 0.04$} & $0.007${\scriptsize$ \pm 0.001$} & 0/5 \\
 & TempScaling & $0.931${\scriptsize$ \pm 0.002$} & $0.96${\scriptsize$ \pm 0.00$} & $0.948${\scriptsize$ \pm 0.002$} & $0.0497${\scriptsize$ \pm 0.0024$} & $0.0141${\scriptsize$ \pm 0.0020$} & $0.99${\scriptsize$ \pm 0.05$} & $0.007${\scriptsize$ \pm 0.002$} & 2/5 \\
 & ConfThresholding & $0.931${\scriptsize$ \pm 0.002$} & $0.96${\scriptsize$ \pm 0.00$} & $0.948${\scriptsize$ \pm 0.002$} & $0.0497${\scriptsize$ \pm 0.0024$} & $0.0141${\scriptsize$ \pm 0.0020$} & $0.99${\scriptsize$ \pm 0.05$} & $0.007${\scriptsize$ \pm 0.001$} & 2/5 \\
 & CRC & $0.931${\scriptsize$ \pm 0.002$} & $0.96${\scriptsize$ \pm 0.00$} & $0.948${\scriptsize$ \pm 0.002$} & $0.0495${\scriptsize$ \pm 0.0025$} & $0.0141${\scriptsize$ \pm 0.0020$} & $0.99${\scriptsize$ \pm 0.05$} & $0.007${\scriptsize$ \pm 0.001$} & 2/5 \\
 & SelectiveNet & $0.929${\scriptsize$ \pm 0.003$} & $0.94${\scriptsize$ \pm 0.01$} & $0.948${\scriptsize$ \pm 0.002$} & $0.0486${\scriptsize$ \pm 0.0017$} & $0.0134${\scriptsize$ \pm 0.0030$}& $0.97${\scriptsize$ \pm 0.03$} & $0.008${\scriptsize$ \pm 0.001$} & 3/5 \\
 & DeepGamblers & \textbf{0.939}{\scriptsize$ \pm 0.002$}& $0.96${\scriptsize$ \pm 0.01$} & $0.950${\scriptsize$ \pm 0.002$} & $0.0482${\scriptsize$ \pm 0.0024$} & \textbf{0.0133}{\scriptsize$ \pm 0.0012$}& $0.96${\scriptsize$ \pm 0.05$} & $0.042${\scriptsize$ \pm 0.002$} & 3/5 \\
  & \textbf{ReliableNet (Ours)} & 0.932{\scriptsize$ \pm 0.002$}& 0.85{\scriptsize$ \pm 0.02$} & \textbf{0.982}{\scriptsize$ \pm 0.005$} & \textbf{0.0154}{\scriptsize$ \pm 0.0045$} & 0.0139{\scriptsize$ \pm 0.0017$}& \textbf{0.31}{\scriptsize$ \pm 0.09$} & \textbf{0.006}{\scriptsize$ \pm 0.002$} & \textbf{5/5} \\
\bottomrule
\end{tabular}
}
\end{table}
 These in-distribution results show that ReliableNet is not merely a post-hoc thresholding rule. It directly trains the model toward a certificate-compatible JCW constraint. On Synthetic Band and German Credit, ReliableNet obtains the lowest JCW and the highest high-confidence accuracy. On Adult and Colored-MNIST, Temperature Scaling remains a strong calibration baseline, but ReliableNet remains competitive while providing a training-time reliability mechanism and consistent certification behavior.

\subsection{Stress Test Results}
\label{subsec_ood_stress_tests}
This part covers aggregated metrics and selective prediction result under various stress tests.
\subsubsection{Aggregate Reliability Under Fixed Shift}
\label{subsec_ood_aggre_tests}
We evaluate reliability under a single fixed distribution shift per dataset, chosen to represent a realistic deployment mismatch rather than a swept severity. Each dataset  defines a distinct shift type: subpopulation shift in Adult (train on male, test on female  observations) and German Credit (train on individuals aged $\ge 30$, test on those aged  $<30$); novel-class shift in UNSW-NB15 (train on Generic, Exploits, Fuzzers and DoS; test on the five held-out attack families); a fixed ambiguity gap in Synthetic Band  ($\gamma_{\text{train}}=0.5$, $\gamma_{\text{test}}=0.7$); spurious-correlation shift in  Colored~MNIST ($\rho_{\text{train}}=0.90$, $\rho_{\text{test}}=0.10$, fully inverting the  colour cue); and covariate shift in CIFAR-10 (grayscale conversion with additive Gaussian 
noise $\sigma=0.15$ and brightness factor $0.60$). Together, these span common shift  types encountered in deployment.

\medskip

Figure~\ref{fig_summary_metrics} compares ReliableNet with the benchmark methods across the six datasets. 
The main reliability metric is $JCW/\alpha$, where values below one satisfy the target confident-error budget. 
ReliableNet consistently reduces $JCW/\alpha$ relative to ERM and all  baselines, on all six datasets. 
On ColoredMNIST, the severe spurious-correlation shift remains challenging; ReliableNet still lowers the normalized confident-error rate compared with the benchmark methods.
 
\begin{figure}[H]
    \centering
    \begin{subfigure}[b]{0.48\textwidth}
        \centering
        \includegraphics[width=\textwidth]{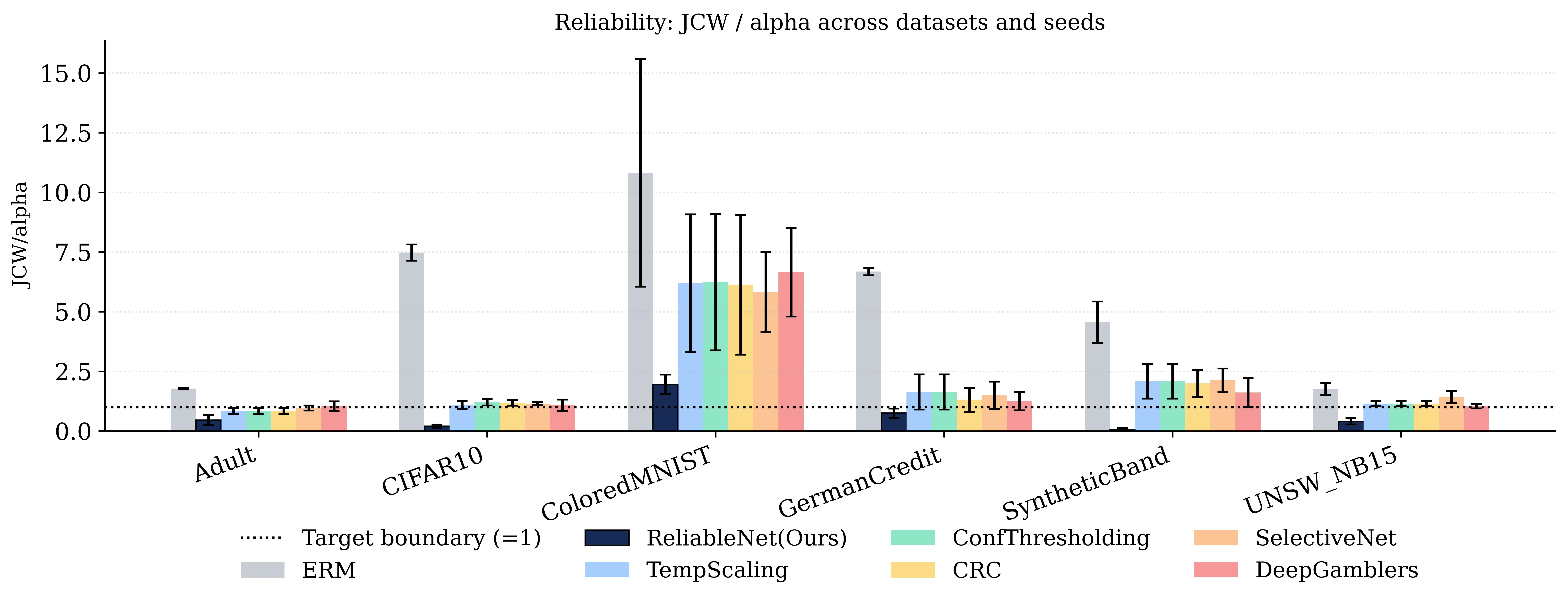}
        \caption{$JCW/\alpha$ across datasets and methods. The dotted line marks the target boundary $JCW/\alpha=1$.}
        \label{fig_summary_jcw_alpha}
    \end{subfigure}
    \hfill
    \begin{subfigure}[b]{0.48\textwidth}
        \centering
        \includegraphics[width=\textwidth]{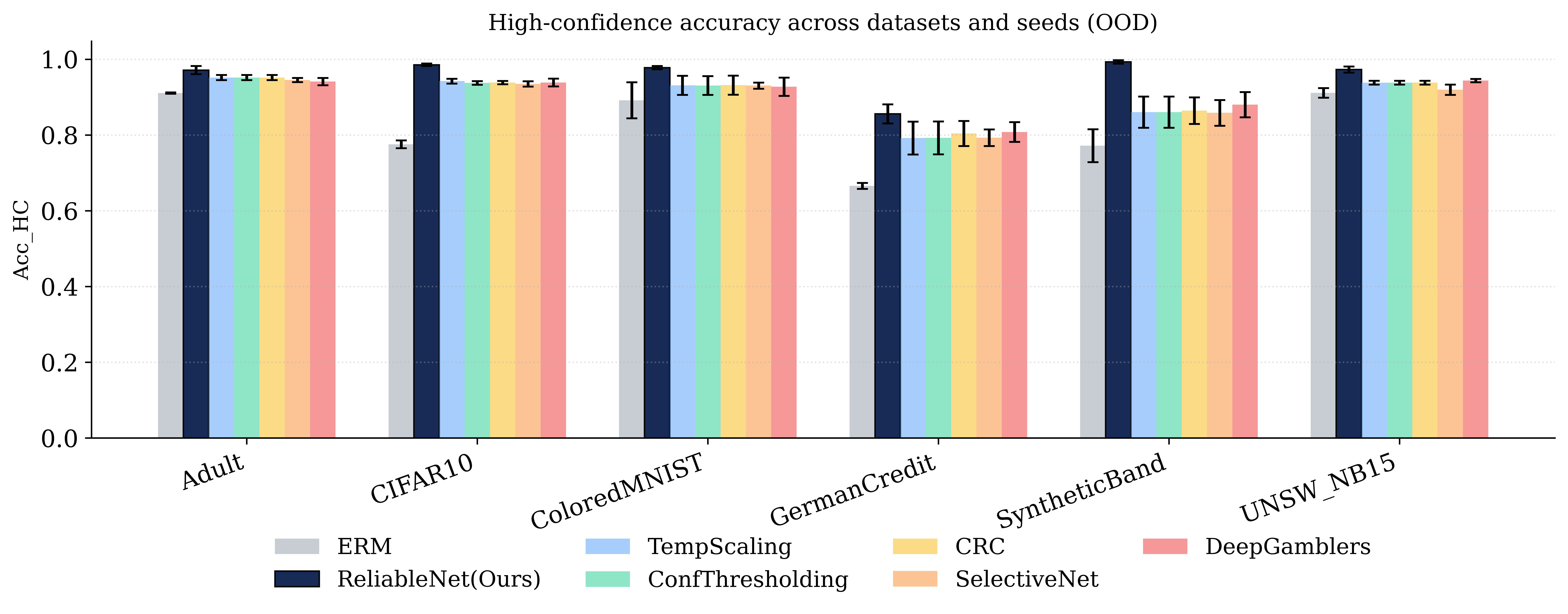}
        \caption{High-confidence accuracy across datasets and methods.}
        \label{fig_summary_acc_hc}
    \end{subfigure}

    \vspace{0.4cm}

    \begin{subfigure}[b]{0.48\textwidth}
        \centering
        \includegraphics[width=\textwidth]{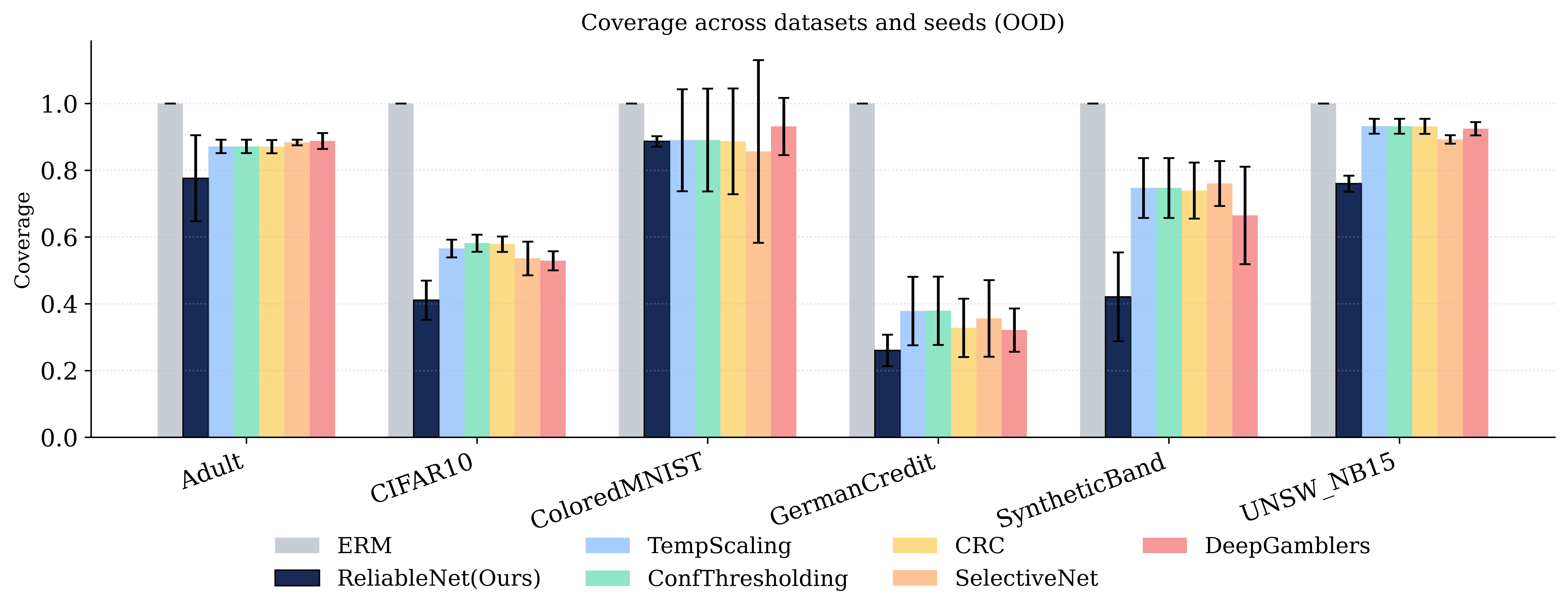}
        \caption{Coverage across datasets and methods.}
        \label{fig_summary_coverage}
    \end{subfigure}
    \hfill
    \begin{subfigure}[b]{0.48\textwidth}
        \centering
        \includegraphics[width=\textwidth]{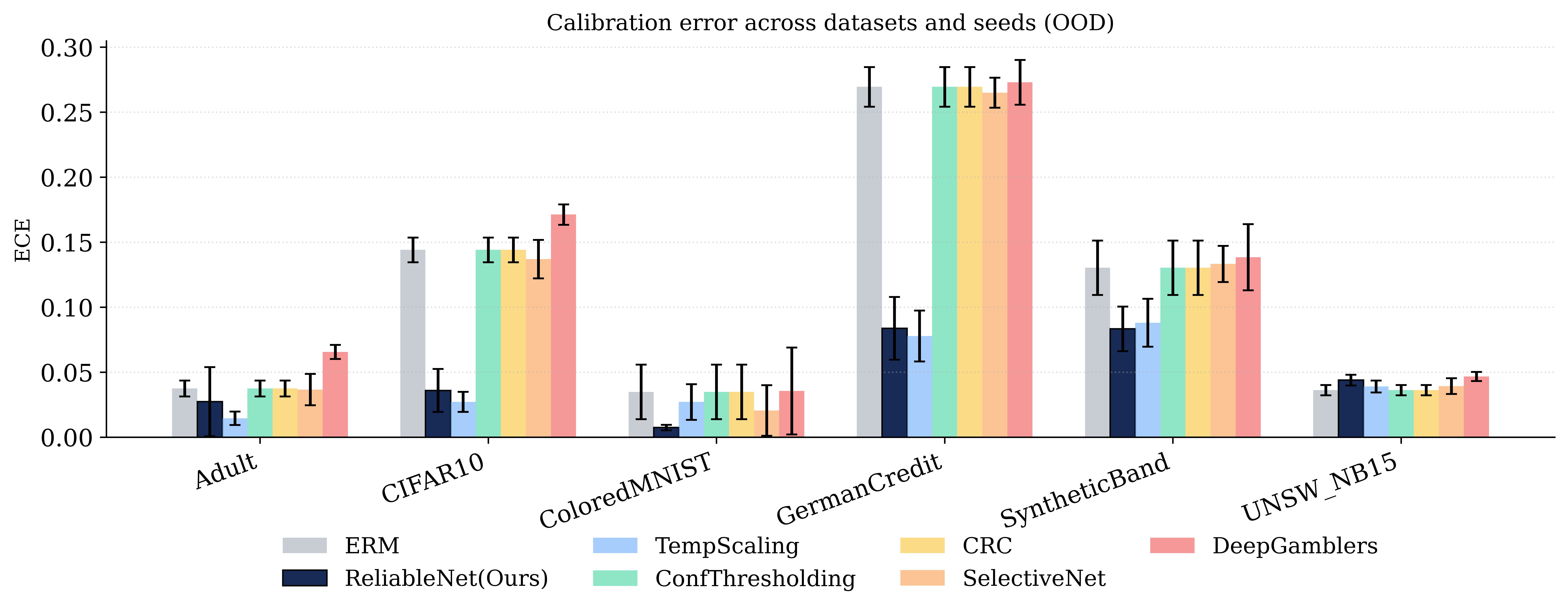}
        \caption{Expected calibration error (ECE) across datasets and methods.}
        \label{fig_summary_ece}
    \end{subfigure}
     \hfill
    \begin{subfigure}[b]{0.5\textwidth}
        \centering
        \includegraphics[width=\textwidth]{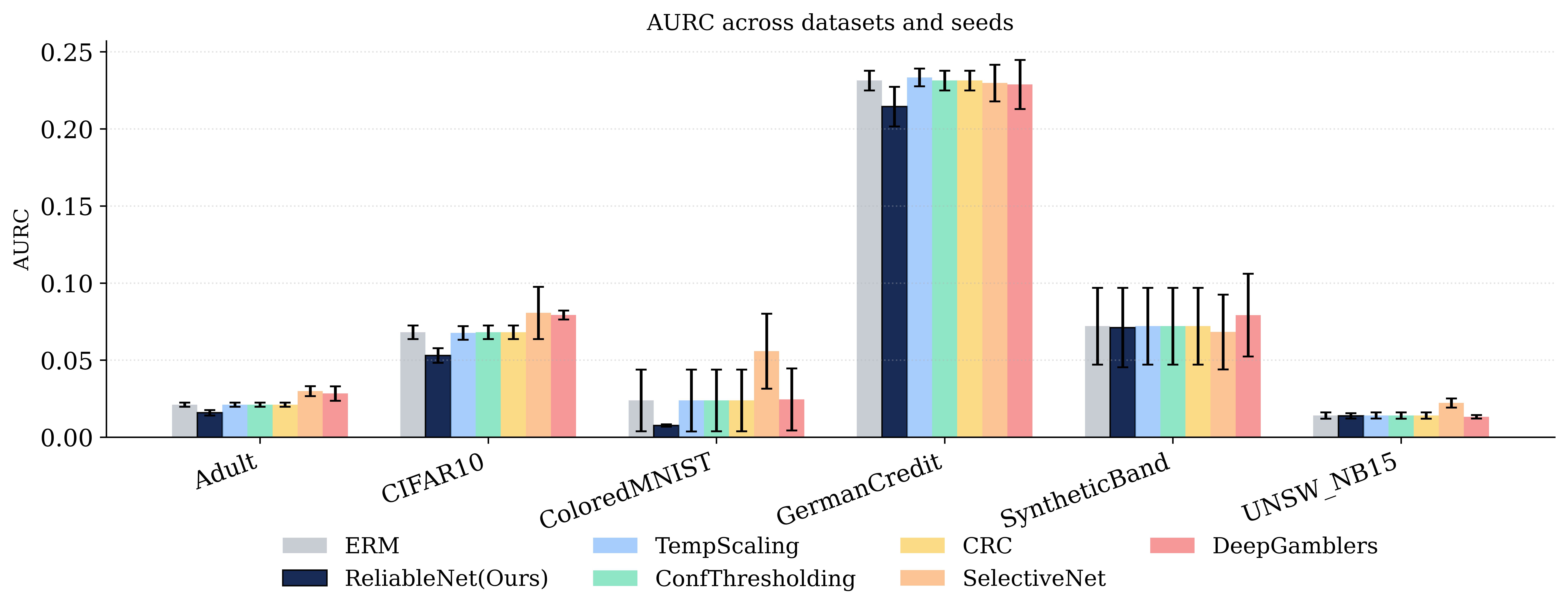}
        \caption{Area under the Risk-Coverage curve (AURC) across datasets and methods.}
        \label{fig_summary_aurd}
    \end{subfigure}

    \caption{
    Summary comparison over five random seeds on Adult, ColoredMNIST, German Credit, and SyntheticBand, CIFAR and UNSW-NB15. 
    Bars show seed averages and error bars show standard deviations.}
    \label{fig_summary_metrics}
\end{figure}
The high-confidence accuracy results show that ReliableNet usually achieves the strongest or near-strongest accuracy among accepted predictions. 
This suggests that the method improves the quality of accepted predictions rather than simply rejecting difficult cases. 
The coverage panel confirms the expected reliability-coverage trade-off: ReliableNet rejects more samples than ERM and other benchmark, but generally preserves meaningful coverage. 
Finally, the ECE results show that Temperature scaling and ReliableNet are  better calibrated than other models. The AURC results show improved ranking quality on four datasets and competitive performance on the remaining two. Overall, the summary metrics indicate that ReliableNet offers a favorable balance between reliability, high-confidence accuracy, coverage, selective prediction and calibration across heterogeneous shifts.

\subsubsection{Selective Prediction Quality: Risk-Coverage Curve During Stress test}
\label{subsec_AURC}
Motivated by Corollary~\ref{cor:selective}, we assess selective risk at matched coverage under distribution shift, where differences in confidence ranking become most informative and cannot be attributed to operating at a lower coverage. We report the area under the risk-coverage curve (AURC), where lower values indicate lower selective risk averaged over all coverage levels. Crucially, a risk-coverage curve evaluates every method over the same coverage grid $[0,1]$: at each coverage $c$ the curve reports the selective risk of the top-ranked fraction $c$ of inputs, so the comparison is coverage-matched by construction and the AURC integrates each method over identical support. Any advantage therefore reflects a genuinely better confidence ordering and cannot be attributed to a method simply operating at lower coverage. Figure~\ref{fig:rc_all} shows mean risk-coverage curves over five seeds under the
stress-test conditions of Section~\ref{subsec_ood_aggre_tests}.

\begin{figure}[htbp]
    \centering
    \begin{subfigure}[b]{0.48\textwidth}
        \includegraphics[width=\textwidth]{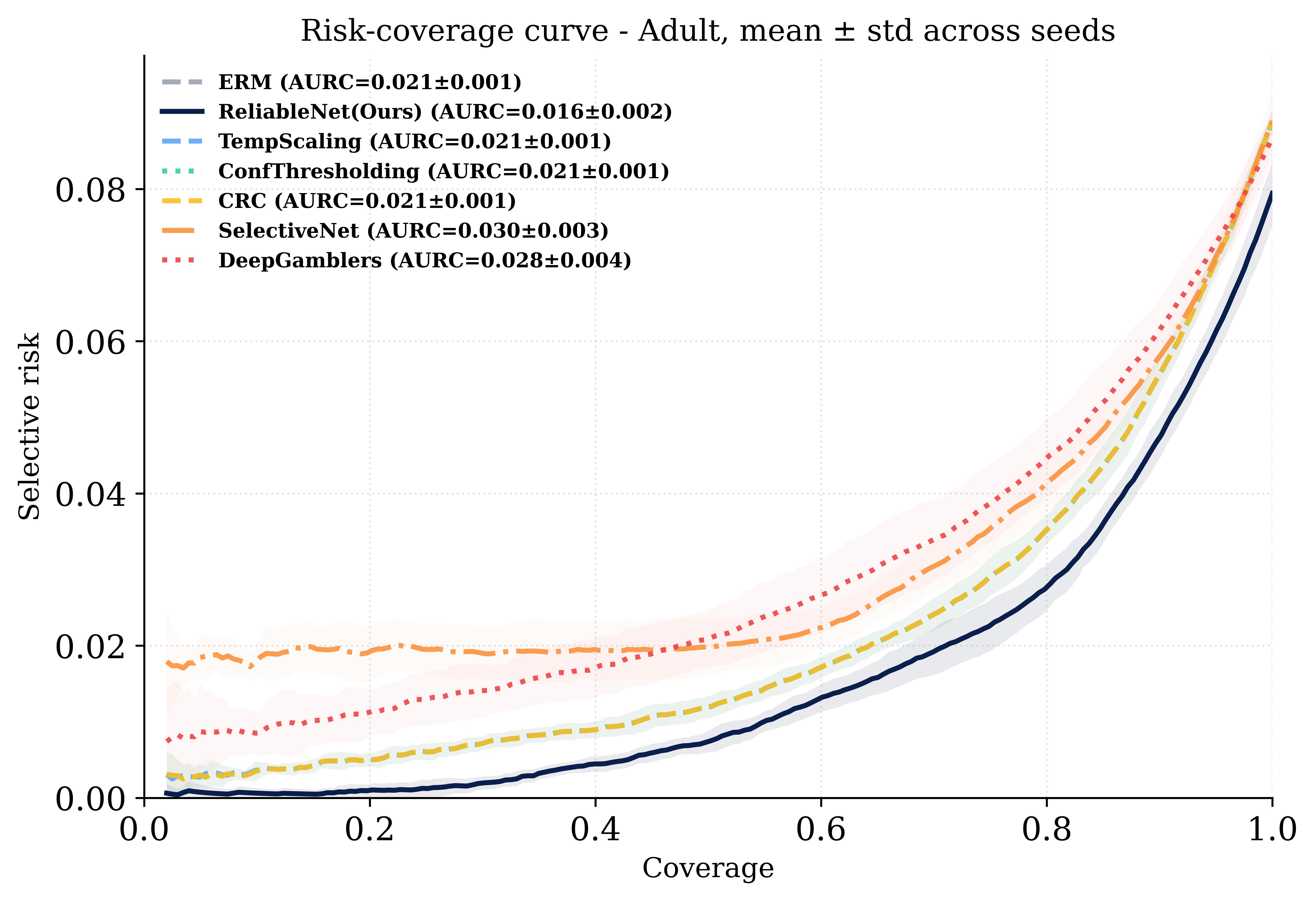}
        \caption{Adult. ReliableNet attains the lowest AURC ($0.016\pm0.002$).}
        \label{fig_rc_adult}
    \end{subfigure}\hfill
    \begin{subfigure}[b]{0.48\textwidth}
        \includegraphics[width=\textwidth]{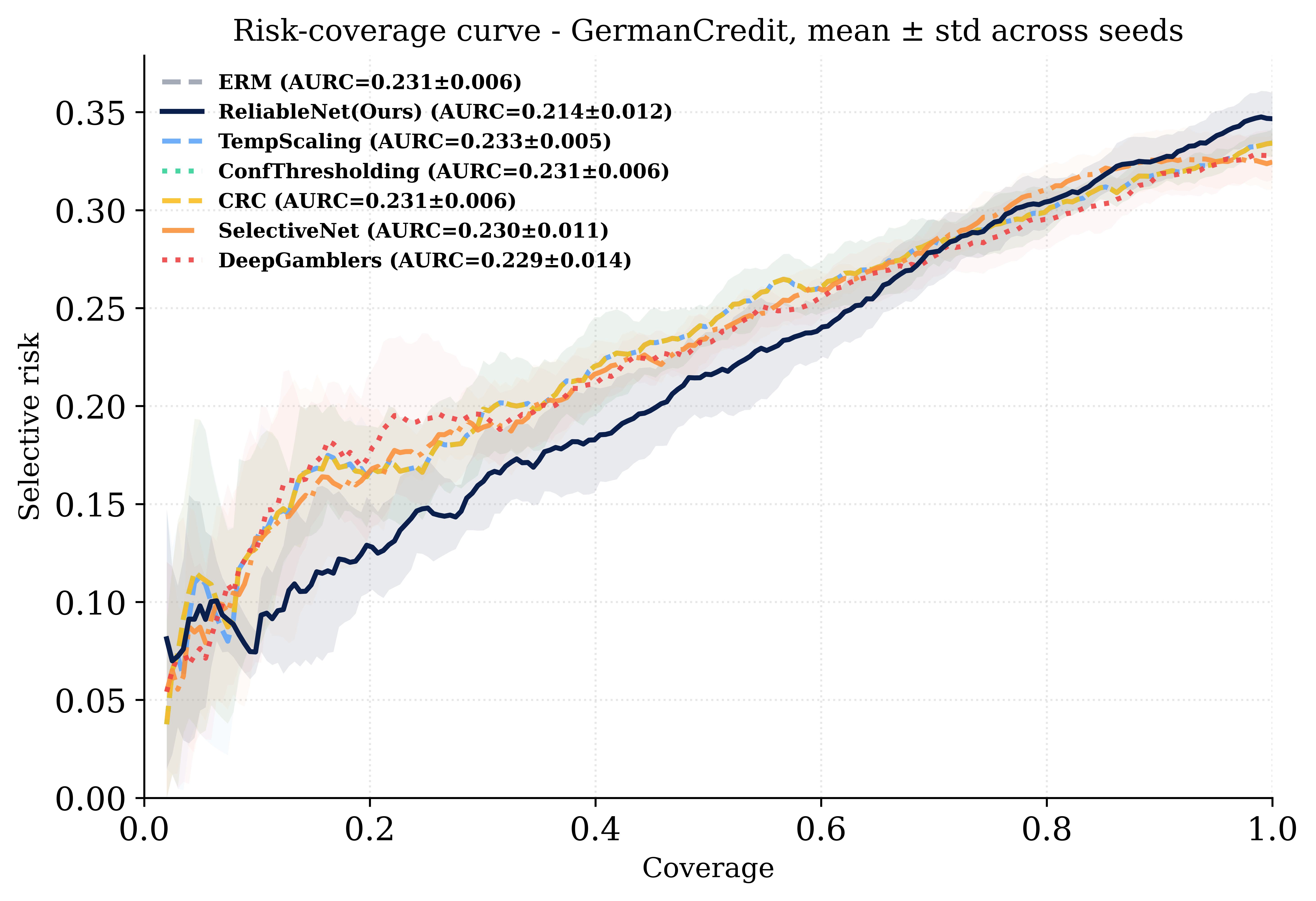}
        \caption{German Credit. ReliableNet attains the lowest AURC ($0.214\pm0.012$) under age-based shift.}
        \label{fig_rc_german}
    \end{subfigure}

    \vspace{0.6em}
    \begin{subfigure}[b]{0.48\textwidth}
        \includegraphics[width=\textwidth]{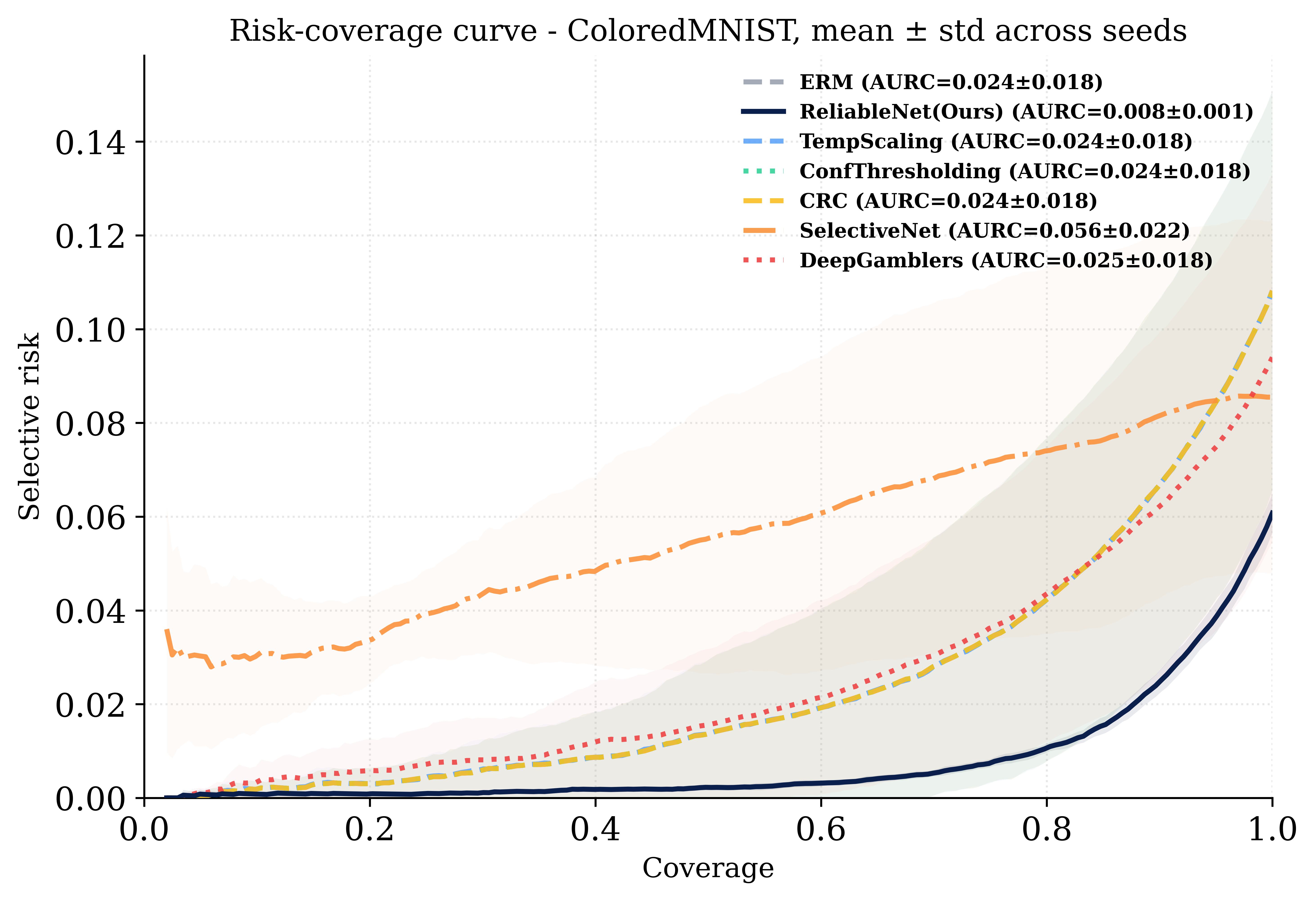}
        \caption{ColoredMNIST. ReliableNet attains the lowest and most stable AURC ($0.008\pm0.001$).}
        \label{fig_rc_mnist}
    \end{subfigure}\hfill
    \begin{subfigure}[b]{0.48\textwidth}
        \includegraphics[width=\textwidth]{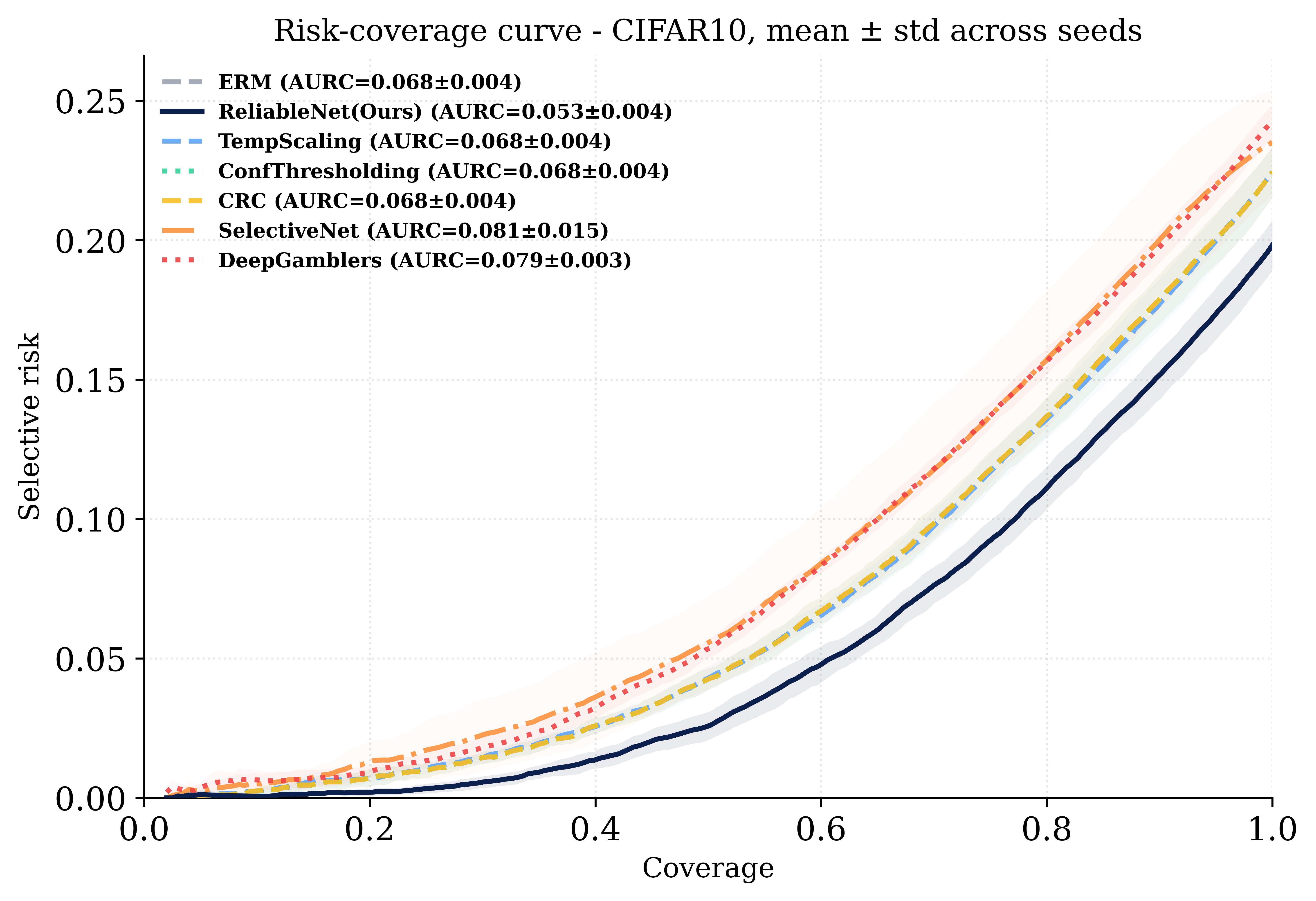}
        \caption{CIFAR-10. ReliableNet attains the lowest AURC ($0.053\pm0.004$).}
        \label{fig_rc_cifar}
    \end{subfigure}

    \vspace{0.6em}
    \begin{subfigure}[b]{0.48\textwidth}
        \includegraphics[width=\textwidth]{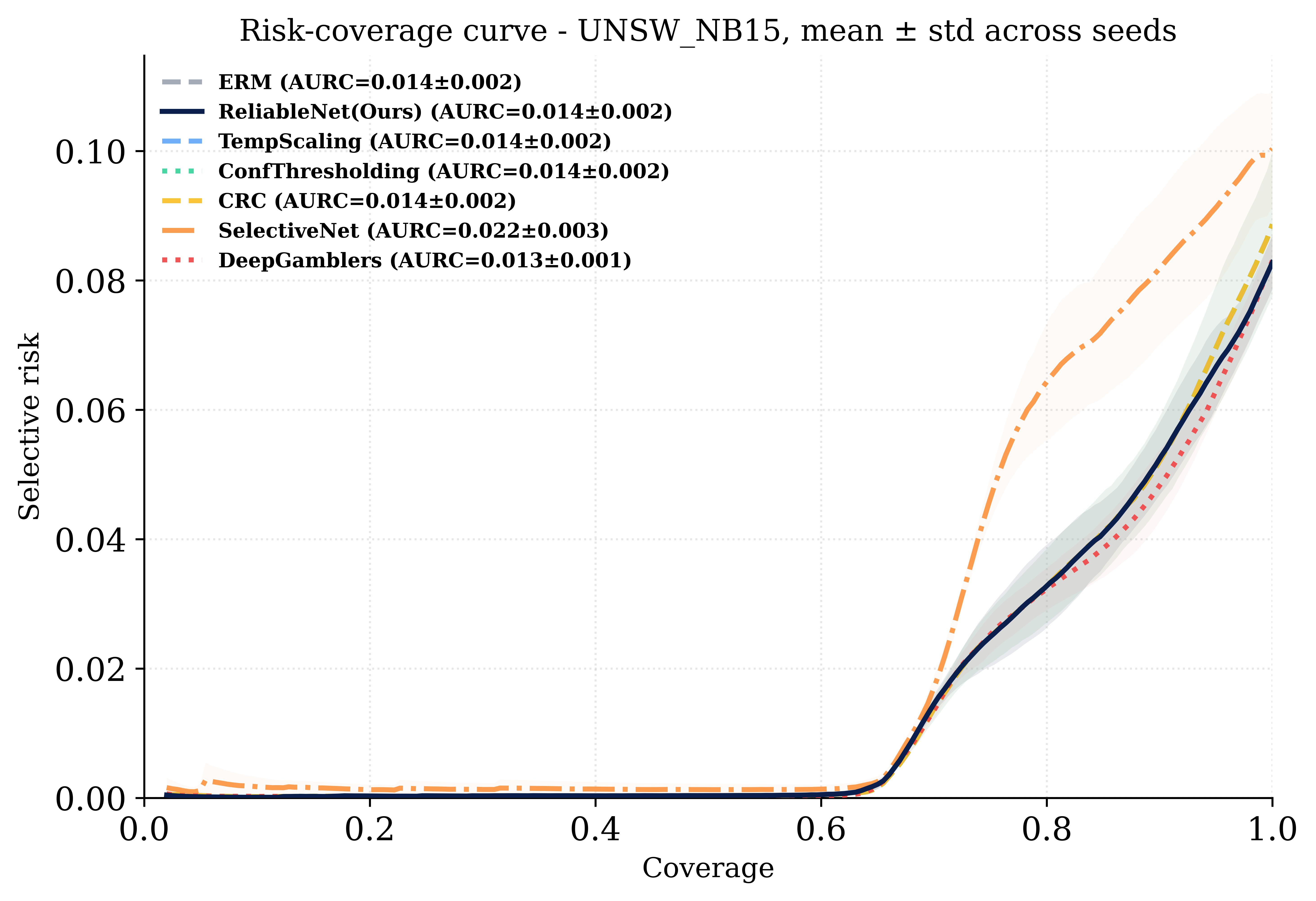}
        \caption{UNSW-NB15. ReliableNet ties the post-hoc baselines ($0.014\pm0.002$); DeepGamblers is marginally lower ($0.013\pm0.001$).}
        \label{fig_rc_unsw}
    \end{subfigure}\hfill
    \begin{subfigure}[b]{0.48\textwidth}
        \includegraphics[width=\textwidth]{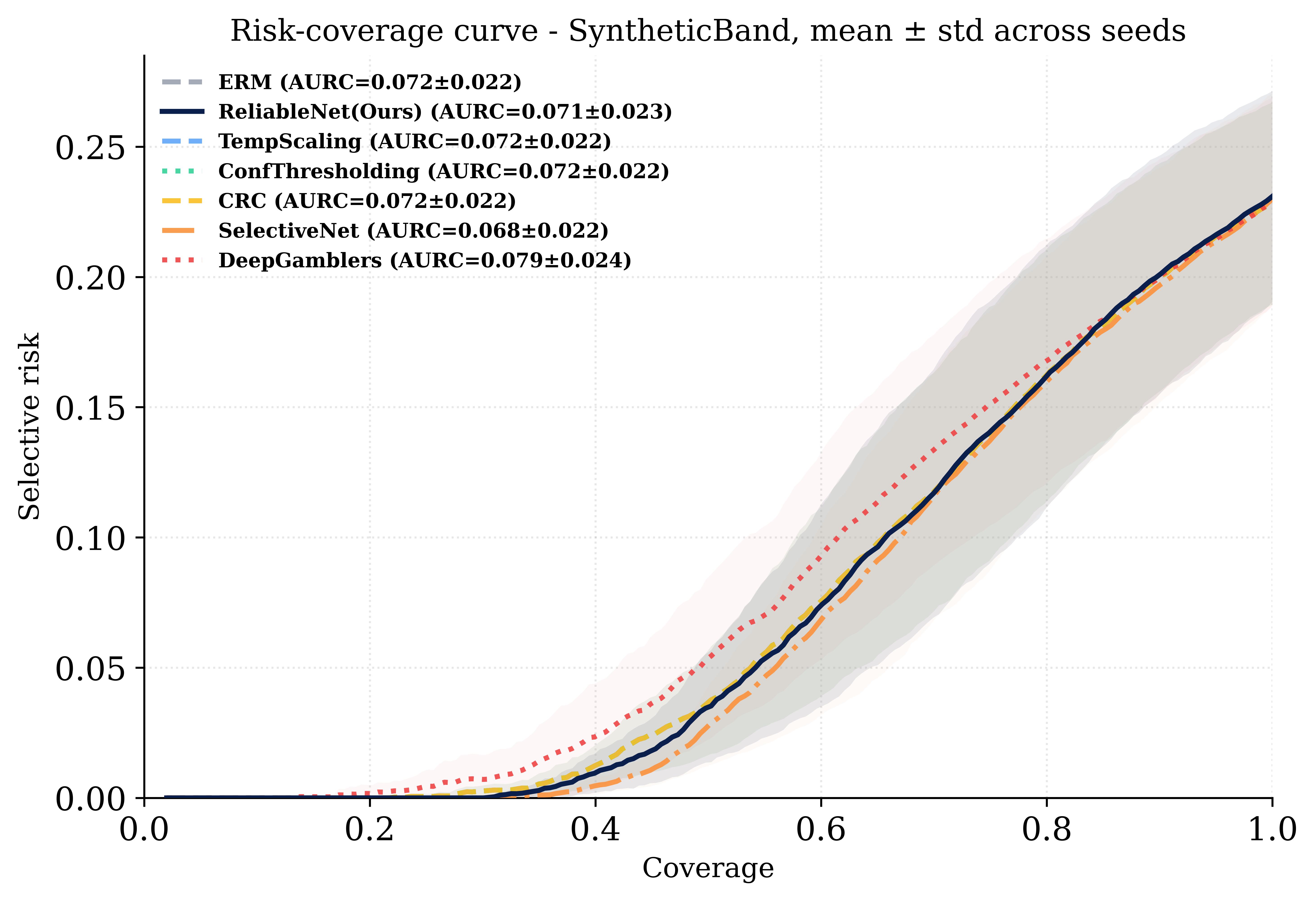}
        \caption{SyntheticBand. Curves overlap within std; ReliableNet ($0.071\pm0.023$) is competitive with the best (SelectiveNet, $0.068\pm0.022$).}
        \label{fig_rc_synthetic}
    \end{subfigure}

    \caption{Risk-coverage curves across four tabular and two image datasets, averaged over five random seeds under the stress-test conditions of Section~\ref{subsec_ood_aggre_tests}. Each curve sweeps coverage over $[0,1]$, so all methods are compared at identical coverage levels; lower curves and lower AURC therefore indicate better confidence ordering independent of any single operating point. ReliableNet improves selective ordering on Adult, ColoredMNIST, German Credit, and CIFAR-10, and remains competitive on SyntheticBand and UNSW-NB15.}
    \label{fig:rc_all}
\end{figure}
At matched coverage, ReliableNet attains the lowest selective risk over most of the coverage range and the lowest AURC on Adult, ColoredMNIST, German Credit, and CIFAR-10, despite never optimizing AURC directly. This indicates that the JCW constraint does more than lower confident-error rate at a single operating point: it reshapes the confidence scores so that correct predictions are ranked above incorrect ones across the whole coverage spectrum. On UNSW-NB15 and SyntheticBand the methods are statistically indistinguishable (curves overlap within one standard deviation), with ReliableNet essentially tied for best. The trained selective baselines, SelectiveNet and DeepGamblers, degrade most under
shift, since their abstention mechanisms are learned on the source distribution and do not transfer. Temperature scaling, confidence thresholding, and CRC track ERM almost exactly, because as post-hoc procedures they recalibrate or threshold the same fixed classifier rather than changing the learned representation.

\newpage

\subsection{Reliability Under Increasing Shift Severity}

Having established aggregate performance, we now examine each dataset's stress test  individually, reporting JCW and coverage as a function of shift severity. The two must be  read together: JCW alone can be lowered simply by abstaining more, so a low confident-error  rate is only meaningful alongside the coverage at which it is achieved. Jointly, the two metrics reveal how each method responds to distribution shift, through increased abstention, retained margin relative to the constraint, or neither, and whether its empirical JCW remains below the risk budget $\alpha$ as shift severity increases beyond the training distribution.

\paragraph{Adult data (Education Shift Analysis).}
Figure~\ref{fig2_} evaluates all methods across female subgroups stratified by education level, ordered from lowest to highest. Panel (a) shows that ReliableNet maintains $\mathrm{JCW} \le 0.05$ across most education levels, except  (Doctorate , Masters, Prof-school) for high-education subgroups where all baselines (except CRC) violate the constraint substantially  at Masters and Doctorate level. This violation pattern is not random: higher education subgroups are better represented in the high-confidence region, making confident errors more likely. Panel (b) shows the cost: ReliableNet's coverage drops sharply for mid-education groups (HS-grad through Bachelors), reflecting increased abstention where the model is uncertain. This suggests that abstention is not uniform but targeted and principled where uncertainty resides the most.
\begin{figure}[H]
    \centering
    \begin{subfigure}[b]{0.4\textwidth}
        \centering
         % \centering
    \includegraphics[width=1.2\linewidth]{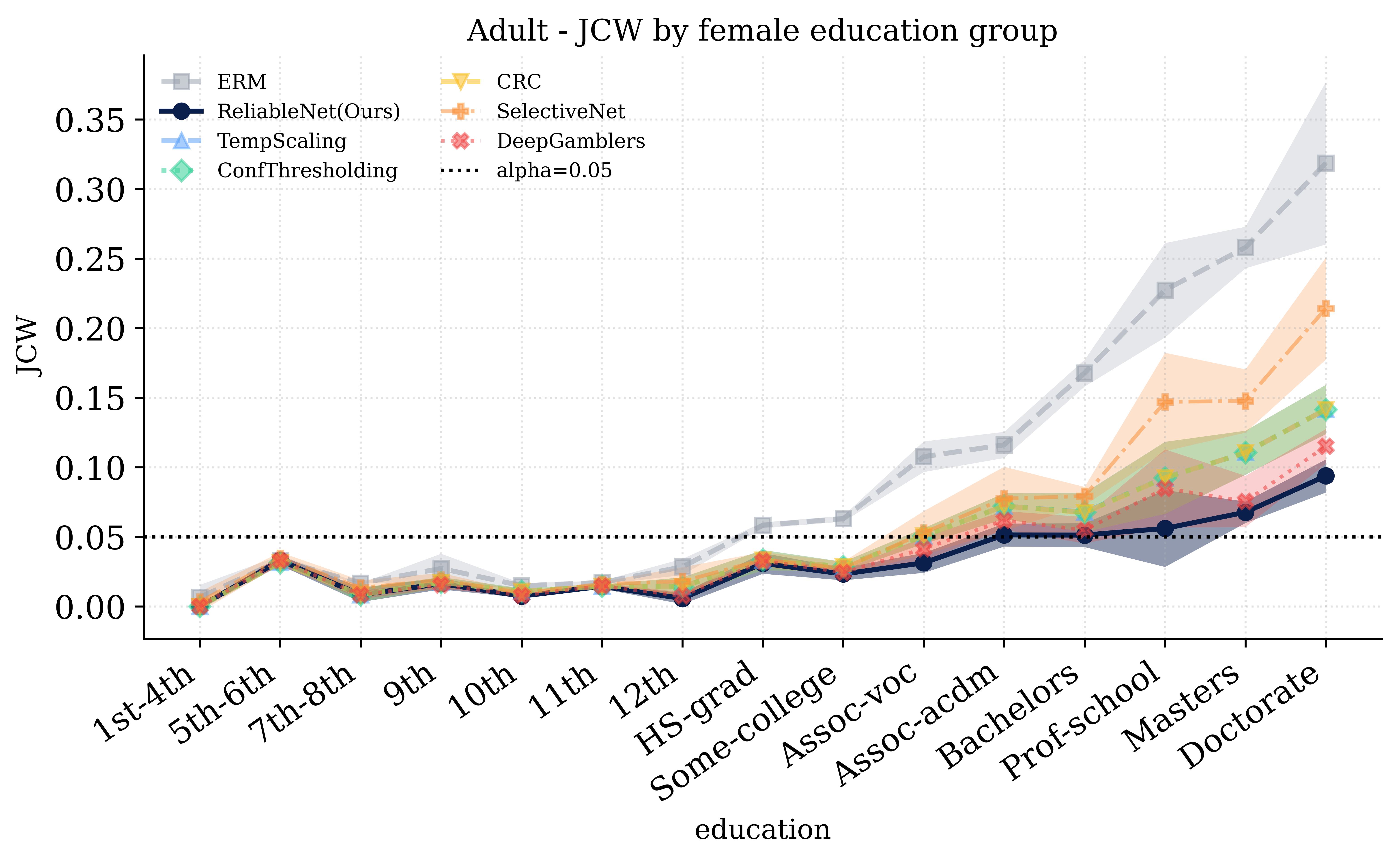}
    \caption{JCW across female subgroups stratified by
    education level.   }
    \end{subfigure}
    \hfill
    \begin{subfigure}[b]{0.55\textwidth}
    \centering
    \includegraphics[width=0.95\linewidth]{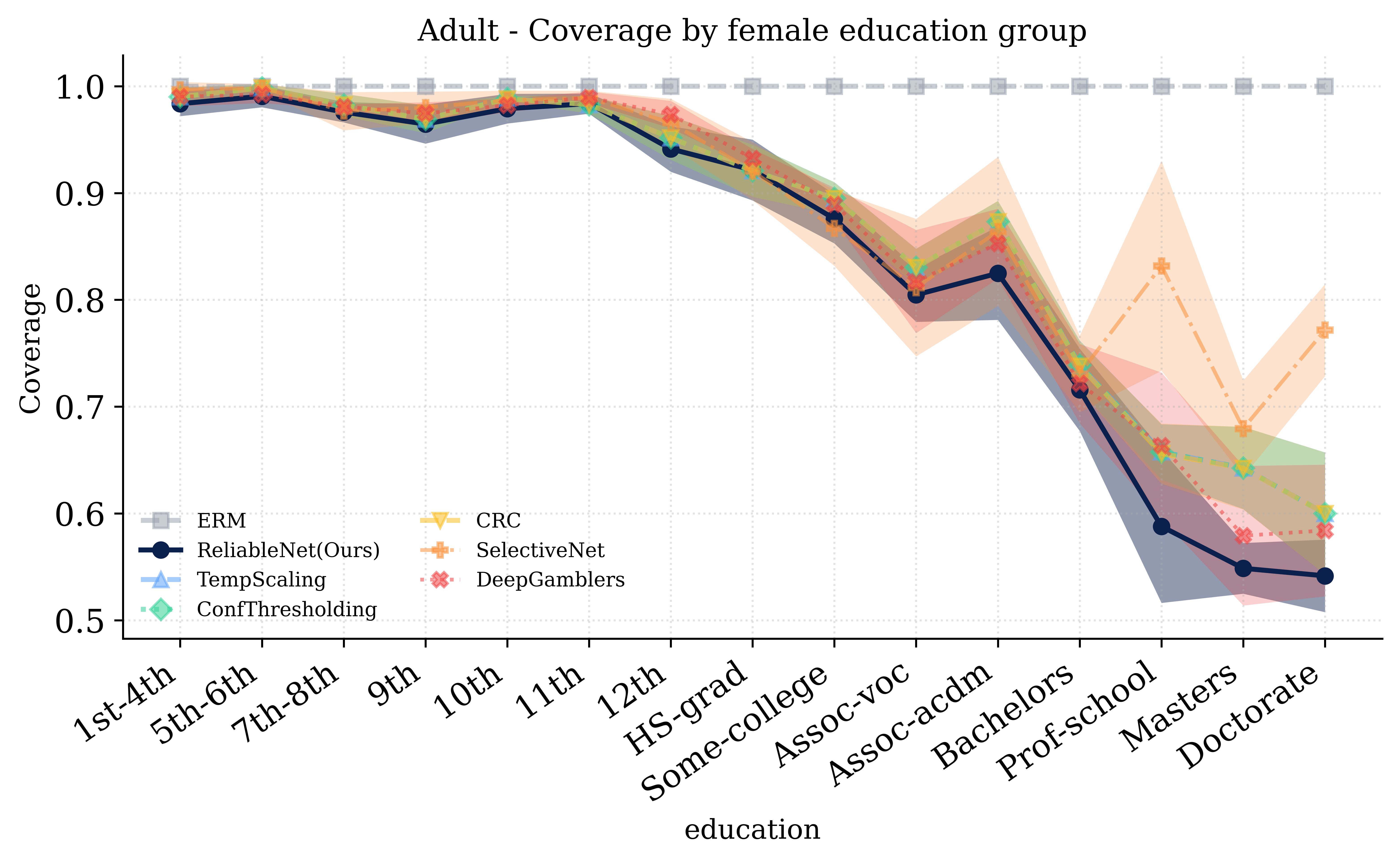}
    \caption{ Coverage across female subgroups stratified by
    education level. }
    \end{subfigure}
    \caption{JCW and Coverage across female subgroups stratified by
    education level}
    \label{fig2_}
\end{figure}
\paragraph{Synthetic Data}
We train and evaluate in $d=10$ with $n_{\mathrm{train}}=5000$, $n_{\mathrm{test}}=3000$ (all examples are $10$-dimensional), sweeping $\gamma\in\{0.1,\dots,1.3\}$ For the geometric visualization (Fig.~\ref{fig7}) we instantiate the identical generator and
training pipeline in $d=2$ with the \emph{same} $s$ and $\gamma$ and its own unit vector $w\in\mathbb{R}^2$ ($n=3000$). The ambiguous band was shaded as the strip between the parallel lines $\langle w,x\rangle=\pm\gamma$, and mark confident errors ($\hat y\neq\tilde y$, confidence $\ge\varepsilon$). \noindent Figures~\ref{fig6} and~\ref{fig7} jointly explain this result. 
At the same threshold $\varepsilon^* = 0.66$, ERM produces 19  confident wrong predictions concentrated inside the ambiguous  band while ReliableNet produces zero, having learned to suppress  confidence near the boundary. This is confirmed  in  Figure~\ref{fig6}: ERM saturates at near-perfect confidence at  distance $\approx 0.7$ from the boundary, while ReliableNet's confidence  rises monotonically with distance as we get away from the ambiguous band. This shows that ReliableNet learns to suppress confidence near the ambiguous region.

\begin{figure}[H]
    \centering
    \begin{subfigure}[b]{0.4\textwidth}
        \centering
        \includegraphics[width=\linewidth]{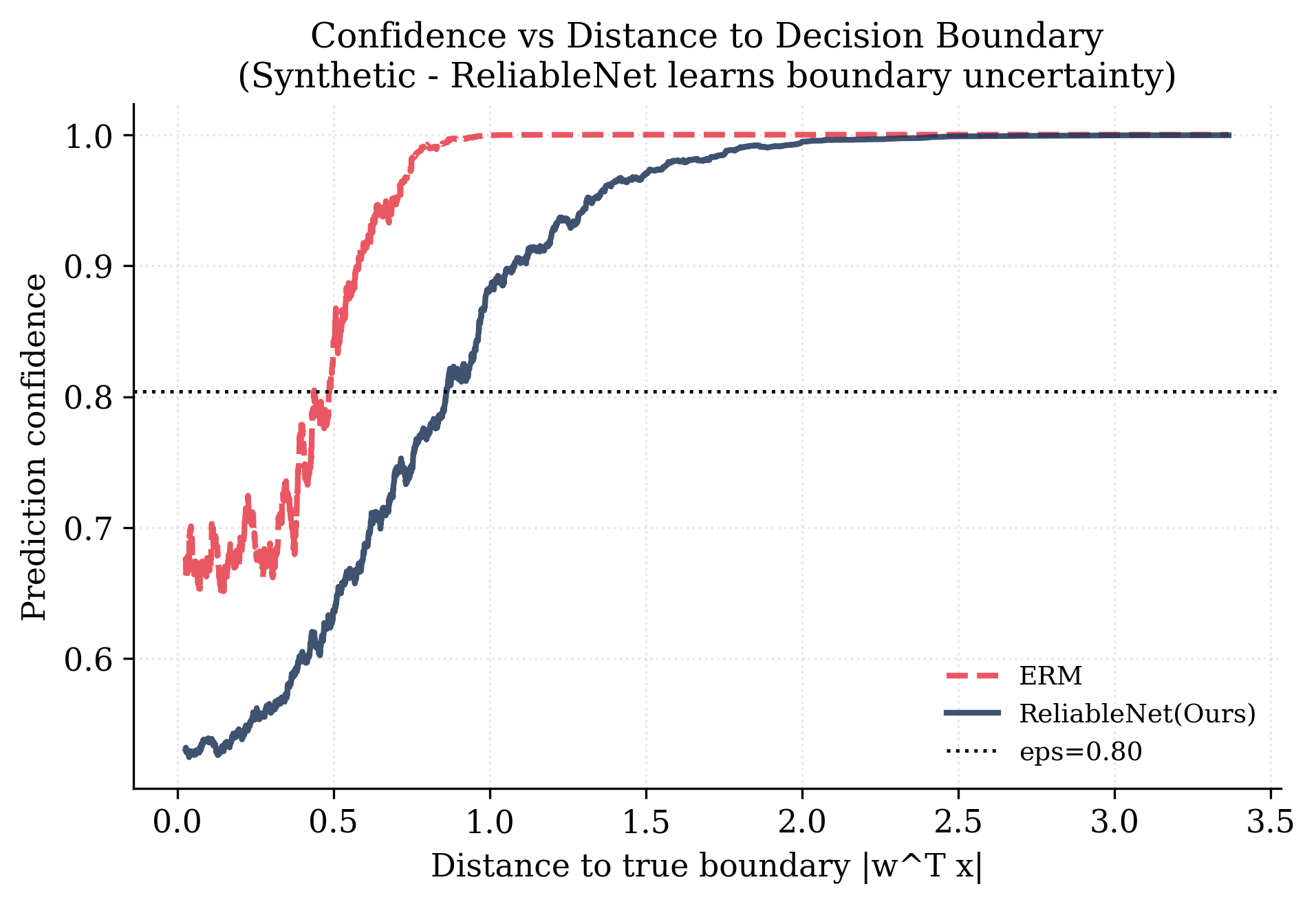}
        \caption{Prediction confidence as a function of distance 
        to the true decision boundary $|w^\top x|$. ERM remains 
        overconfident near the boundary where label noise is 
        irreducible. ReliableNet learns to be uncertain there and 
        recovers high confidence only as distance increases.}
        \label{fig6}
    \end{subfigure}
    \hfill
    \begin{subfigure}[b]{0.55\textwidth}
    \centering
    \includegraphics[width=1.2\linewidth]{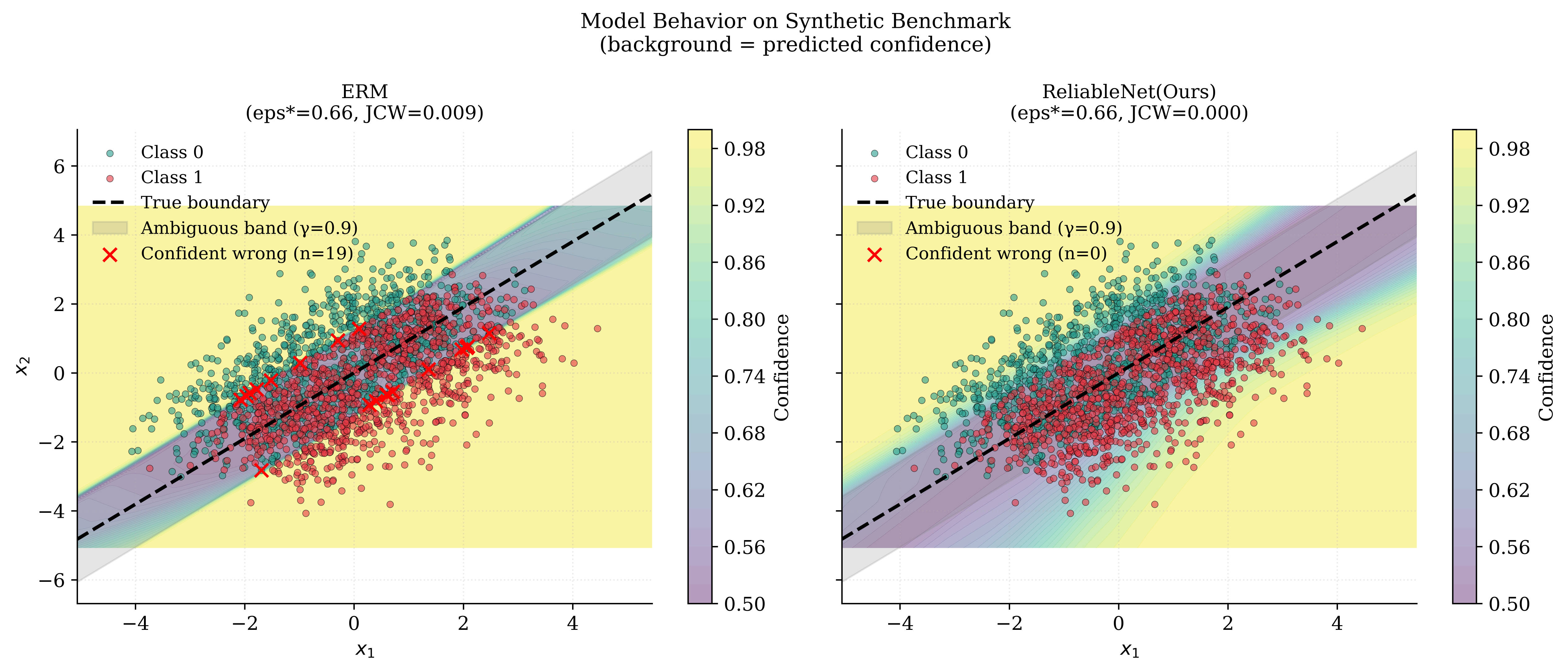}
    \caption{\textbf{ERM vs ReliableNet on the synthetic data on a 2D setting 
    ($\gamma=0.9$, $\varepsilon^*=0.66$).} Both models are evaluated at the same threshold. ERM produces 19 confident wrong predictions 
    (red crosses) concentrated inside the ambiguous band (grey), where   misclassification is irreducible. ReliableNet eliminates all  confident errors by suppressing confidence near the boundary,  visible as the wider low-confidence region (purple) in the right 
    panel.}
    \label{fig7}
    \end{subfigure}
    \caption{Geometric confidence calibration (
    $n_{\mathrm{test}}=3{,}000$, $\alpha=0.05$).}
    \label{fig_synthetic_reliability}
\end{figure}

\begin{figure}[H]
    \centering
    \begin{subfigure}[b]{0.4\textwidth}
        \centering
         % \centering
    \includegraphics[width=1.2\linewidth]{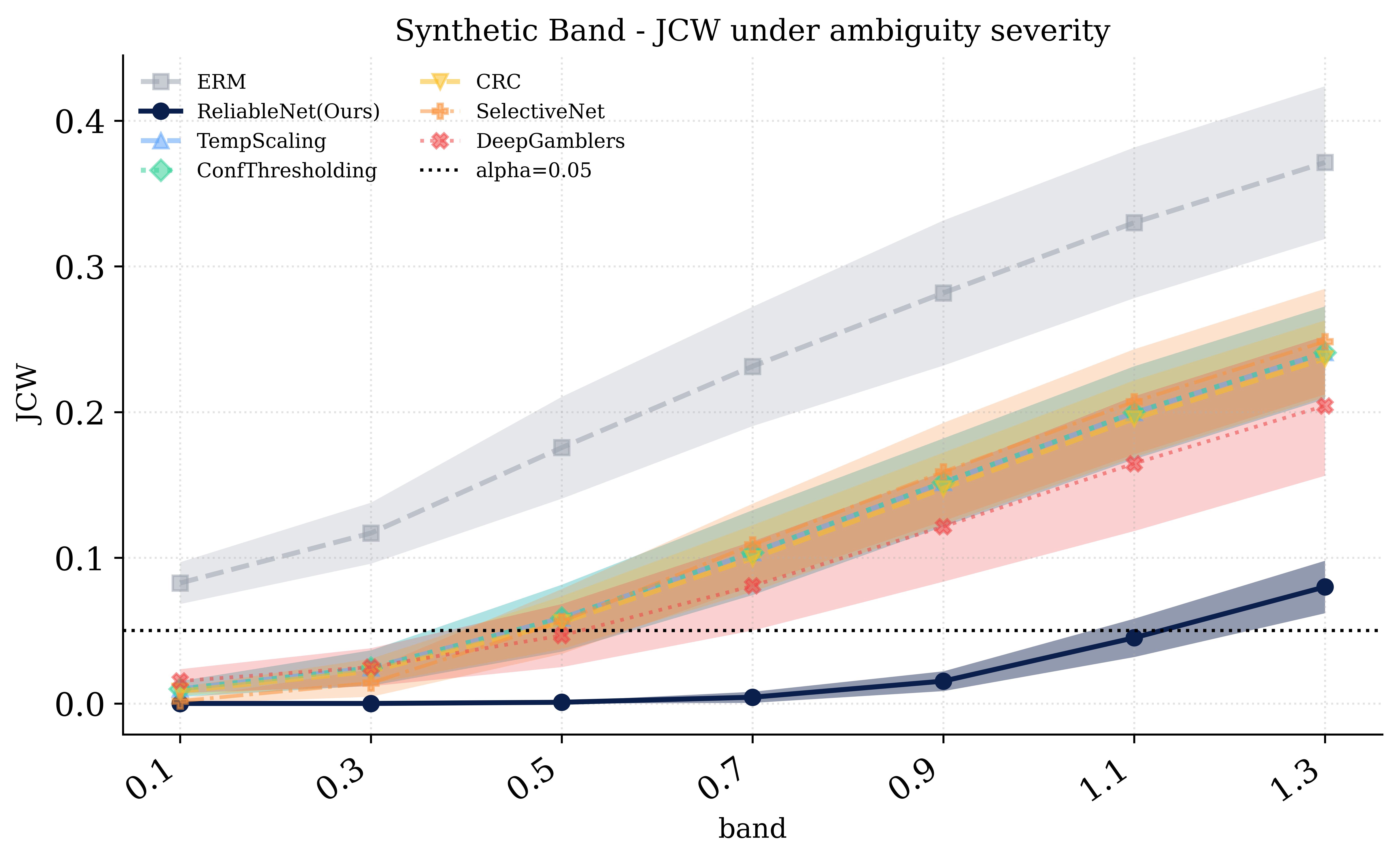}
    \caption{JCW vs ambiguity band width $\gamma \in 
        \{0.1,0.3,0.5,0.7,0.9,1.1,  1.3\}$ for $\alpha = 0.05$. All methods 
        except ReliableNet violate the threshold from 
        $\gamma = 0.5$ onward. ReliableNet maintains the lowest 
        JCW across all severity levels.}
    \end{subfigure}
    \hfill
    \begin{subfigure}[b]{0.55\textwidth}
    \centering
    \includegraphics[width=0.95\linewidth]{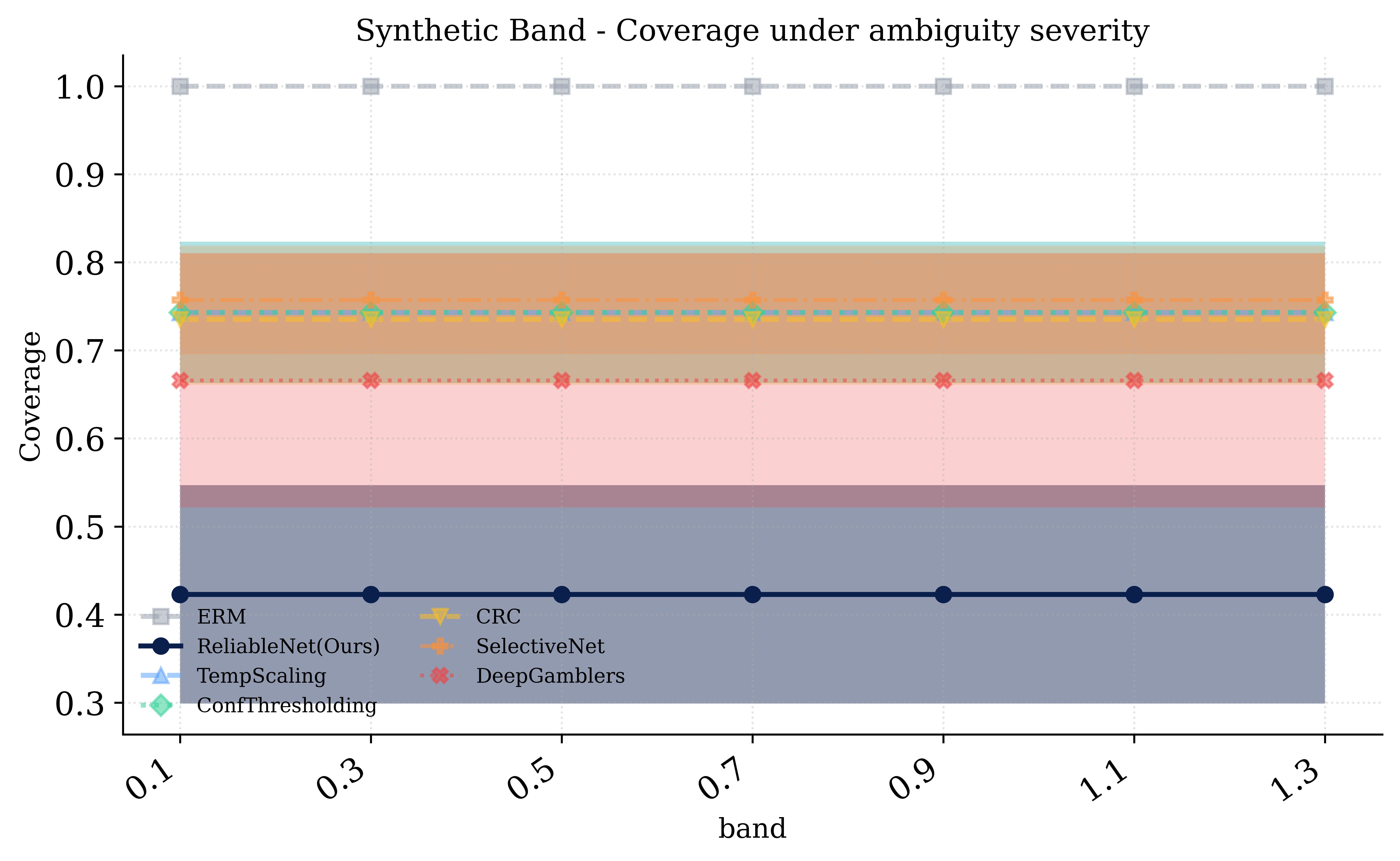}
    \caption{ Coverage vs ambiguity band width $\gamma \in \{0.1,0.3,0.5,0.7,0.9,1.1,  1.3\}$ for $\alpha = 0.05$. }
    \end{subfigure}
    \caption{Synthetic Band ambiguity sweep}
    \label{fig2_synth}
\end{figure}
% \paragraph{Ambiguity severity (Synthetic Band).}
Figures~\ref{fig2_synth} shows that coverage is invariant to the band parameter for all methods.  This is because the sweep perturbs $p(y\mid x)$ while leaving $p(x)$ unchanged: no confidence-thresholding rule can  adapt by abstaining more, so robustness here is determined entirely by the margin a method  holds against the constraint. The post-hoc baselines, which select the most permissive  threshold satisfying $\mathrm{JCW}\le\alpha$ on validation data, sit \emph{on} the  constraint boundary at the training band and consequently violate it under any increase in  severity, reaching $4$-$5\alpha$ at band $1.3$. ReliableNet attains  $\mathrm{JCW}\approx 0.001$ at the same band,   about $2\%$ of the budget,   and remains below the risk budget up to band $\approx 1.1$, roughly twice the severity range tolerated by any  baseline, while degrading gracefully rather than linearly. Constrained training drives the  solution into the interior of the feasible set, and this slack,   obtained at a fixed  coverage cost incurred once at training time,   is what converts into robustness under  shift.

\paragraph{German Age Shift Analysis.}

Figure~\ref{fig_german} evaluates the young-age subpopulations (20-25, 25-30), under-represented at training time. ERM is severely unreliable here (JCW $0.32$-$0.37$), and none of the post-hoc or selective baselines holds the target either: temperature scaling, confidence thresholding, CRC, SelectiveNet, and Deep Gamblers all sit at JCW $0.06$-$0.09$, above $\alpha = 0.05$. ReliableNet is the only method that keeps confident-wrong risk below $\alpha$ on both groups, attaining JCW $\approx 0.03$ at approximately $25\%$ coverage. Confidence Thresholding and CRC retain slightly higher coverage, approximately $30$-$40\%$, but their JCW remains above the target. Thus, ReliableNet achieves the most favorable reliability-coverage trade-off in these subgroups, although part of its lower JCW is associated with its more conservative acceptance rate. Enforcing the constraint during training, rather than selecting a threshold post-hoc, is what attains the target level under shift.
\begin{figure}[H]
    \centering
    \begin{subfigure}[b]{0.48\textwidth}
        \centering
        \includegraphics[width=\linewidth]{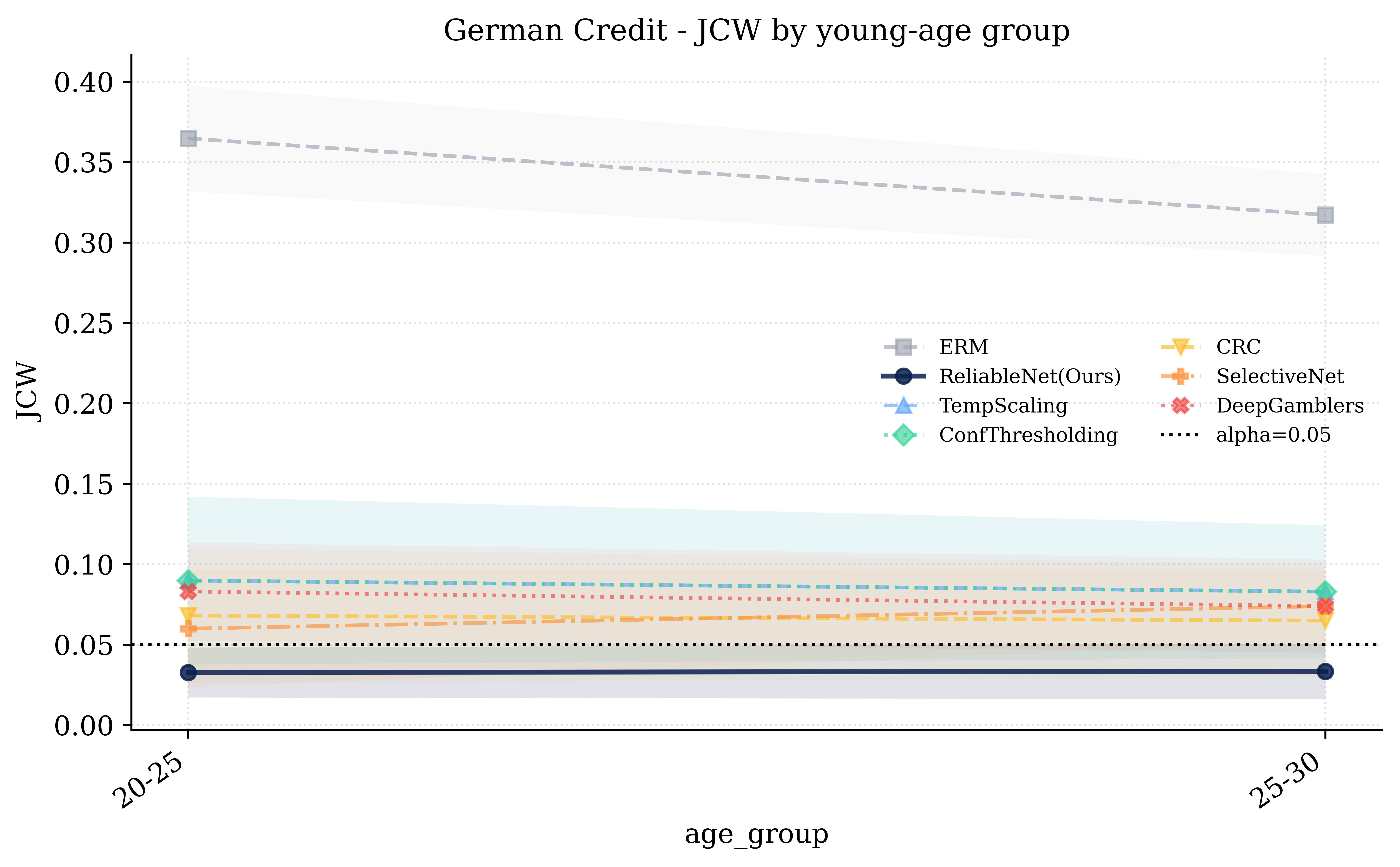}
        \caption{JCW calibration error ($\alpha = 0.05$).}
    \end{subfigure}
    \hfill
    \begin{subfigure}[b]{0.48\textwidth}
        \centering
        \includegraphics[width=\linewidth]{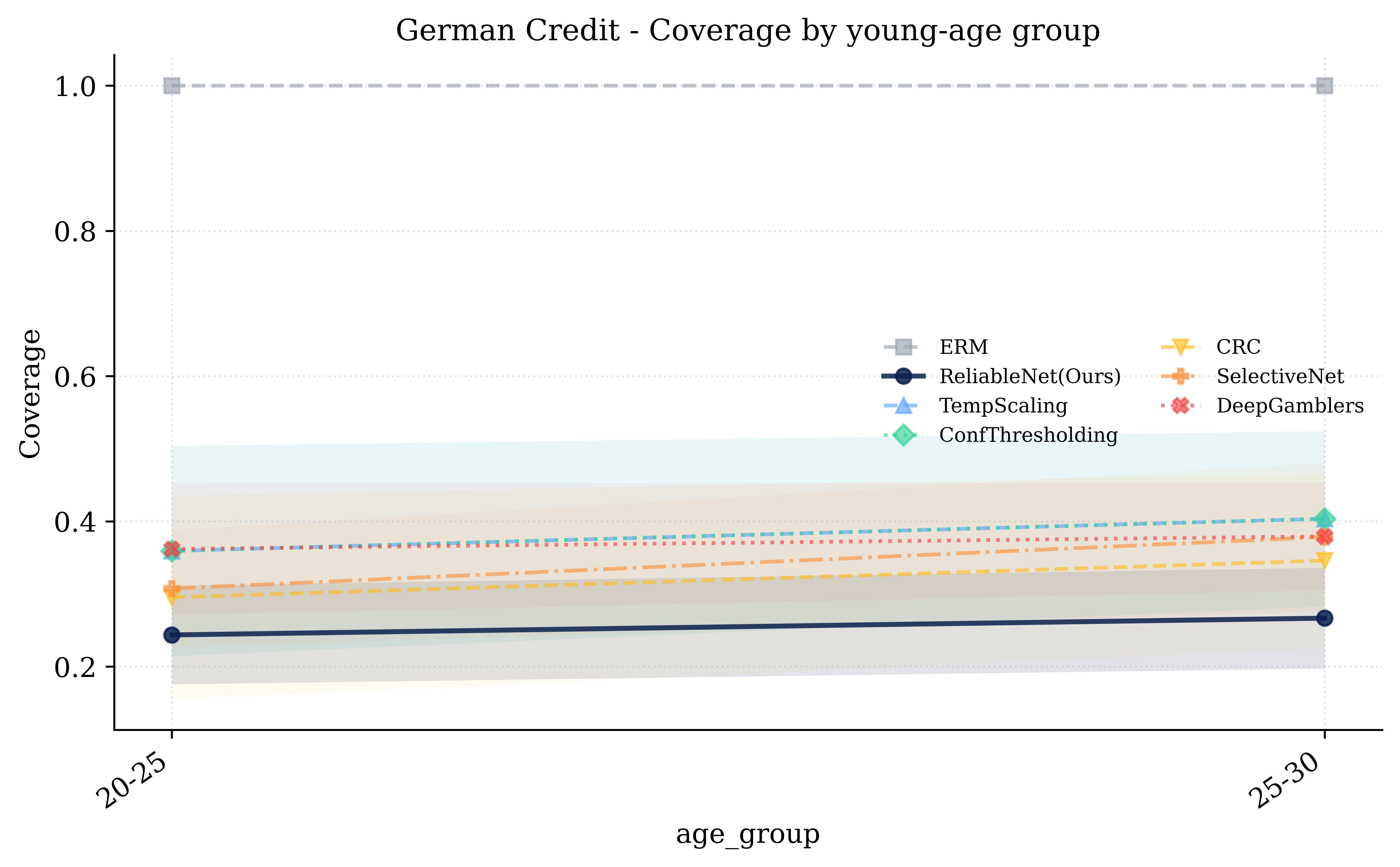}
        \caption{Empirical coverage.}
    \end{subfigure}
    \caption{Evaluation under age subpopulation shifts on German Credit (train: age $\geq 30$).}
    \label{fig_german}
\end{figure}

% Panel
% (c) confirms that this abstention is well-targeted, 
% $\mathrm{Acc_{HC}}$ remains near $1.0$ for ReliableNet across
% all subgroups while baselines degrade significantly at higher
% education levels.

\paragraph{Colored MNIST.}
Figure~\ref{fig_mnist_ind} reports JCW and coverage as the test colour, label correlation $\rho_{\text{test}}$ decreases from $0.9$ to
$0.1$, progressively reversing the spurious colour cue learned during training (lower $\rho_{\text{test}}$ indicates stronger shift). Coverage
declines gradually for the selective methods (ReliableNet: $0.88\rightarrow0.67$), while JCW increases as the colour cue becomes increasingly misleading. Across the full severity range, ReliableNet maintains the lowest JCW, reaching approximately $0.17$ under full inversion, compared with $0.24$-$0.26$ for the strongest baselines and $0.37$ for ERM. Since ReliableNet does not achieve this reduction at higher coverage, the improvement is not explained by accepting more inputs and is consistent with more reliable confidence ordering under spurious-feature shift. At full inversion, no method remains empirically below the JCW budget, although ReliableNet exhibits the slowest
degradation.

\begin{figure}[H]
    \centering
    \begin{subfigure}[b]{0.48\textwidth}
        \centering
        \includegraphics[width=\linewidth]{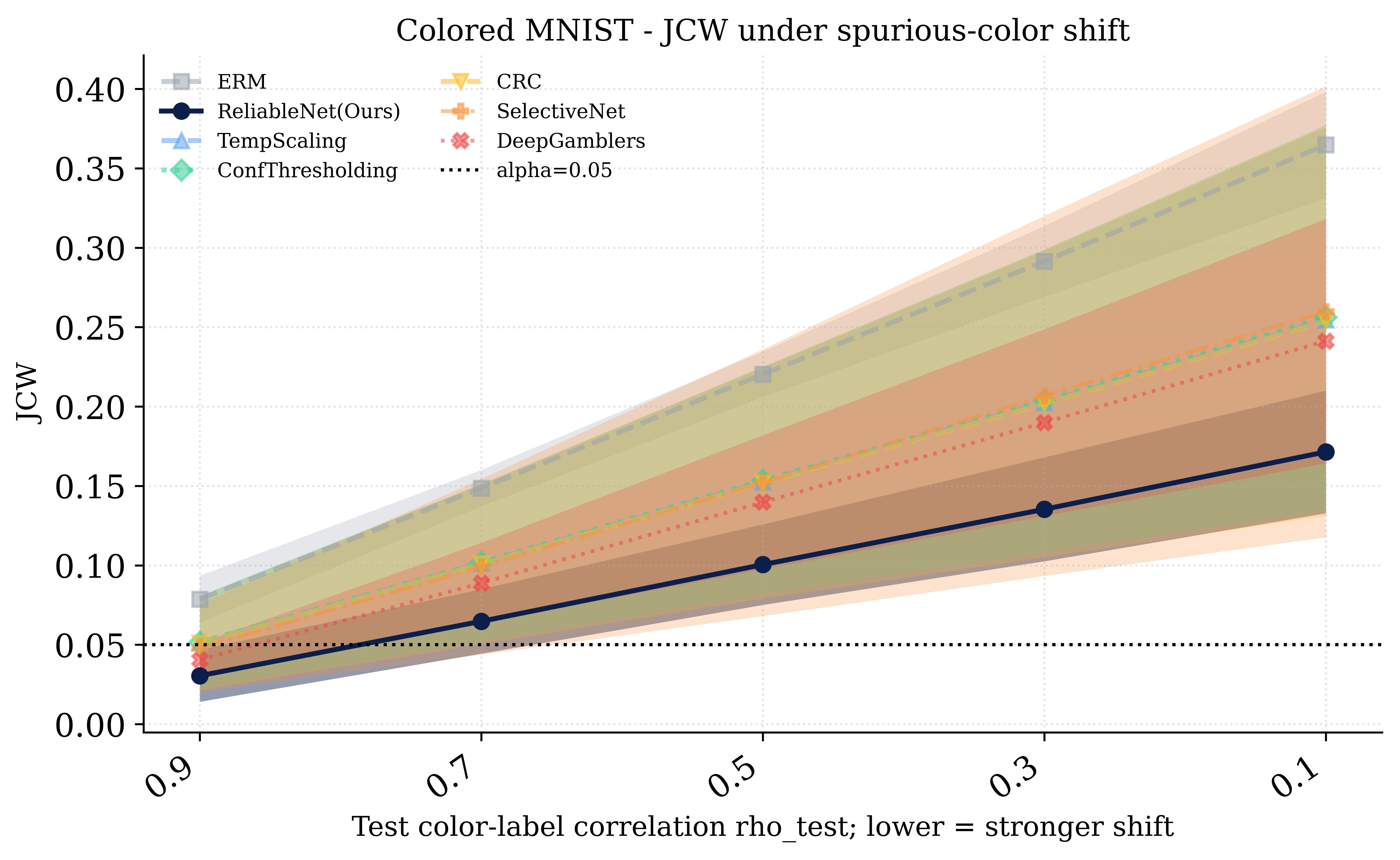}
        \caption{JCW  error at sweeping $\rho_{\text{test}}$.}
    \end{subfigure}
    \hfill
    \begin{subfigure}[b]{0.48\textwidth}
        \centering
        \includegraphics[width=\linewidth]{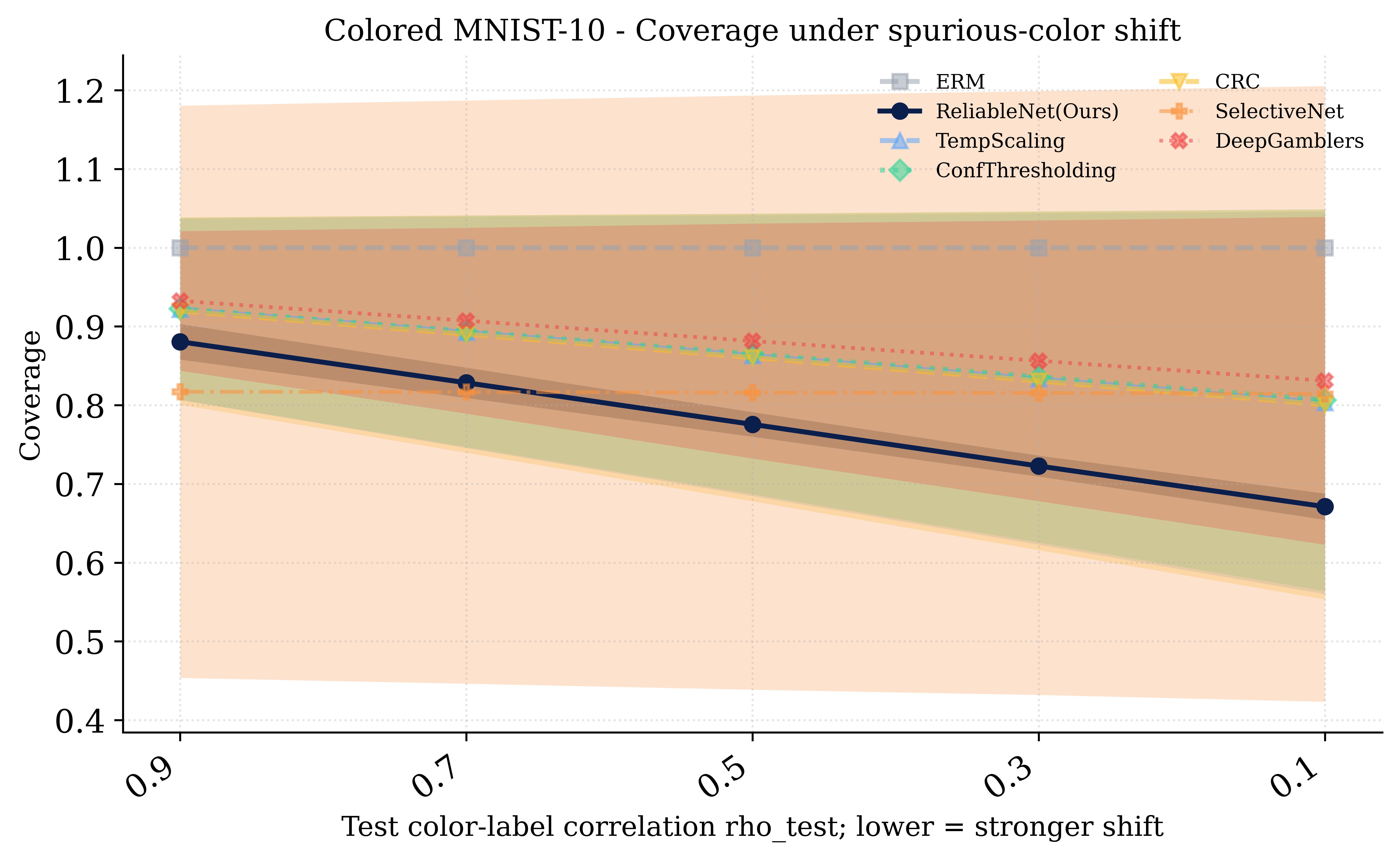}
        \caption{Empirical coverage at sweeping $\rho_{\text{test}}$.}
    \end{subfigure}
    \caption{Evaluation under sweeping $\rho_{\text{test}}$.}
    \label{fig_mnist_ind}
\end{figure}

\paragraph{CIFAR-10 (Noise severity).}
Figure  ~\ref{fig_cifar10} assesses all the methods under input  Gaussian noise  perturbation. 
Here, coverage responds to the  shift: all selective methods reject increasingly many inputs as severity grows. Robustness  here, therefore, combines margin with adaptive rejection. At $\sigma=0$ ReliableNet attains  $\mathrm{JCW}\approx0.006$, a sixth of the budget $\alpha=0.03$, whereas the post-hoc  baselines sit on the constraint boundary by construction; at $\sigma=0.05$ this slack is 
decisive, as ReliableNet remains below the risk budget ($\approx0.022$) while every competing method has  already violated ($2$-$2.7\alpha$, and $0.34$ for ERM). The ordering persists at all  severities, with ReliableNet reducing confident-error rates by $2$-$3\times$ relative to  the strongest baseline. Notably, baseline coverage \emph{increases} between $\sigma=0.1$ and 
$\sigma=0.25$ while their JCW rises,   the signature of overconfidence under corruption,    whereas ReliableNet's coverage remains more stable, indicating that constrained training yields  confidence that degrades more faithfully under shift.

\begin{figure}[H]
    \centering
    \begin{subfigure}[b]{0.48\textwidth}
        \centering
        \includegraphics[width=\linewidth]{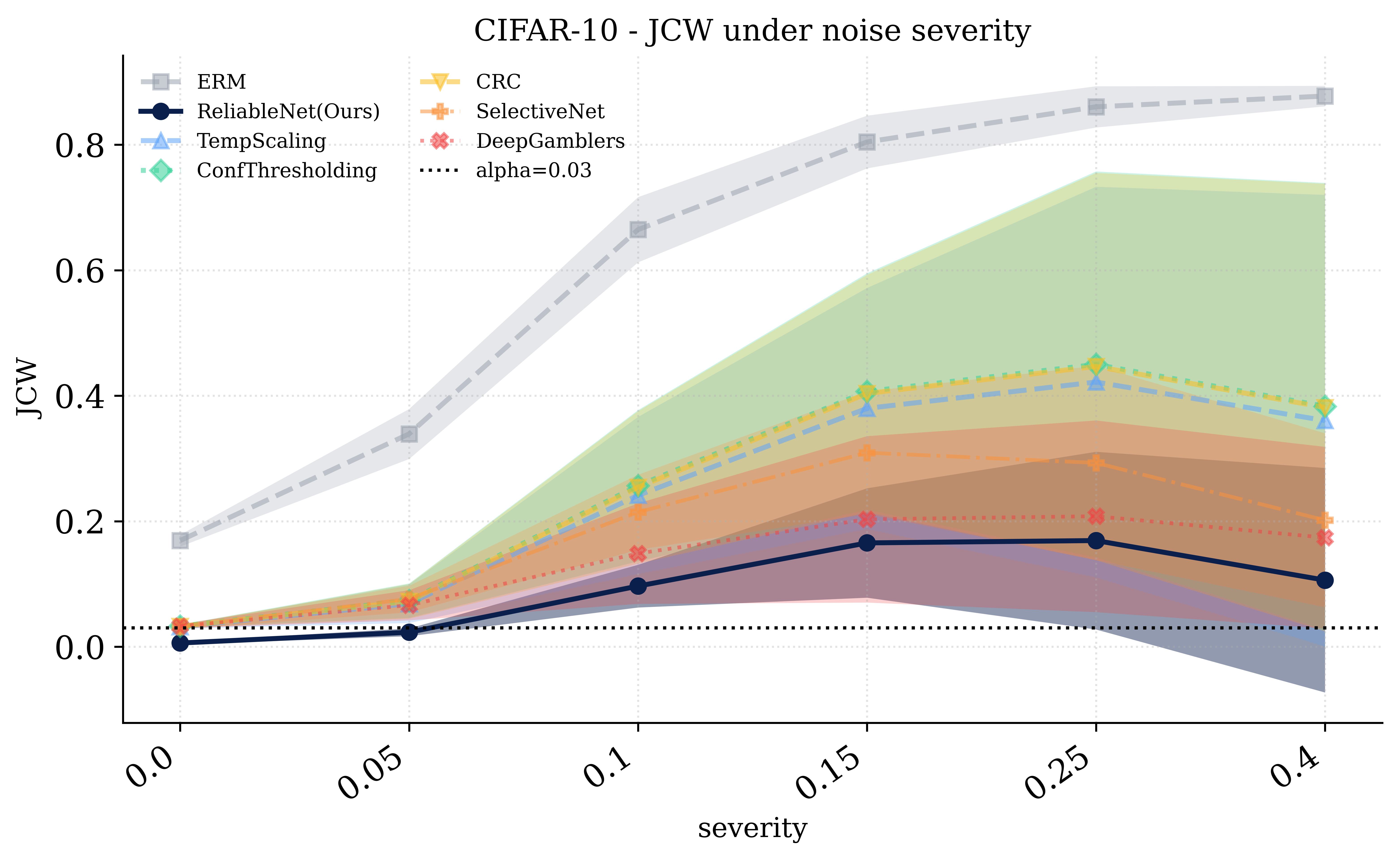}
        \caption{JCW  error ($\alpha = 0.03$).}
    \end{subfigure}
    \hfill
    \begin{subfigure}[b]{0.48\textwidth}
        \centering
        \includegraphics[width=\linewidth]{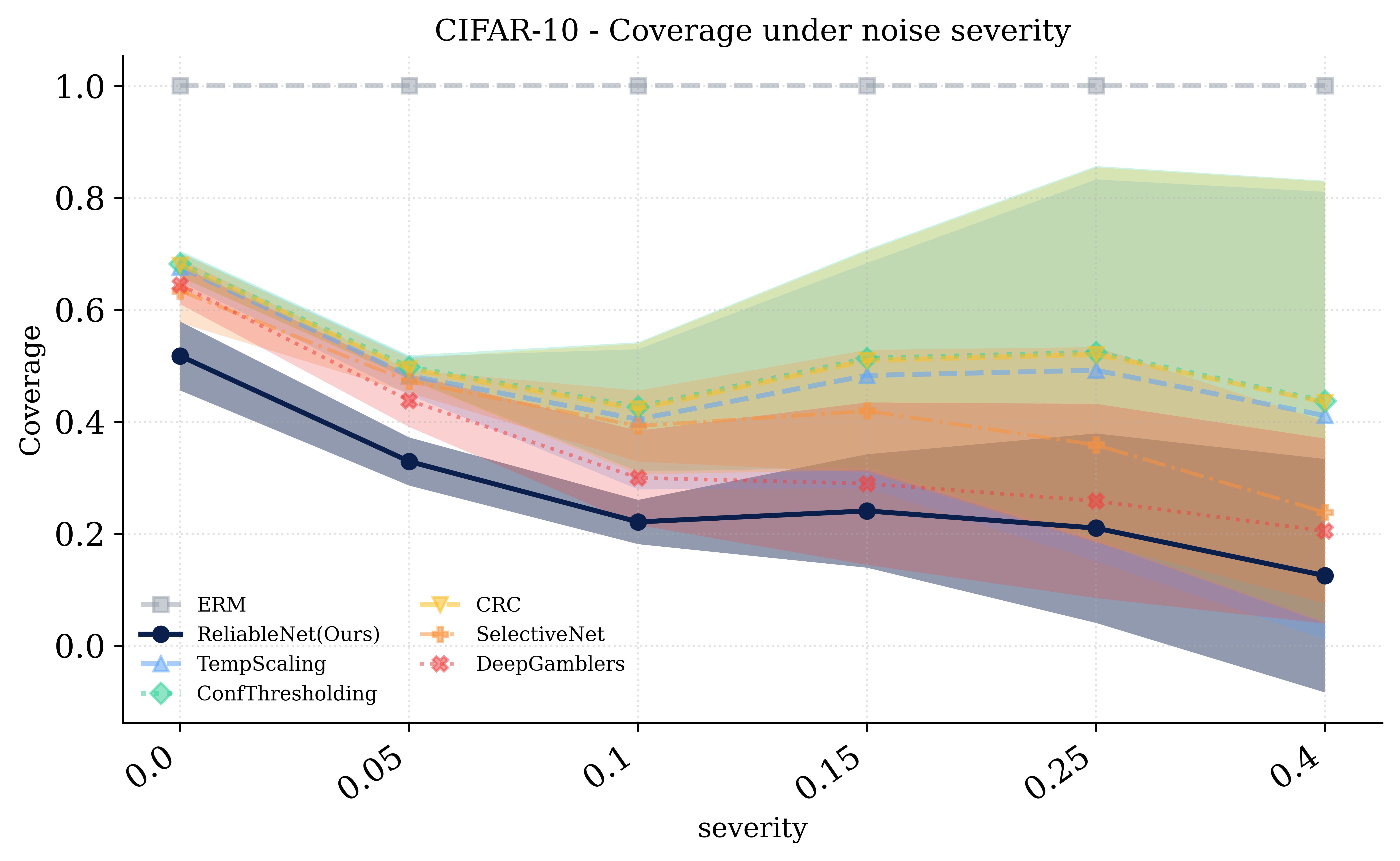}
        \caption{Empirical coverage.}
    \end{subfigure}
    \caption{Evaluation under Gaussian noise injection (varying noise standard deviation $\sigma$)  on CIFAR10 (train ID data).}
    \label{fig_cifar10}
\end{figure}

\paragraph{UNSW-NB15 Unseen-Attack Analysis.}
% In this experiments  we notice that in most unseen attacks ReliableNet manages to keep its JCW below the specified budget ($\alpha=0.05$), except for the Analysis attack.
Figure \ref{fig_un} decomposes the UNSW-NB15 unseen-attack split by family, reporting JCW and coverage  respectively.  Coverage alone suggests excessive conservatism, ReliableNet retains only 45\% of Reconnaissance and 37\% of Shellcode samples versus $78-100\%$ for the baselines, but these are precisely the families where confidence-based baselines fail (JCW up to $2 \alpha$; ERM 0.19 and 0.13), while ReliableNet attains 0.012 and 0.007. Where confidence remains informative (Backdoor, Worms), it retains 90-96\% coverage while providing the lowest JCW. This gives a clear message: abstention is targeted, not uniform.
\begin{figure}[H]
    \centering
    \begin{subfigure}[b]{0.48\textwidth}
        \centering
        \includegraphics[width=\linewidth]{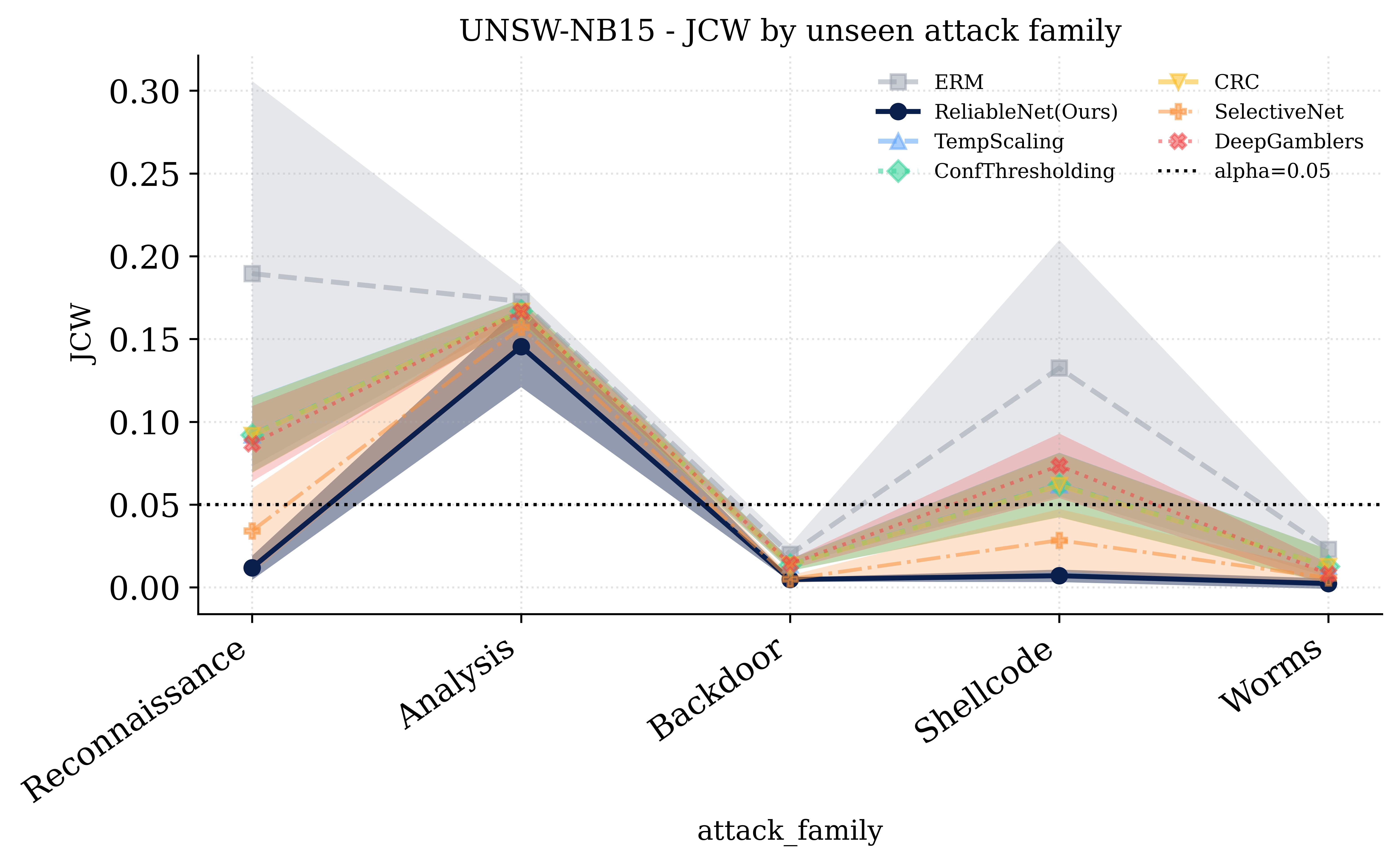}
        \caption{JCW  error ($\alpha = 0.05$).}
    \end{subfigure}
    \hfill
    \begin{subfigure}[b]{0.48\textwidth}
        \centering
        \includegraphics[width=\linewidth]{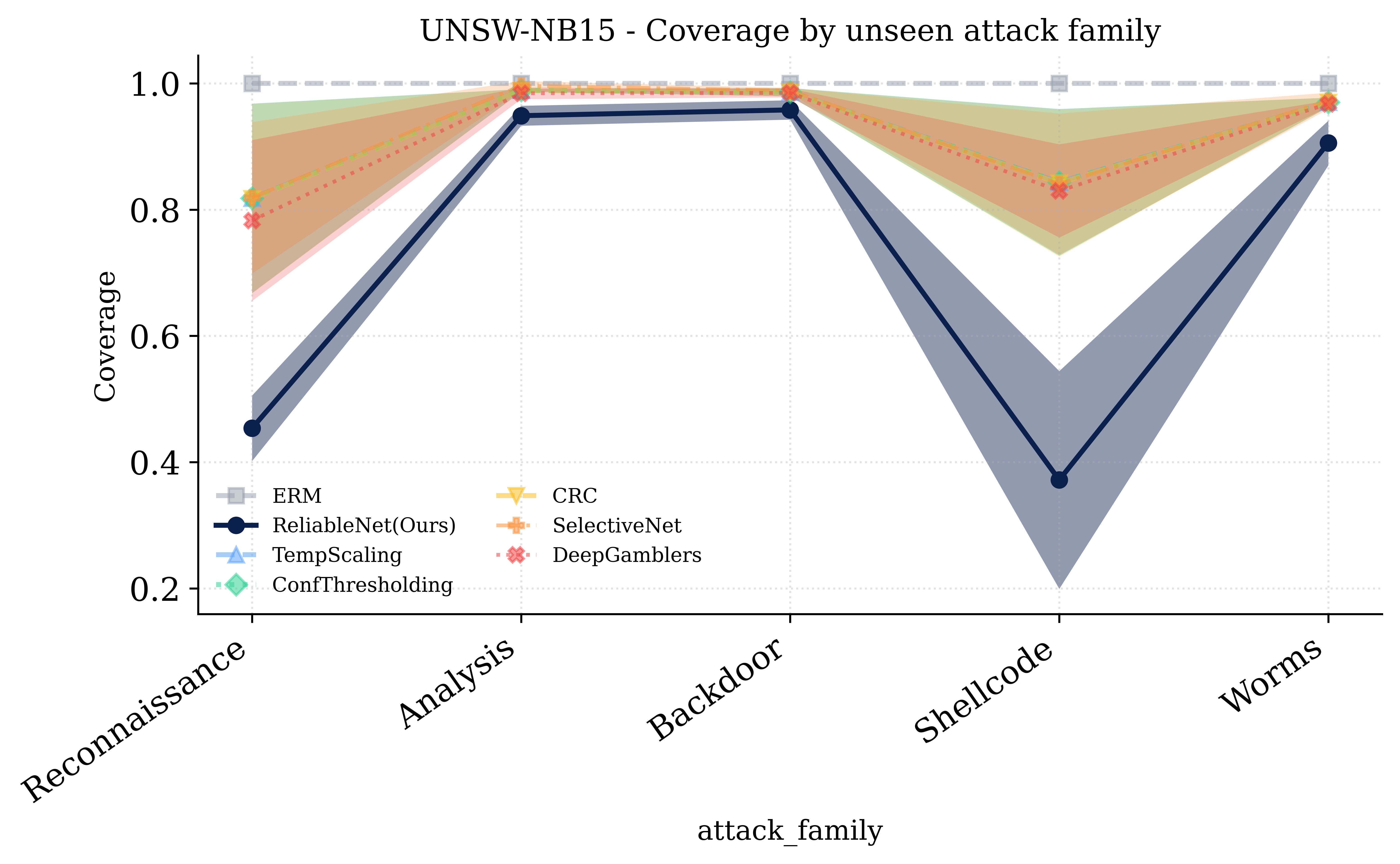}
        \caption{Empirical coverage.}
    \end{subfigure}
    \caption{Evaluation under unseen attacks.}
    \label{fig_un}
\end{figure}

\section{Conclusion}
\label{sec_conclusion}

High-confidence misclassification is not an ordinary prediction error. It is a reliability failure that can suppress abstention, human review, or other corrective mechanisms precisely when the model is wrong. This paper introduced \textbf{ReliableNet} to address this failure mode directly. Rather than treating reliability as a post-hoc calibration problem, we formulate the probability of being simultaneously confident and wrong as a probabilistic constraint and incorporate this constraint into ERM training.
The central idea is that confident-error control admits a clean probabilistic formulation through the Joint Confident-Wrong probability, $\mathrm{JCW}$. By encoding the confident-wrong event through the scalar violation function $g_\theta$, the target reliability requirement becomes a standard probabilistic constraint. We then adapt the inner approximation of probabilistic-constrained optimization to obtain a smooth conservative surrogate that can be optimized by gradient-based training. At the population level, feasibility of the smooth inner surrogate implies feasibility of the original JCW constraint. At the finite-sample level, we further provide certification tools that connect empirical feasibility to population reliability, including a held-out Clopper-Pearson certificate for the hard JCW event.

The experiments support the main claim that confident-error control can be learned during training. Across tabular and image-based benchmarks, ReliableNet consistently reduces high-confidence errors while maintaining competitive accuracy under the evaluated distribution shifts. The in-distribution certification results show that ReliableNet satisfies the held-out JCW certificate on all seeds for all datasets, unlike the benchmark methods. The comparison with ERM, temperature scaling, confidence thresholding, conformal risk control, SelectiveNet, and Deep Gamblers further shows that ReliableNet is not merely a post-hoc thresholding rule: it changes the learned confidence behavior itself. This is especially visible in settings where the accepted region must become more selective to keep confident mistakes under control.

The results also clarify the relationship between ReliableNet, calibration, and selective prediction. ReliableNet does not optimize expected calibration error, coverage, or selective risk directly. Nevertheless, by suppressing confident errors, it often improves calibration in the high-confidence regime and yields strong accepted-sample reliability. The JCW decomposition explains this behavior: controlling $\mathrm{JCW}$ constrains the product of coverage and high-confidence error rate. A reliable model may therefore trade some coverage for substantially safer accepted predictions. This is not a weakness of the constraint but its intended behavior in decision-sensitive applications.

The theoretical analysis complements the empirical findings: the target confident-wrong event up to a null boundary set can be exactly encoded, ensuring that the probabilistic constraint controls the intended failure mode rather than an unrelated surrogate.  ReliableNet induces a selective prediction mechanism. The main ReliableNet tool, JCW control also bounds a weighted high-confidence calibration error under the overconfidence condition, providing a theoretical explanation for the calibration improvements observed in several experiments. The finite-sample feasibility results further distinguish the population inner approximation from the empirical learning problem and justify the use of independent certification folds for post-training reliability claims.

\paragraph{Limitations and future work.}
Some limitations remain. First, ReliableNet is developed here for classification with a scalar confidence threshold. Extending the same probabilistic-constrained principle to regression, structured prediction, and sequential decision-making is an important direction for future work. Second, although the experiments include tabular, synthetic, spurious-correlation, and image benchmarks, larger-scale image and foundation-model settings remain to be studied. Third, the uniform finite-sample surrogate certificate is theoretically useful but can be conservative for high-dimensional neural networks because it depends on parameter dimension and global Lipschitz constants. In practice, the held-out Clopper-Pearson certificate is more directly interpretable, but it requires an independent certification sample from the distribution being certified and can be conservative when the certification fold is small. Finally, the confidence threshold $\varepsilon^\ast$ is selected before constrained 
training, and the final model is chosen by held-out feasibility via a validation  checkpoint; a fully joint treatment of threshold selection, checkpoint selection, and  reliability certification,  ideally on a single certification budget, would further  strengthen the framework. A related open direction is subgroup-conditional control: our per-family analysis on UNSW-NB15 shows that marginal JCW feasibility does not guarantee  feasibility within every subpopulation, motivating group-conditional constraints.

\medskip
\noindent Overall, ReliableNet  provides a principled path toward classifiers that are not only accurate on average, but also safer in the predictions they choose to make confidently.

\bibliography{main} 
\bibliographystyle{unsrtnat}

\appendix
\section{Appendix}
\label{Appendix}
This part presents some useful properties of the Geletu-Hoffman parametric function.

% ================================================
\subsection{Some Properties of the Smoothing Family}
\begin{lemma}[Properties of the smoothing family]
\label{lem:family_props}
Let $h(s):=\mathbf 1\{s\ge0\}$ and, for $0<m_2<\dfrac{m_1}{1+m_1}$ and $\tau\in(0,1)$,
\[ \zeta(\tau,s)=\frac{1+m_1\tau}{1+m_2\tau e^{-s/\tau}},\qquad s\in\mathbb{R}.\]
Then:
\begin{enumerate}
\item[\textnormal{(i)}] \emph{(Boundedness)}  $0<\zeta(\tau,s)\le 1+m_1\tau$ for all $s\in\mathbb{R}$.
\item[\textnormal{(ii)}] \emph{(Conservatism)} $\zeta(\tau,s)\ge h(s)$ for all $s\in\mathbb{R}$.
\item[\textnormal{(iii)}] \emph{(Pointwise limit)} $\displaystyle\lim_{\tau\downarrow0}\zeta(\tau,s)=h(s)$ for all $s\in\mathbb{R}$.
\end{enumerate}
\end{lemma}

\begin{proof}
Write $w:=m_2\tau e^{-s/\tau}>0$, so the denominator is $1+w>1$.

\emph{(i)} Since $1+w>1$ and the numerator $1+m_1\tau>0$, $0<\zeta(\tau,s)=\tfrac{1+m_1\tau}{1+w}<1+m_1\tau$; the upper value $1+m_1\tau$ is approached as $s\to+\infty$ (then $w\to0$).

\emph{(ii)} The map $s\mapsto\zeta(\tau,s)$ is strictly increasing, because as $s$ increases $e^{-s/\tau}$ decreases, hence $w$ decreases and
$\zeta=\tfrac{1+m_1\tau}{1+w}$ increases. For $s<0$, $h(s)=0<\zeta(\tau,s)$ by (i). For $s\ge0$, monotonicity gives $\zeta(\tau,s)\ge\zeta(\tau,0)=\tfrac{1+m_1\tau}{1+m_2\tau}>1=h(s)$. Hence $\zeta(\tau,s)\ge h(s)$ for all $s$.

\emph{(iii)} Fix $s$ and let $\tau\downarrow0$. If $s>0$ then $-s/\tau\to-\infty$, so $e^{-s/\tau}\to0$, the denominator $\to1$, and
$\zeta\to1=h(s)$. If $s<0$ then $-s/\tau\to+\infty$ and $\tau e^{-s/\tau}\to+\infty$ (the exponential dominates the factor $\tau\downarrow0$), so the denominator $\to+\infty$ and $\zeta\to0=h(s)$. If $s=0$ then $\zeta(\tau,0)=\tfrac{1+m_1\tau}{1+m_2\tau}\to1=h(0)$. In all cases $\zeta(\tau,s)\to h(s)$.
\end{proof}

\end{document}